%% file: main.tex
\documentclass{article}
\usepackage{ijcai25}

\usepackage{times}
\usepackage{soul}
\usepackage{url}
\usepackage[hidelinks]{hyperref}
\usepackage[utf8]{inputenc}
\usepackage[small]{caption}
\usepackage{graphicx}
\usepackage{booktabs}
\usepackage[switch]{lineno}

\usepackage{amsmath,amssymb,amsfonts}
\usepackage{amsthm}
\usepackage[ruled]{algorithm2e}
\usepackage{algorithmic}
\usepackage{subcaption}
\usepackage{xcolor}

\title{A Flexible Fairness Framework with Surrogate Loss Reweighting for Addressing Sociodemographic Disparities}

\author{
    Author Name
    \affiliations
    Affiliation
    \emails
    email@example.com
}

\author{
Wen Xu$^1$\footnote{Corresponding Author}
\and
Elham Dolatabadi$^2$
\affiliations
$^1$University of Toronto \\
$^2$York University, Vector Institute\\
\emails
realwen.xu@mail.utoronto.ca,
edolatab@yorku.ca
}
\input{defs}

\begin{document}
\maketitle
\begin{abstract}
%What is considered fair can change over time, across different domains, and in response to evolving ethical standards. To address this dynamic landscape, a flexible fairness approach is essential for continuously assessing and adapting to fairness concerns across different contexts. 

This paper presents a new algorithmic fairness framework called $\boldsymbol{\alpha}$-$\boldsymbol{\beta}$ Fair Machine Learning ($\boldsymbol{\alpha}$-$\boldsymbol{\beta}$ FML), designed to optimize fairness levels across sociodemographic attributes. Our framework employs a new family of surrogate loss functions, paired with loss reweighting techniques, allowing precise control over fairness-accuracy trade-offs through tunable hyperparameters $\boldsymbol{\alpha}$ and $\boldsymbol{\beta}$. To efficiently solve the learning objective, we propose Parallel Stochastic Gradient Descent with Surrogate Loss (P-SGD-S) and establish convergence guarantees for both convex and nonconvex loss functions. Experimental results demonstrate that our framework improves overall accuracy while reducing fairness violations, offering a smooth trade-off between standard empirical risk minimization and strict minimax fairness. Results across multiple datasets confirm its adaptability, ensuring fairness improvements without excessive performance degradation.
\end{abstract}
\section{Introduction}
Sociodemographic disparities—driven by unequal access to education, healthcare, and employment—remain one of the most urgent challenges of our time \cite{zajacova2018relationship}. Addressing these inequities demands more than equal treatment; it requires confronting historical injustices and systemic barriers. Amartya Sen’s ``The Idea of Justice''~\cite{sen2008idea} and Matt Cavanagh’s ``Against Equality of Opportunity''~\cite{roemer2015equality} advocate for a justice framework that is adaptive and context-sensitive, ensuring fairness is measured by tangible improvements in people’s lives rather than abstract ideals.

This need for a more nuanced and contextual approach to fairness is especially critical in machine learning (ML) due to its increasingly influence on decision-making in high-stakes domains such as healthcare, finance, and criminal justice. The widespread adaoption of ML models raises urgent concerns about their potential to either mitigate or reinforce existing sociodemographic disparities~\cite{FairML_book,FairML_survey,juhn2022assessing}. At the core of this challenge lies the dual task of \textit{defining} and \textit{achieving} fairness in algorithmic systems. While formal fairness criteria offer structured approaches to bias mitigation, they often fail to account for broader sociodemographic contexts, limiting their ability to address deeper systemic inequities~\cite{FairML_book,FairML_survey}.

Existing fairness-aware ML approaches primarily address bias through statistical constraints—such as demographic parity and equalized odds—applied at various stages of the pipeline, including pre-processing transformations, in-processing optimization frameworks, and post-processing adjustments~\cite{AISTATS17_Zafar,ICML18_Hashimoto,NeurIPS20_Lahoti,ICLR20_Baharlouei}. While these methods provide measurable fairness guarantees, they often frame fairness as a rigid optimization constraint rather than an adaptable principle that evolves with changing contexts. This static approach risks overlooking the broader structural factors that contribute to disparities~\cite{buyl2024fairret}. Strategies that prioritize improving outcomes for historically marginalized groups frequently introduce trade-offs that may compromise performance for others, leading to ethical and practical tensions~\cite{selbst2019fairness}. Bridging this gap requires fairness frameworks that move beyond static parity metrics, incorporating a \textit{flexible}, context-aware understanding of equity while \textit{preserving model utility} —one that aligns with Sen’s and Cavanagh’s vision of justice as an active and corrective force~\cite{noble2018algorithms,dolata2022sociotechnical}.

In this work, we propose a novel fairness-aware optimization framework that dynamically balances fairness and accuracy, adapting to the specific socioeconomic context of an application. Drawing inspiration from fair resource allocation in network congestion control~\cite{alpha_fairness} and fairness-aware learning in federated settings~\cite{ICLR20_Li}, our framework generalizes existing fairness paradigms, including minimax fairness proposed in~\cite{ICML18_Hashimoto,ICML19_AFL}, which prioritizes the most underrepresented groups. By embedding fairness directly into the optimization process, our approach advances the broader goal of developing fairness-aware models that account for ethical and policy implications across diverse domains such as healthcare, criminal justice, and economic decision-making. 
The key \textbf{contributions} of this work are as follows:

\begin{itemize}
    \setlength\itemsep{0em}
    \item We propose a novel model-agnostic framework, called $\boldsymbol{\alpha}$-$\boldsymbol{\beta}$ Fair Machine Learning ($\boldsymbol{\alpha}$-$\boldsymbol{\beta}$ FML), for in-processing fair prediction through the utilization of a family of surrogate loss functions with loss reweighting techniques. These functions incorporate reweighting the polynomials of loss functions on different subgroups to calibrate the distribution of performance among protected features or imbalanced labels. The $\boldsymbol{\alpha}$-$\boldsymbol{\beta}$ FML has two sets of hyperparameters $\boldsymbol{\alpha}$ and $\boldsymbol{\beta}$, which enable an arbitrary level of fairness on model performance.
    \item We introduce a distributed stochastic optimization algorithm, referred to as Parallel Stochastic Gradient Descent with Surrogate Loss (\textsc{P-SGD-S}), designed to train ML models using our newly proposed loss functions. We also provide convergence guarantees for~\textsc{P-SGD-S} for both convex and nonconvex loss functions.
    \item We validate the flexibility of our framework in achieving fairness across three classification tasks: two based on sociodemographic attributes and one on label-partitioned fairness. Our approach adapts to different fairness definitions, balancing accuracy and fairness trade-offs in both convex and nonconvex settings.
\end{itemize}

\section{Related Work}\label{sec:related}

% \paragraph{Fairness in Machine Learning} 
There is no uniform definition of fairness in ML, but two main levels of fairness are generally recognized in the literature: individual-level fairness and group-level fairness. Individual level fairness aims to achieve fairness for each unique entity in the dataset~\cite{ITCS12_Dwork,ICML13_Zemel,ICLR16_Louizos}, as well as counterfactual fairness~\cite{Neurips17_Kusner} and PAC-learnable individual fairness~\cite{ICML18_Yona}. On the other hand, group fairness is concerned with metrics such as demographic parity~\cite{AISTATS17_Zafar} and equality of opportunity~\cite{NIPS16_Hardt}. These recognitions have led to the adoption of three types of methods for achieving fair predictions: pre-processing, in-processing, and post-processing methods. 

\textbf{Pre-processing methods} modify input data before training, such as \emph{fairness through unawareness}~\cite{NIPS16_Nina}, advocating for the removal of protected features before model training. Another approach, presented in \cite{NeurIPS17_Calmon}, introduces convex optimization for learning a data transformation. This method aims to limit distortion and control discrimination while simultaneously preserving utility.

\textbf{In-processing methods} integrate fairness considerations directly into the learning process. In~\cite{AISTATS17_Zafar,ICLR20_Baharlouei}, the authors proposed incorporating fairness constraints, such as the Pearson correlation and R\'enyi correlation between the protected features and the predictions, to impose fairness. Reweighting-based methods have also been proposed~\cite{ICML20_Lohaus,NeurIPS21_Bendekgey,TMLR24_Yao}, where fairness objectives are incorporated by combining empirical loss with fairness constraints, focusing on metrics like demographic parity and equality of opportunity. Building on Rawlsian distributive justice~\cite{Justice01_Rawls}, distributionally robust optimization (DRO)~\cite{arXiv19_Rahimian,AoS21_Duchi} has been explored as a fairness mechanism, where minimax optimization problems are solved to ensure robust model performance~\cite{ICML18_Hashimoto,ICLR20_Sagawa}. DRO-based methods have also been applied in distributed ML~\cite{ICML19_AFL}, highlighting their applicability to federated settings. Some works in distributed ML have considered fairness notions such as proportional fairness~\cite{TMLR23_Zhang} and $\alpha$-fairness~\cite{ICLR20_Li}. However, proportional fairness, as defined in~\cite{TMLR23_Zhang}, focuses on the relative improvement in each client’s performance, which is unrelated to our objective. The most related work to ours is~\cite{ICLR20_Li}, where the authors introduce a $q$-fair federated learning ($q$-FFL) framework with surrogate loss functions. However, their approach lacks theoretical convergence analysis. Besides, the model's hyperparameter $q$ was uniform among all clients, i.e., groups, and only the fairness on model accuracy was investigated. Unlike much of the algorithmic fairness literature~\cite{rabonato2024systematic}, which relies on auxiliary fairness constraints or regularization techniques, our approach directly integrates fairness into the learning process through a novel loss function, making it fundamentally distinct from existing literature.

\textbf{Post-processing methods} mitigate discrimination by adjusting model outcomes. For instance, \cite{NIPS16_Hardt} introduced an optimal approach for equality of opportunity, relying solely on the joint statistics of the predictor, the target, and the protected feature. Moreover, \cite{SDM16_Fish} proposed shifting decision boundaries for protected groups in various models. \cite{FACCT18_Dwork} offered a transfer learning-based decoupling algorithm for fair adjustments across black-box models, and \cite{NeurIPS22_Alghamdi} used information projection to post-process classifier outputs. However, these methods often incur additional computational overhead, limiting their practicality in resource-constrained settings.

\section{Preliminaries}\label{sec:preliminaries}
In this section, we provide essential technical groundwork on fairness in ML to underpin the development of our theory.

%In the context of learning, the typical goal is to utilize some provided data to construct a model that can generalize well to unknown data from the same underlying distribution. 
We consider a supervised learning scenario, where each data point consists of a feature vector $Z \in \cZ$, with $\cZ$ as the feature space, and a label $Y \in \cY$, with $\cY$ representing the outcome space. To simplify our presentation, we focus on the binary classification problem, i.e., $\cY = \{0, 1\}$. However, the concepts and findings can be readily extended to multi-class classification problems or regression tasks. It is important to note that some features in $Z$ should be treated carefully, concerning fairness and bias in ML. For instance, under the Civil Rights Act of 1964 in the United States~\cite{CivilRight64}, features such as age, gender, race, and religion are classified as protected features for individuals. Thus, we categorize all features into protected features $A \in \cA$ and non-protected features $X \in \cX$ such that $Z = (A, X) \in \cA \times \cX = \cZ$. The output prediction from the trained model is represented by $\hat{Y}$.
In the remaining parts of the paper, we use uppercase symbols to denote random variables or vectors while we use lowercase symbols to represent unspecified realizations of them. If these lowercase symbols represent realizations of random vectors, they are presented in bold font. Specifically, we use the lowercase $\bx$, $\ba$, $\bz$, $y$, and $\hat{y}$ to represent a realization of the random variable/vector $X$, $A$, $Z$, $Y$, and $\hat{Y}$, respectively.

We consider the following well-known definition for group-level fairness, which entails equivalent model performance across different groups. Here, we make an additional assumption that the protected feature $A$ is a single, binary feature and use $A=a$ to represent the protected feature taking a scalar value of $a$. The definitions can be extended to include multiple protected features and categorical ones beyond binary.

\begin{definition}[Equality of Accuracy (EA)]\label{def:ea}
A model satisfies equality of accuracy if 
\begin{align}
    P(\hat{Y}= Y|A=0) = P(\hat{Y}= Y|A=1).
\end{align}
\end{definition}

Achieving definition~\ref{def:ea} implies equal accuracy of predicted outcomes for the two groups, categorized based on the protected feature, $A$.
\begin{definition}[EA violation]\label{def:ea:violation}
The EA violation is defined as 
\begin{align}\label{eqn:ea:violation}
    \epsilon_{\text{EA}} = |P(\hat{Y}= Y|A=0) - P(\hat{Y}= Y|A=1)|.
\end{align}
\end{definition}
A smaller EA violation $\epsilon_{\text{EA}}$ indicates a more uniform model performance across the two groups. In addition, we incorporate other widely recognized fairness criteria, such as demographic parity (DP)~\cite{AISTATS17_Zafar} and equality of opportunity (EO)~\cite{NIPS16_Hardt}, as discussed in the Appendix.

\section{Fair Machine Learning Formulation}\label{sec:betefair}
% \subsection{Definitions}
We begin by formalizing the $\boldsymbol{\alpha}$-$\boldsymbol{\beta}$ fair machine learning (FML) under the assumption that all available $N$ feature-label pairs $\{(\bz_{i}, y_{i})\}_{i=1}^{N}$ are independently drawn from an unknown distribution $\cD$. Here, a learner is represented by a hypothesis $h \in \cH$, where $\cH$ denotes the hypothesis class. For any $Z \in \cZ$, the predicted outcome of $h$ is $\hat{Y} = h(Z)$. We further assume that each learner $h$ is parameterized by $\bw \in \cW \subseteq \bbR^{d}$. We define the loss function as $\ell: \cW \times \cZ \times \cY \rightarrow \bbR_{+}$ such that the loss incurred by a data point $(\bz, y)$ with respect to a model $w$ is given by a non-negative value $\ell(\bw, (\bz,y))$. For convenience, we represent the loss incurred by a data point $(\bz_j, y_j)$ using a model $\bw$ by $\ell_{j}(\bw)$ for any $j \in [N]$, where $[N] = \{1, \dots,N\}$. 

A canonical approach to solving a learning problem involves empirical risk minimization (ERM)~\cite{UML,FML}, where the objective is to identify a learner that minimizes the average loss incurred by the provided training examples. This optimization can be formulated as the following finite sum minimization problem:
\begin{align}\label{def:erm:opt}
    \min_{\bw \in \cW} \frac{1}{N}\sum_{j=1}^{N} \ell_{j}(\bw).
\end{align}

We propose a new family of surrogate loss functions that facilitate the incorporation of diverse levels of fairness, inspired by the utility function in resource allocation~\cite{alpha_fairness} and fair resource allocation in FL~\cite{ICLR20_Li}. To differentiate this approach from previous works, we name this family of loss functions as $\beta$-fair surrogate loss functions. 
\begin{definition}[$\beta$-fair Surrogate Loss Function]
The family of $\beta$-fair loss functions for any $(\bz, y) \in (\cZ, \cY)$ and for any $\bw \in \cW$ is
\begin{align}\label{def:betafair:loss}
    % f_{\beta}(w, (z, y)) = \frac{\ell^{1+\beta}(w, (z, y))}{1+\beta},
    f_{\beta}(\bw, (\bz, y)) = \frac{(1+\ell(\bw, (\bz, y)))^{^{1+\beta}}}{1+\beta},
\end{align}
where $\beta \geq 0$ and $\ell$ is any given nonnegative loss function.
\end{definition}

\begin{remark}
    Our proposed $\beta$-fair surrogate loss function can be interpreted as placing larger penalties on data points with higher losses when $\beta > 0$. Furthermore, it serves as an upper bound for the original loss, which contrasts with the surrogate loss function used in~\cite{ICLR20_Li}. Notably, their approach does not include a bias term (such as $+1$ in the numerator of \eqref{def:betafair:loss}), making this upper-bound property not always guaranteed in their formulation.
\end{remark}

We further define our surrogate empirical loss for the learning problem. As previously mentioned, we focus on a scenario where the protected feature $A$ is a single binary attribute, assuming values in $\{0, 1\}$. We denote $\cS_{0}$ for the set of samples with $A = 0$ and $\cS_{1}$ for the set of samples with $A = 1$. The cardinality of $\cS_{0}$ and $\cS_{1}$ are $S_{0}$ and $S_{1}$, respectively.  
\begin{definition}[$\boldsymbol{\alpha}$-$\boldsymbol{\beta}$ Fair Surrogate Loss Function]
   The group-aware $\boldsymbol{\alpha}$-$\boldsymbol{\beta}$ fair surrogate loss is defined as
\begin{align}\label{def:alphabetafair}
    L_{(\boldsymbol{\alpha}, \boldsymbol{\beta})}(\bw)  = \alpha_{0}F_{\beta_0}(\bw) + \alpha_{1}F_{\beta_1}(\bw), 
\end{align}
where $F_{\beta_0}(\bw) = \frac{1}{S_{0}} \sum_{i \in \cS_{0}}f_{\beta_{0}}(\bw, (\bz_{i}, y_{i}))$ is the average $\beta_{0}$-fair surrogate loss on $\cS_{0}$,  $F_{\beta_1}(\bw) = \frac{1}{S_{1}} \sum_{i \in \cS_{1}}f_{\beta_{1}}(\bw, (z_{i}, y_{i}))$ is the average $\beta_{1}$-fair surrogate loss on $\cS_{1}$, and $\boldsymbol{\alpha} = (\alpha_{0}, \alpha_{1}) \in \Delta_{1}$, i.e., the $1$-dimensional probability simplex in $\mathbb{R}^{2}$. For convenience, we let $\boldsymbol{\beta} = (\beta_{0}, \beta_{1})$.
\end{definition}
In our proposed framework, the objective is to minimize the $\boldsymbol{\alpha}$-$\boldsymbol{\beta}$ fair surrogate loss in~\eqref{def:alphabetafair}, i.e.,
\begin{align}\label{def:alpha-beta:opt}
    \min_{\bw \in \cW} L_{(\boldsymbol{\alpha}, \boldsymbol{\beta})}(\bw).
\end{align}
% Whenever the context is evident, we denote $L_{(\boldsymbol{\alpha}, \boldsymbol{\beta})}(w)$ as $L(w)$.
\begin{remark}
    Note that our formulation has two sets of tunable hyperparameters: $\boldsymbol{\alpha}$ and $\boldsymbol{\beta}$. The hyperparameter $\boldsymbol{\alpha}$ captures the explicit weights assigned to each group of the data, while the hyperparameter $\boldsymbol{\beta}$ steers the surrogate loss for data sample from each group. Common choices for $\boldsymbol{\alpha}$ include equal group weights, i.e., $\alpha_{0} = \alpha_{1} = 1/2$, and equal sample weights, i.e.,  $\alpha_{0} =S_0/N$ and $\alpha_{1} =S_1/N$, making the weights proportional to the number of data points in each group.
\end{remark}

\begin{remark}
Our formulation subsumes many previous works as special cases. When setting $\alpha_{0} =S_0/N$, $\alpha_{1} =S_1/N$, and $\beta_{0} = \beta_{1} = 0$, our formulation in~\eqref{def:alpha-beta:opt} is equivalent to the original ERM formulation in~\eqref{def:erm:opt}. If we set $\beta_{0} = \beta_{1} = q > 0$, and reset the $+1$ bias term in the numerator to be $0$, our formulation in~\eqref{def:alpha-beta:opt} aligns with the optimization formulation of q-FFL in~\cite{ICLR20_Li} with $2$ clients. Additionally, by letting $\alpha_{0}$ and $\alpha_{1}$ be positive while $\beta_{0}$ and $\beta_{1}$ approach positive infinity, min-max fairness~\cite{ICML18_Hashimoto,ICML19_AFL} is imposed, where the data point with largest loss dominates. As such, tuning $\boldsymbol{\alpha}$ and $\boldsymbol{\beta}$ allows for imposing any level of fairness between the extreme case of ERM and the min-max fairness. 
\end{remark}

\begin{remark}
    Note that~\eqref{def:alpha-beta:opt} can be adapted for scenarios involving protected or unbalanced labels instead of protected features. In this case, assuming binary labels, we can redefine $\cS_{0}$ as the index set of samples with label $Y = 0$ and $\cS_{1}$ as the index set of samples with label $Y = 1$. We refer to the scenario of grouping samples by values of $Y$ as the \textit{label-partitioned scenario} and the scenario of grouping samples by values of protected features $A$ as the \textit{feature-partitioned scenario}. In our experiments, we demonstrate the flexibility of our proposed methods in both scenarios.
\end{remark}

\section{Algorithm Design}\label{sec:algorithm}
% \subsection{Doubly Parallel Stochastic Gradient Descent (P-SGD-S)}
In this section, we introduce the Parallel Stochastic Gradient Descent with Surrogate Loss~(\textsc{P-SGD-S}) algorithm to solve the optimization problem in~\eqref{def:alpha-beta:opt}. \textsc{P-SGD-S} is based on stochastic gradient descent~(\textsc{SGD}), one of the most popular first-order methods in ML~\cite{SIAM_Bottou}. Besides, \textsc{P-SGD-S} allows model updates for both groups in a parallel manner to better utilize modern computation resources~\cite {NIPS10_Zinkevich}. The pseudocode for \textsc{P-SGD-S} is shown in Algorithm~\ref{alg:P-SGD-S}. 

We now provide a detailed description of~\textsc{P-SGD-S}. The algorithm begins by initializing a model $\bw^{(0)} \in \cW$ and then proceeds with training over $T$ rounds. For each round $t \in \cT = \{0, \dots, T\!-\!1\}$, the following three steps are performed for each $i$th group, where $i \in \cI = \{0,1\}$: i) A mini-batch of training examples $\xi_{i}^{(t)} \subseteq \cS_{i}$ is sampled uniformly at random, ii) A stochastic gradient is obtained:
\begin{align}\label{eqn:stochasticgradient}
    \!\!\tilde{\nabla} F_{\beta_i}(\bw^{(t)}; \xi_{i}^{(t)}) \!=\! \sum_{i \in \xi_{i}^{(t)}} \frac{(1+\ell_{i}(\bw^{(t)}))^{\beta_{i}} \nabla \ell_{i}(\bw^{(t)})}{|\xi_{i}^{(t)}|},
\end{align}
iii) The model is updated according to the derived gradient:
\begin{align}\label{eqn:modelupdate}
    \bw_{i}^{(t+1)} = \bw^{(t)} - \gamma_{i}^{(t)} \tilde{\nabla} F_{\beta_i}(\bw^{(t)};\xi_{i}^{(t)}).
\end{align}
After completing these steps for each group, the model is aggregated, finalizing one training round:
\begin{align}\label{eqn:modelaggregation}
    \bw^{(t+1)} = \sum_{i \in \cI} \alpha_{i}\bw_{i}^{(t+1)},
\end{align}
Unlike traditional \textsc{SGD}, \textsc{P-SGD-S} enables parallel updates for the two groups indexed by $\cI$, leveraging computational resources when multiple computing units are available. Furthermore, \textsc{P-SGD-S} can be extended to incorporate multiple steps of model updates between periodic aggregations, similar to \textsc{local SGD}~\cite{ICLR19_Stich,NeurIPS20_Woodworth} to reduce communication overhead in the distributed setting. Extensions of \textsc{P-SGD-S} including momentum~\cite{momentum}, adaptive learning rates~\cite{JMLR11_Duchi}, and stochastic average gradient~\cite{SAGA,SAG} are also possible. However, it is worth noting that convergence analysis for these extensions can be intricate and we leave it as future work.
% We only provide convergence analysis for the basic \textsc{P-SGD-S} for clarity.

\begin{algorithm}[t]
\renewcommand{\algorithmicrequire}{\textbf{Input:}}
\renewcommand{\algorithmicensure}{\textbf{Output:}}
\caption{\textsc{Parallel Stochastic Gradient Descent with Surrogate Loss (P-SGD-S)}}\label{alg:P-SGD-S}
\begin{algorithmic}[1]
\REQUIRE learning rate $\gamma_{i}^{(t)}$ for $i \in \cI$, hyperparameters $\alpha_{i}$ and $\beta_{i}$ for $i \in \cI$, and training rounds $T$.
\ENSURE $\{\bw^{(t)}\}_{t=0}^{T}$.
\STATE Initialize $\bw^{(0)} \in \cW$.
\FOR{each round $t = 0, \dots, T-1$}
    \FOR{each $i \in \cI$}
        % \STATE Sample $\xi_{i}^{(t)}$ from $\cS_{i}$ and obtain stochastic gradients $\nabla F_{\beta_{i}}(w^{(t)}; \xi_{i}^{(t)})$.
        % \STATE $w_{i}^{(t+1)} = w^{(t)} - \eta_{i}\nabla F_{\beta_{i}}(w^{(t)}; \xi_{i}^{(t)})$.
        \STATE Sample $\xi_{i}^{(t)} \subseteq \cS_{i}$ uniformly at random.
        \STATE Obtain stochastic gradient $\tilde{\nabla} F_{\beta_i}(\bw^{(t)})$ via~\eqref{eqn:stochasticgradient}.
    
        % \STATE Obtain $\nabla \ell(\bw^{(t)}; \xi_{i}^{(t)})$ and $\ell(\bw^{(t)}; \xi_{i}^{(t)})$.
        % \STATE Calculate $\gamma_{i}^{(t)} = \eta_{i} l^{\beta_{i}}(w^{(t)}; \xi_{i}^{(t)})$.        
        % \STATE Calculate $\gamma_{i}^{(t)} = \eta_{i} (1+\ell(\bw^{(t)}; \xi_{i}^{(t)}))^{\beta_{i}}$.
        \STATE Update the model via~\eqref{eqn:modelupdate}.
    \ENDFOR 
    \STATE Aggregate the model via~\eqref{eqn:modelaggregation}.
\ENDFOR   
\end{algorithmic}
\end{algorithm}

\section{Convergence Analysis}\label{sec:analysis}
In this section, we provide the convergence analysis of \textsc{P-SGD-S}. To maintain the clarity of our explanations, we assume without loss of generality that $\gamma_{i}^{(t)} = \gamma, \forall i \in \cI, t \in \cT$. We begin by making the following assumptions~\cite{SIAM_Bottou}.

\begin{assumption}[Smoothness]\label{assumption:smooth:1}
% Each loss function $\ell$ is $L$-smooth:
% \begin{equation}\label{eqn:assumption:smooth:1}
%      \left\|\nabla \ell(w_0) - \nabla \ell(w_1)\right\| \leq L\left\|w_0 - w_1\right\|, \quad \forall w_0, w_1 \in \cW.
% \end{equation}
There exists some positive $L_{0}$ such that $\left\|\nabla L_{(\boldsymbol{\alpha}, \boldsymbol{\beta})}(\bw_0) - \nabla L_{(\boldsymbol{\alpha}, \boldsymbol{\beta})}(\bw_1)\right\| \leq L_{0}\left\|\bw_0 - \bw_1\right\|,  \forall \bw_0, \bw_1 \in \cW.$
% Each loss function $\ell$ is $L$-smooth:
% \begin{equation}\label{eqn:assumption:smooth:1}
%      \left\|\nabla \ell(w_0) - \nabla \ell(w_1)\right\| \leq L\left\|w_0 - w_1\right\|, \quad \forall w_0, w_1 \in \cW.
% \end{equation}
\end{assumption}
\begin{remark}
    The smoothness $L_{0}$ is actually dependent on $\boldsymbol{\alpha}$, $\boldsymbol{\beta}$, and the smoothness of the bounded original loss function $\ell$. For simplicity, we represent the smooth of surrogate loss $L_{(\boldsymbol{\alpha}, \boldsymbol{\beta})}$ by $L_{0}$ without explicitly writing down the dependence. 
    % The detailed bound for $L_{0}$ is provided in the Appendix~\ref{appendix:smoothness}. 
\end{remark}

\begin{assumption}[Unbiased Local Stochastic Gradient]\label{assumption:unbiased}
The stochastic gradient $\tilde{\nabla} F_{\beta_{i}}(\bw; \xi)$ is an unbiased estimator of the full gradient for any group-specific loss function $F_{\beta_{i}}(\bw)$, i.e., $\bbE[\tilde{\nabla} F_{\beta_{i}}(\bw; \xi)] = \nabla F_{\beta_{i}}(\bw), \, \forall \bw \in \cW, \forall i \in \cI, \forall \xi \in \Xi$, where $\Xi$ denotes the set of all possible mini-batches (i.e., the set of randomness in the sampling process).
\end{assumption}
\begin{remark}
The assumption holds since in \textsc{P-SGD-S}, a mini-batch $\xi_{i}^{(t)}$ is chosen uniformly at random for all $i \in \cI, t \in \cT$, i.e., we have
\begin{align}
\bbE[\tilde{\nabla} F_{\beta_i}(\bw^{(t)})] 
& \!=\! \bbE[\sum_{i \in \xi_{i}^{(t)}} |\xi_{i}^{(t)}|^{-1} (1 \!+\! \ell_{i}(\bw^{(t)}))^{\beta_{i}} \nabla \ell_{i}(\bw^{(t)})] \nonumber \\
& \!=\! \sum_{i \in \cS_{i}} (S_{i})^{-1} (1 \!+\! \ell_{i}(\bw^{(t)}))^{\beta_{i}} \nabla \ell_{i}(\bw^{(t)}) \nonumber \\
& \!=\! \nabla F_{\beta_i}(\bw^{(t)}). \nonumber
\end{align}
\end{remark}

\begin{assumption}[Bounded Stochastic Gradient Variance]\label{assumption:boundedvariance}
There exists a positive $\sigma^{2}$ such that the variance of the stochastic gradient $\tilde{\nabla} F_{\beta_{i}}(\bw; \xi)$ is bounded as $\bbE[\|\tilde{\nabla} F_{\beta_{i}}(\bw; \xi) - \nabla F_{\beta_{i}}(\bw)\|^2] \leq \sigma^{2}, \, \forall \bw \in \cW, \, \forall i \in \cI$.
\end{assumption}

We first provide an important lemma to facilitate our analysis of \textsc{P-SGD-S}.
\begin{lemma}[Equivalence and Expectation of Model Update]\label{lemma:update}
An equivalent expression of the per round update in~\eqref{eqn:modelaggregation} is
\begin{equation}
    \bw^{(t+1)} =  \bw^{(t)} - \gamma\sum_{i \in \cI} \alpha_{i}\tilde{\nabla} F_{\beta_i}(\bw^{(t)}; \xi_{i}^{(t)}), \label{eqn:lemma:update}
\end{equation}
whose expectation is
\begin{equation}
    \bbE[\bw^{(t+1)}] =  \bw^{(t)} - \gamma \nabla L_{(\boldsymbol{\alpha}, \boldsymbol{\beta})}(\bw^{(t)}). \label{eqn:lemma:update:expectation}
\end{equation}
\end{lemma}
\begin{proof}
Substituting the update rule in~\eqref{eqn:modelupdate} into the aggregation rule in~\eqref{eqn:modelaggregation}, we can obtain~\eqref{eqn:lemma:update}. Taking expectation on  both sides of~\eqref{eqn:lemma:update} and applying Assumption~\ref{assumption:unbiased}, we obtain~\eqref{eqn:lemma:update:expectation}.  
\end{proof}

In the following two subsections, we analyze the two scenarios involving convex and nonconvex losses, respectively.
\subsection{Convex Loss}
For this scenario, we make an additional assumption regarding the convexity of the original loss functions. Each loss function $\ell$ is convex, $\ell(\bw_1) \geq \ell(\bw_0) + \nabla \ell(\bw_0)^{T}(\bw_1-\bw_0)$, holds for $\forall \bw_0, \bw_1 \in \cW$. The surrogate loss 
$L_{(\boldsymbol{\alpha}, \boldsymbol{\beta})}(\bw)$ is also convex as it is a nonnegative weighted sum of the $\beta_{i}$-fair surrogate loss functions, each of which is also convex due to the facts that $\beta_{i} \geq 0$ holds $\forall i \in \cI$ and $\ell$ is nonnegative and convex.

\begin{theorem}[Convergence Analysis for Convex Loss]\label{thm:cvx}
If Assumptions~\ref{assumption:smooth:1}-\ref{assumption:boundedvariance} hold and all $\ell$ loss functions are convex, with $\eta \leq \frac{1}{L_{0}}$, we obtain the output $\{\bw^{(t)}\}_{t=0}^{T}$ from Algorithm~\ref{alg:P-SGD-S} satisfies
\begin{align}
    & \bbE[L_{(\boldsymbol{\alpha}, \boldsymbol{\beta})}(\hat{\bw})] - L_{(\boldsymbol{\alpha}, \boldsymbol{\beta})}(\bw^{*})  \leq \frac{1}{2\eta}(\|\bw^{(0)} - \bw^{*}\|^2 ) + \eta \sigma^{2}, \nonumber
\end{align}
where we define $\hat{\bw} = \frac{1}{T} \sum_{t=0}^{T-1} \bw^{(t)}$ and $\bw^{*} = \arg\min_{\bw \in \cW} L_{(\boldsymbol{\alpha}, \boldsymbol{\beta})}(\bw)$.
\end{theorem}
\begin{proof}
    See Appendix~\ref{appendix:thm1}.
    % See the Appendix.
\end{proof}
\begin{corollary}
Let $\eta = \sqrt{\frac{\|\bw^{(0)} - \bw^{*}\|^2 }{2T \sigma^{2}}}$, we have 
\begin{align}
& \bbE[L_{(\boldsymbol{\alpha}, \boldsymbol{\beta})}(\frac{1}{T} \sum_{t=0}^{T-1} \bw^{(t)})] - L_{(\boldsymbol{\alpha}, \boldsymbol{\beta})}(\bw^{*}) \leq \sqrt{\frac{\|\bw^{(0)} - \bw^{*}\|^2 \sigma^{2} }{2T}}. \nonumber
\end{align} 
We observe that it has the same convergence rate $\cO(\frac{1}{\sqrt{T}})$ as \textsc{SGD} for smooth, convex functions~\cite{SIAM_Bottou}.
\end{corollary}

\subsection{Nonconvex Loss}
In modern ML such as neural network training, the loss functions are typically nonconvex. In this scenario, we relax the assumption of convexity and provide convergence analysis to a stationary point~\cite{SIAM_Bottou}.
\begin{theorem}[Convergence Analysis for Nonconvex Loss]\label{thm:noncvx}
If Assumptions~\ref{assumption:smooth:1}-\ref{assumption:boundedvariance} holds and $\eta \leq \frac{1}{L_{0}}$, then the output $\{\bw^{(t)}\}_{t=0}^{T}$ from Algorithm~\ref{alg:P-SGD-S} satisfies
\begin{align}
& \frac{1}{T}\sum_{t=0}^{T-1}\left\|\nabla L_{(\boldsymbol{\alpha}, \boldsymbol{\beta})}(\bw^{(t)})\right\|^{2} \nonumber \\
& \leq \frac{2}{\eta}\left(L_{(\boldsymbol{\alpha}, \boldsymbol{\beta})}(\bw^{(0)})-L^{*} \right)
+L_{0} \eta T \sigma^{2},
\end{align}
where $L^{*} = \min_{\bw \in \cW} L_{(\boldsymbol{\alpha}, \boldsymbol{\beta})}(w)$.
\end{theorem}
\begin{proof}
    % See the Appendix.
    See Appendix~\ref{appendix:thm2}.
\end{proof}
\begin{corollary}
Let $\eta = \sqrt{\frac{2\left(L_{(\boldsymbol{\alpha}, \boldsymbol{\beta})}(\bw^{(0)})-L^{*} \right)}{L_{0} T \sigma^{2}}}$, we have 
\begin{align}
& \frac{1}{T}\sum_{t=0}^{T-1}\left\|\nabla L_{(\boldsymbol{\alpha}, \boldsymbol{\beta})}(\bw^{(t)})\right\|^{2} \leq \sqrt{\frac{2L_{0}\sigma^{2} \left(L_{(\boldsymbol{\alpha}, \boldsymbol{\beta})}(\bw^{(0)})-L^{*} \right)}{T }}. \nonumber
\end{align} 
We observe that it has the same convergence rate $\cO(\frac{1}{\sqrt{T}})$ as \textsc{SGD} for smooth, nonconvex functions~\cite{SIAM_Bottou}.
\end{corollary}

\section{Experiments}\label{sec:experiments}
\begin{figure*}[tb!]\label{fig:adult:cvx:comparison}
     \centering
     \begin{subfigure}[b]{0.24\textwidth}
         \centering
         \includegraphics[width=\textwidth]{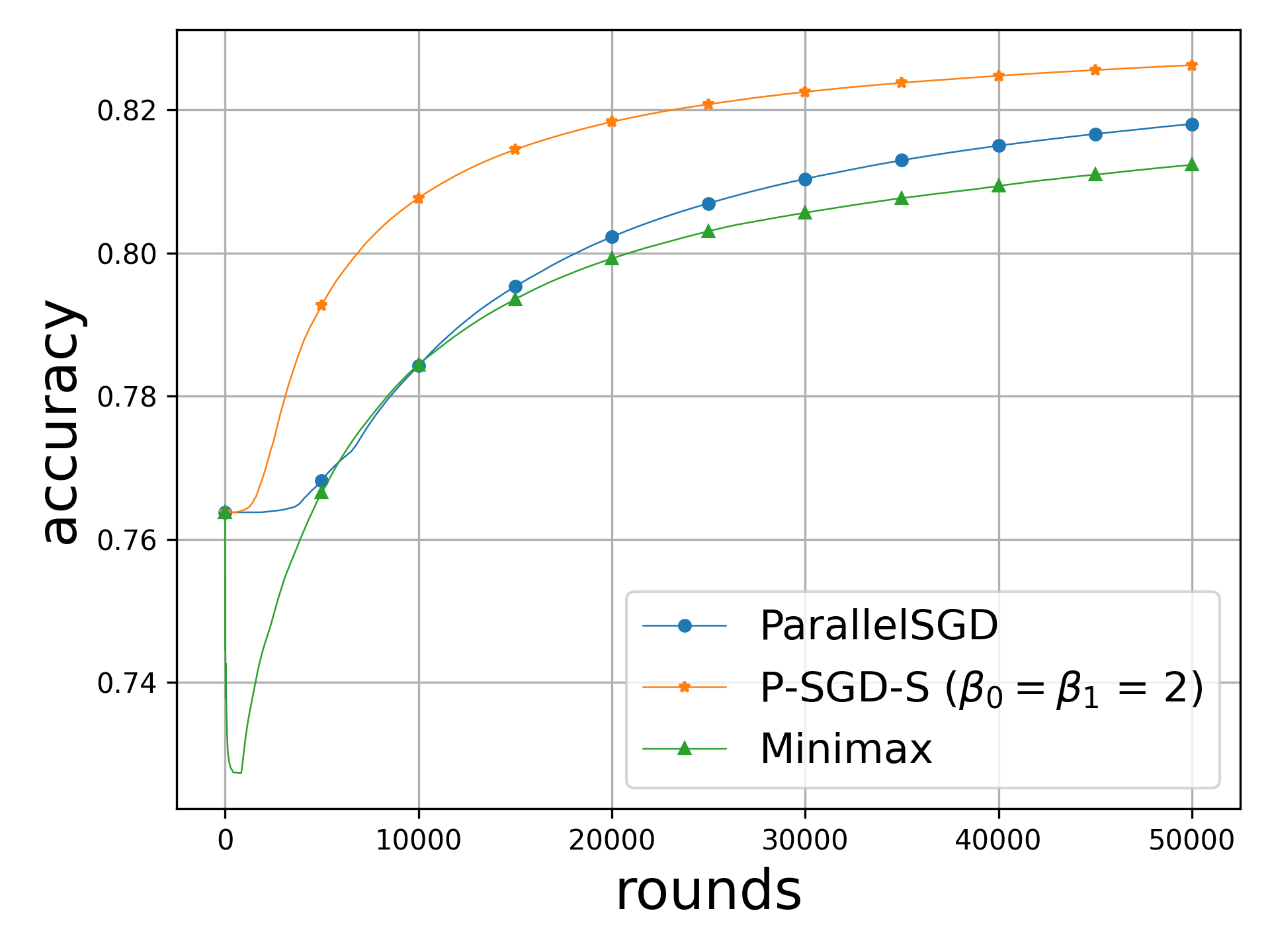}
         \caption{Overall test accuracy}
         \label{fig:adult:cvx:comparison:accuracy}
     \end{subfigure}
     \hfill
     \begin{subfigure}[b]{0.24\textwidth}
         \centering
         \includegraphics[width=\textwidth]{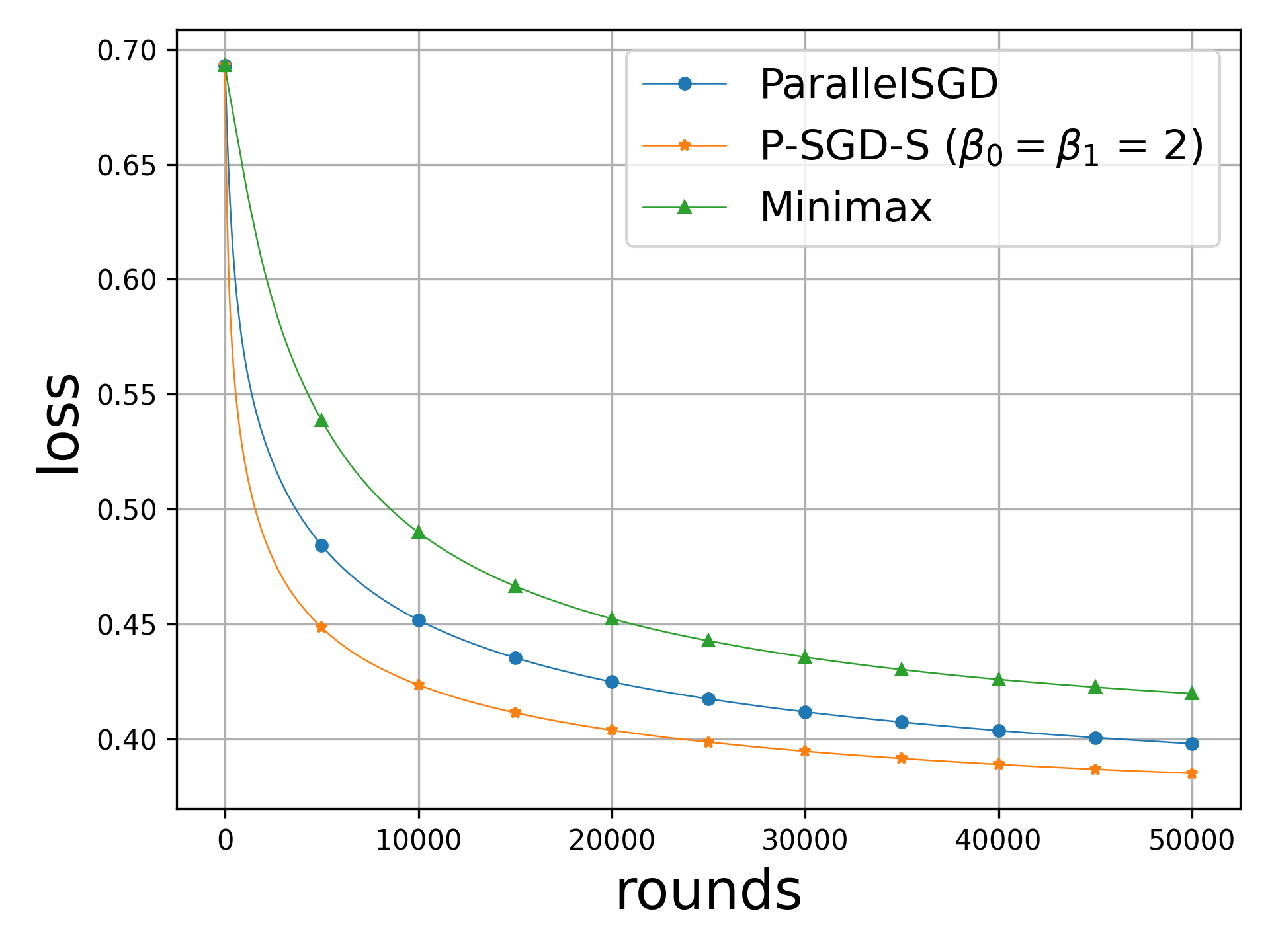}
         \caption{Overall test loss}
         \label{fig:adult:cvx:comparison:loss}
     \end{subfigure}
     \hfill
     \begin{subfigure}[b]{0.24\textwidth}
         \centering
         \includegraphics[width=\textwidth]{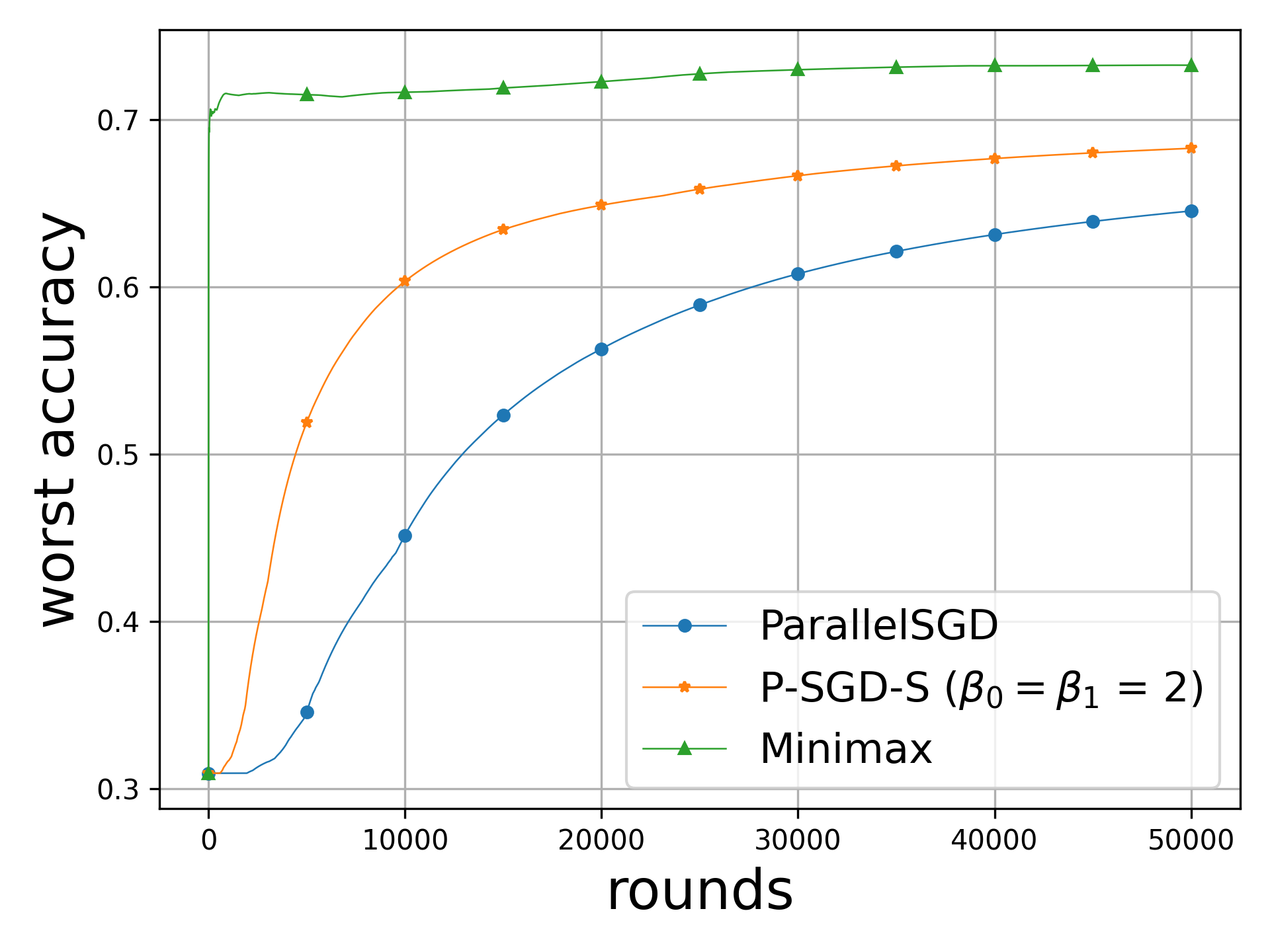}
         \caption{Worst Accuracy}
         \label{fig:adult:cvx:comparison:worstaccuracy}
     \end{subfigure}
     \hfill
     \begin{subfigure}[b]{0.24\textwidth}
         \centering
         \includegraphics[width=\textwidth]{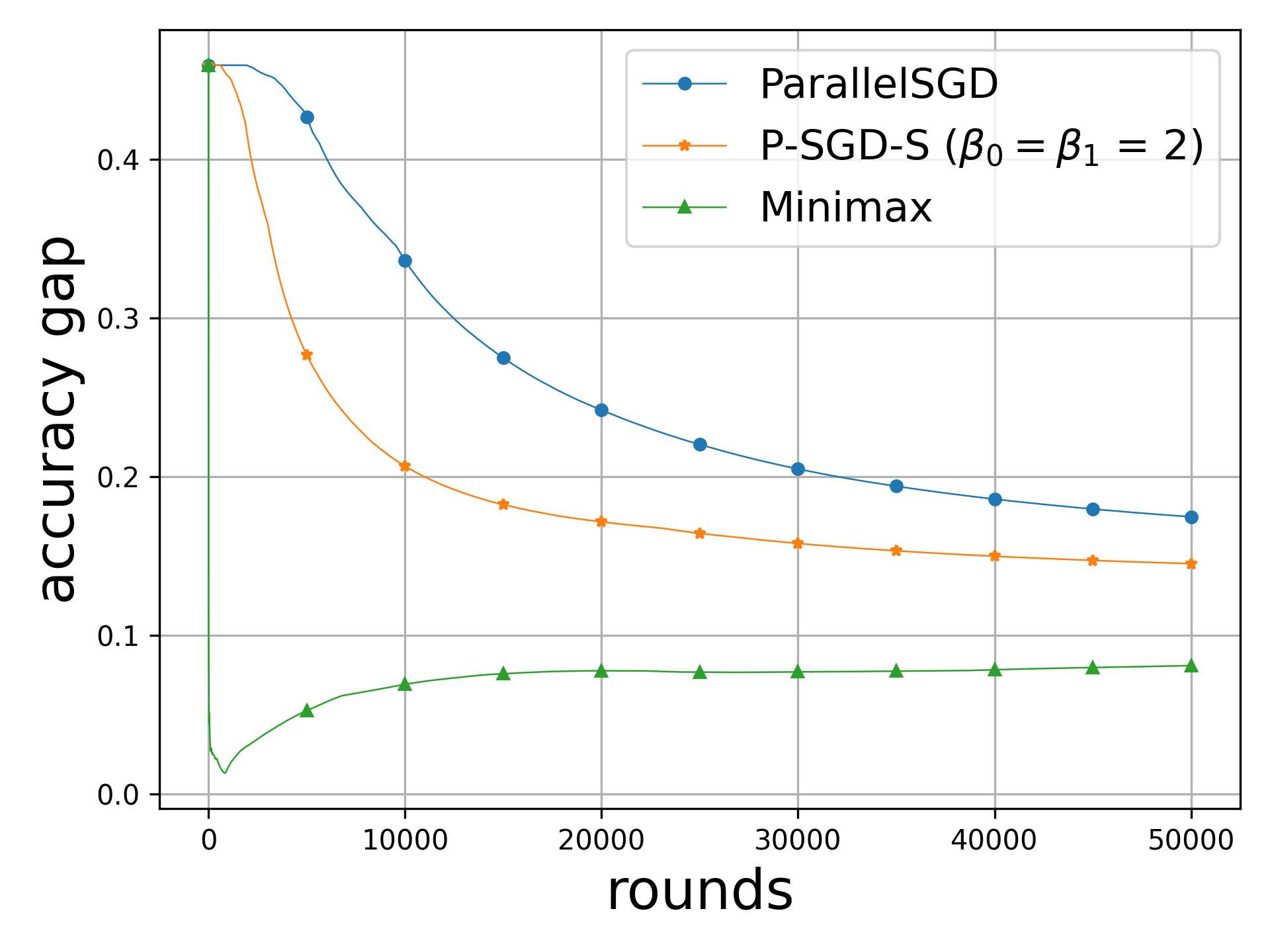}
         \caption{EA violation}
         \label{fig:adult:cvx:comparison:gap}
     \end{subfigure}
        \caption{Comparison of accuracy, loss, worst accuracy, and $\epsilon_{\text{EA}}$  for convex loss on Adult. \textsc{Parallel SGD} operates without explicit fairness constraints, while \textsc{Minimax} enforces an extreme fairness constraint by prioritizing the worst-performing group in group fairness.}
        \label{fig:adult:cvx:comparison}
        \vspace{-0.2cm}
\end{figure*}

\begin{figure*}[tb!]\label{fig:adult:noncvx:comparison}
     \centering
     \begin{subfigure}[b]{0.24\textwidth}
         \centering
         \includegraphics[width=\textwidth]{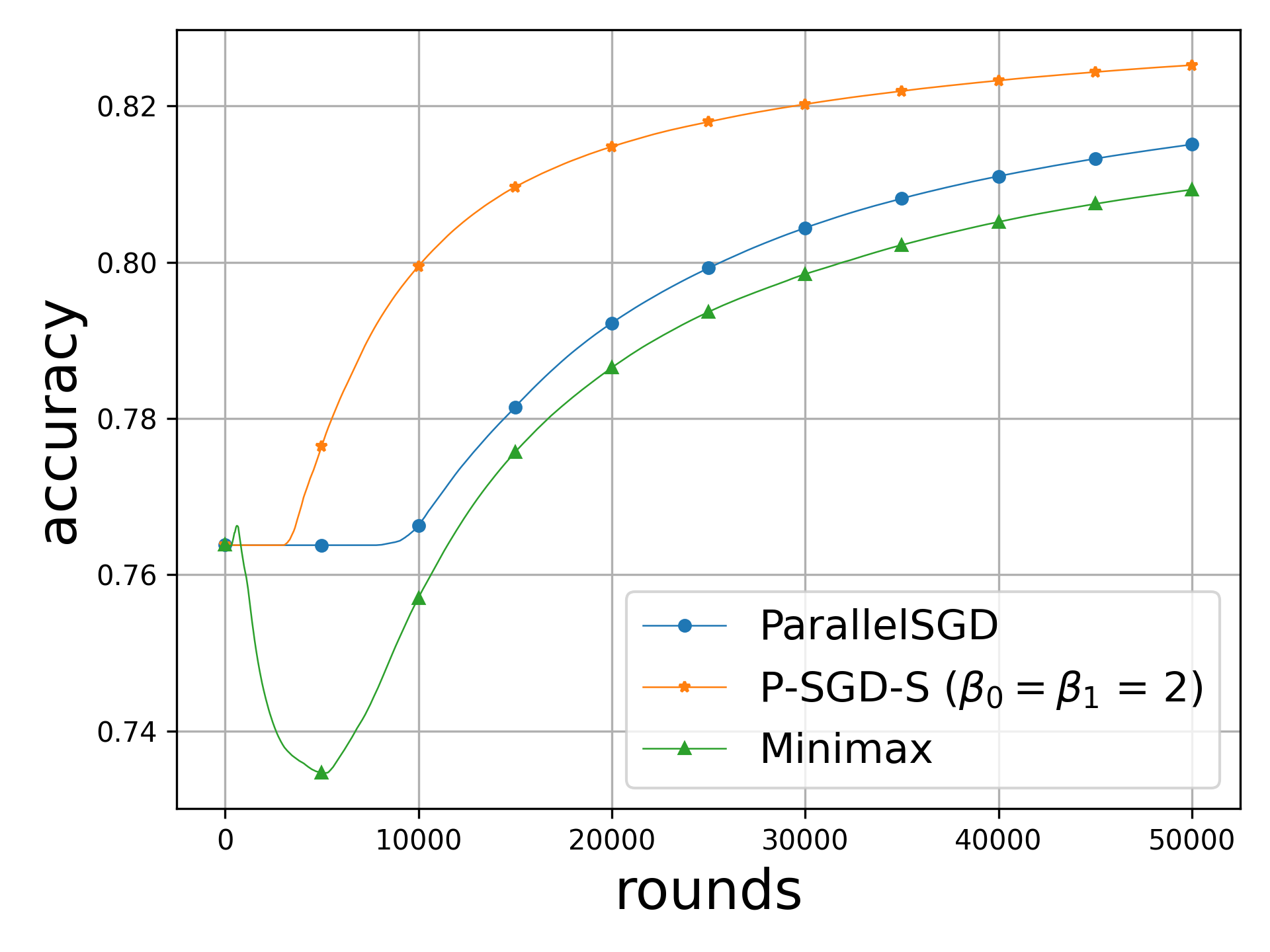}
         \caption{Overall test accuracy}
         \label{fig:adult:noncvx:comparison:accuracy}
     \end{subfigure}
     \hfill
     \begin{subfigure}[b]{0.24\textwidth}
         \centering
         \includegraphics[width=\textwidth]{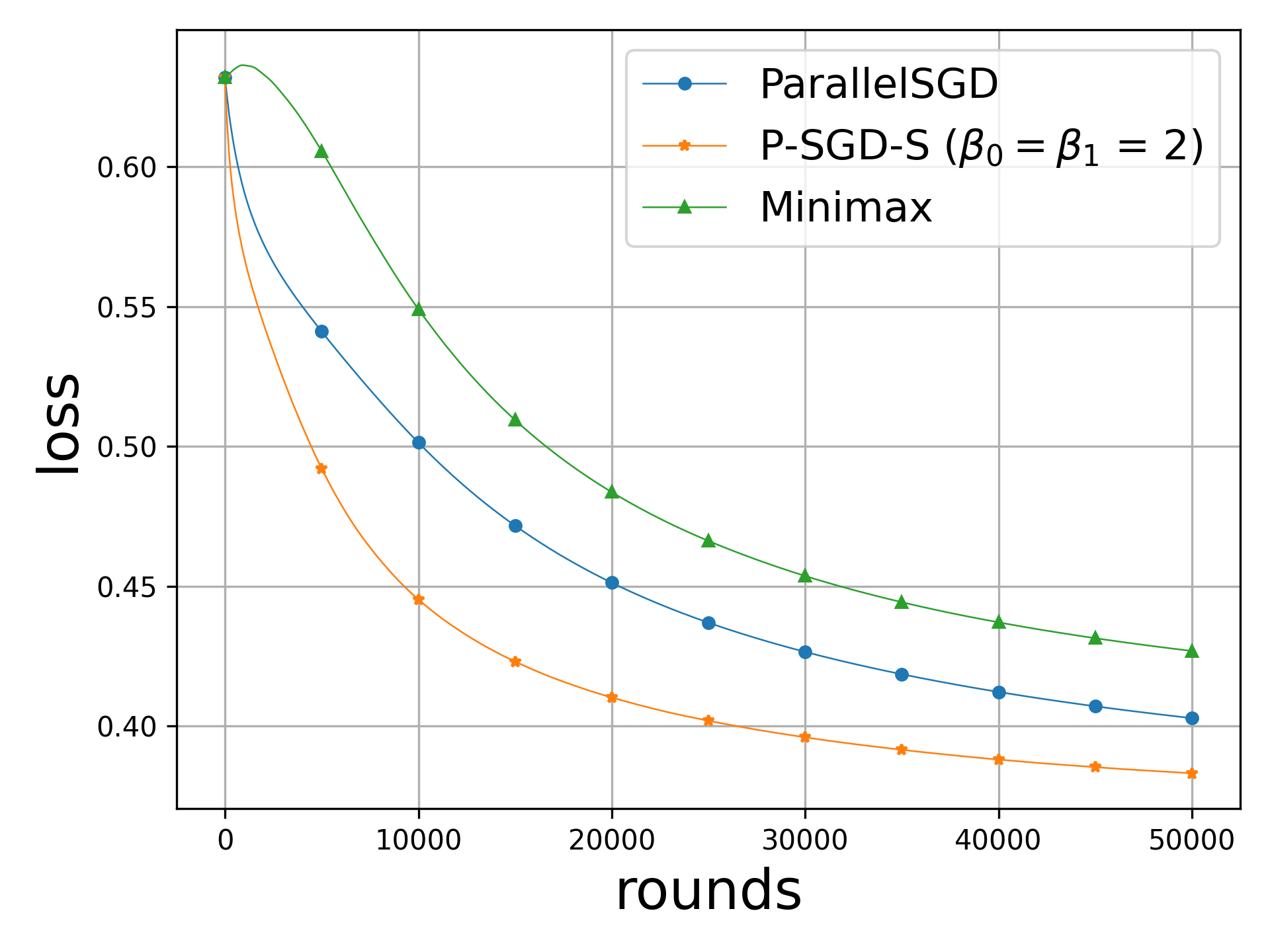}
         \caption{Overall test loss}
         \label{fig:adult:noncvx:comparison:loss}
     \end{subfigure}
     \hfill
     \begin{subfigure}[b]{0.24\textwidth}
         \centering
         \includegraphics[width=\textwidth]{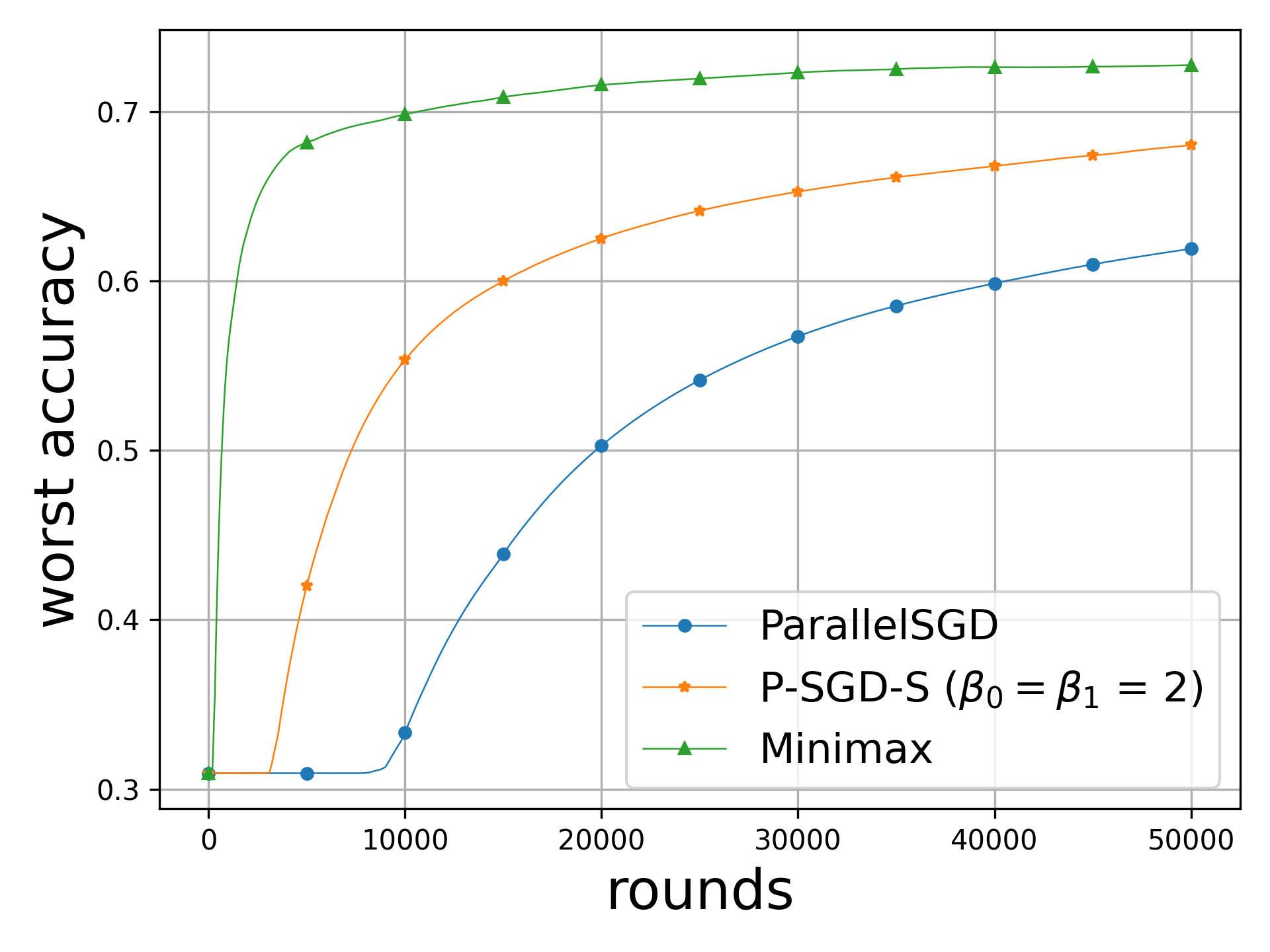}
         \caption{Worst Accuracy}
         \label{fig:adult:noncvx:comparison:worstaccuracy}
     \end{subfigure}
     \hfill
     \begin{subfigure}[b]{0.24\textwidth}
         \centering
         \includegraphics[width=\textwidth]{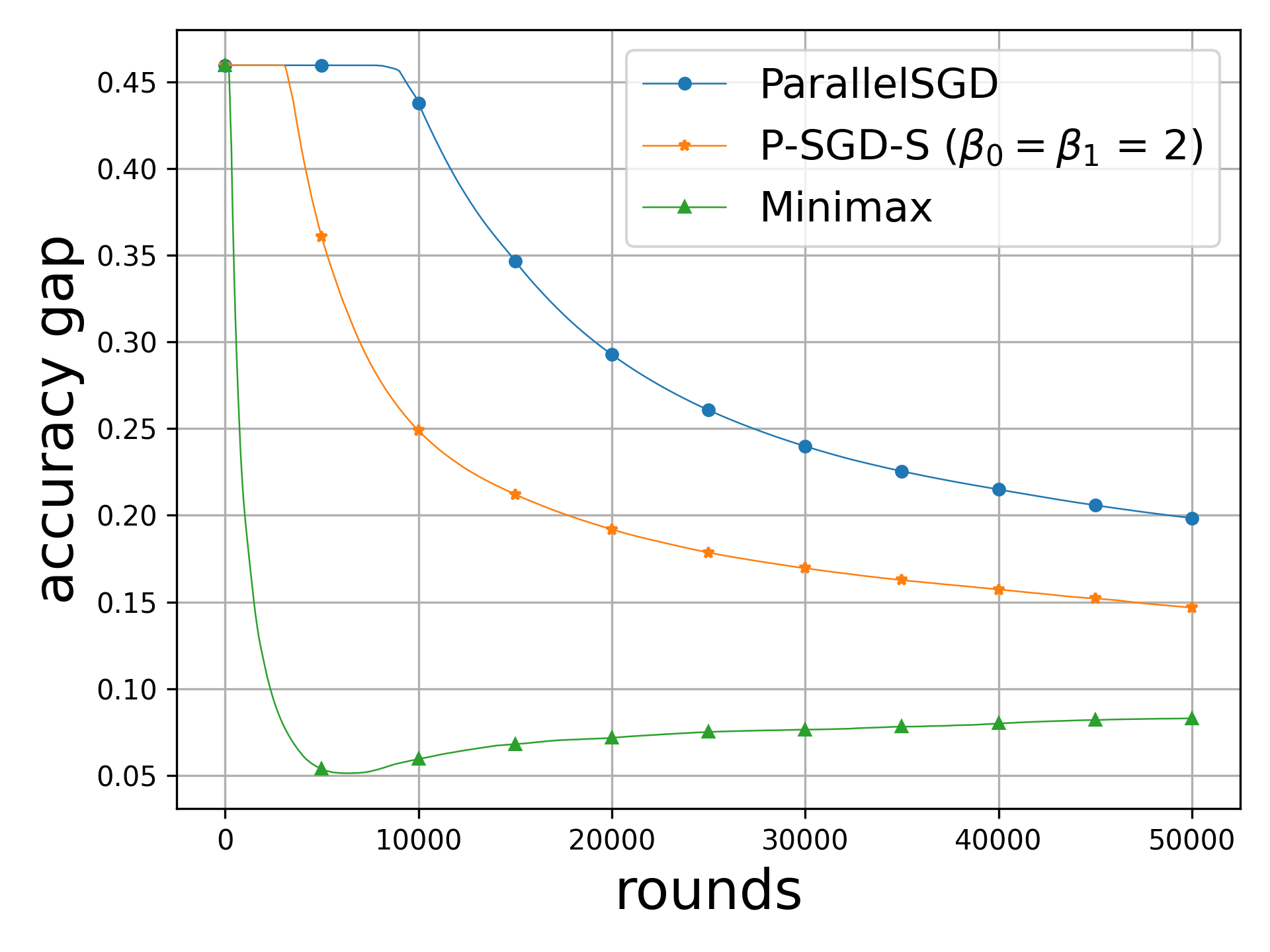}
         \caption{EA violation}
         \label{fig:adult:noncvx:comparison:gap}
     \end{subfigure}
        \caption{Comparison of accuracy, loss, worst accuracy, and $\epsilon_{\text{EA}}$ for nonconvex loss on Adult. \textsc{Parallel SGD} operates without explicit fairness constraints, while \textsc{Minimax} enforces an extreme fairness constraint by prioritizing the worst-performing group in group fairness.}
        \label{fig:adult:noncvx:comparison}
        \vspace{-0.2cm}
\end{figure*}

\begin{figure*}[tb!]
     \centering
     \begin{subfigure}[b]{0.24\textwidth}
         \centering
         \includegraphics[width=\textwidth]{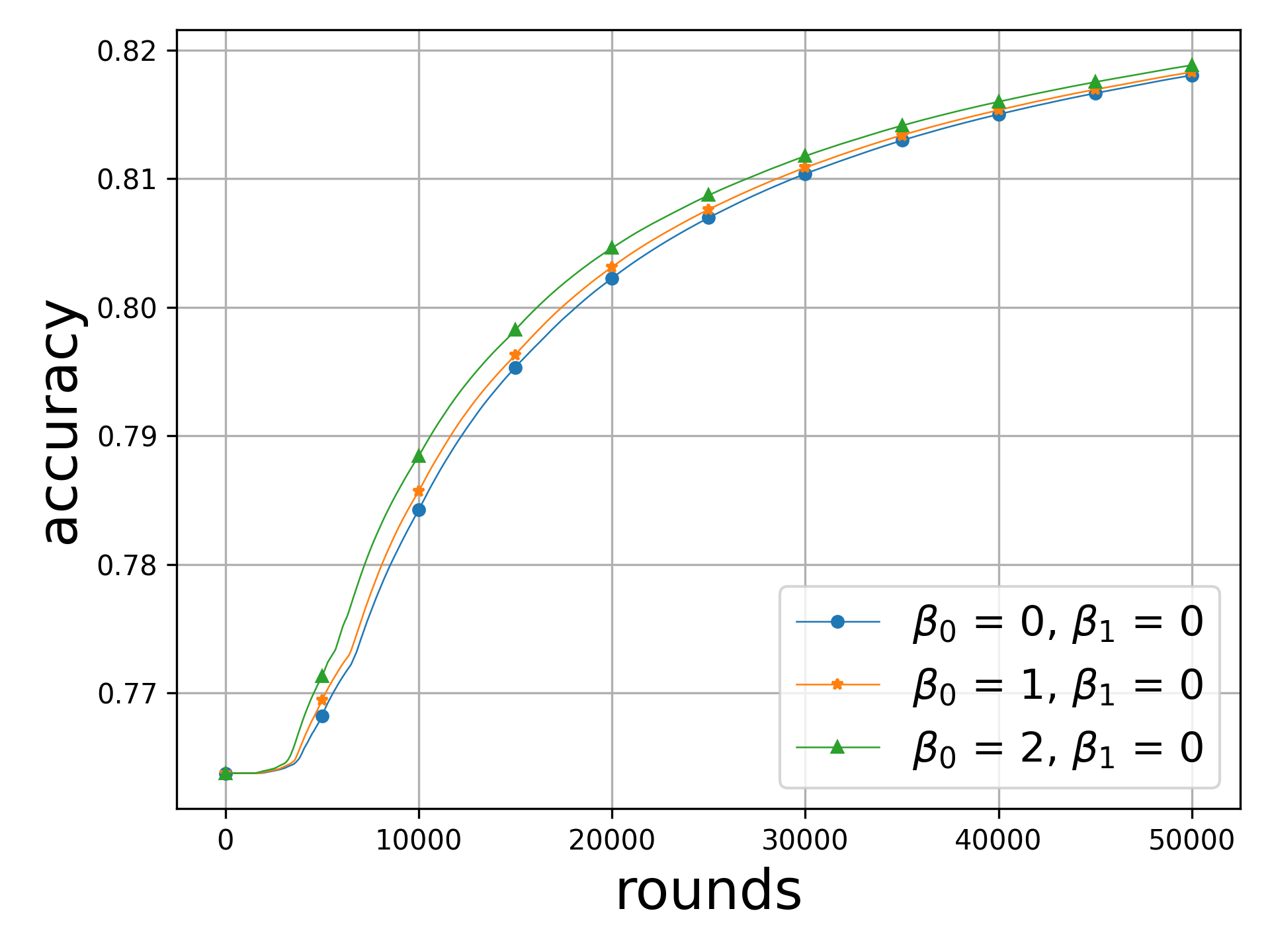}
         \caption{Overall test accuracy}\label{fig:adult:cvx:beta0:accuracy}
     \end{subfigure}
     \hfill
     \begin{subfigure}[b]{0.24\textwidth}
         \centering
         \includegraphics[width=\textwidth]{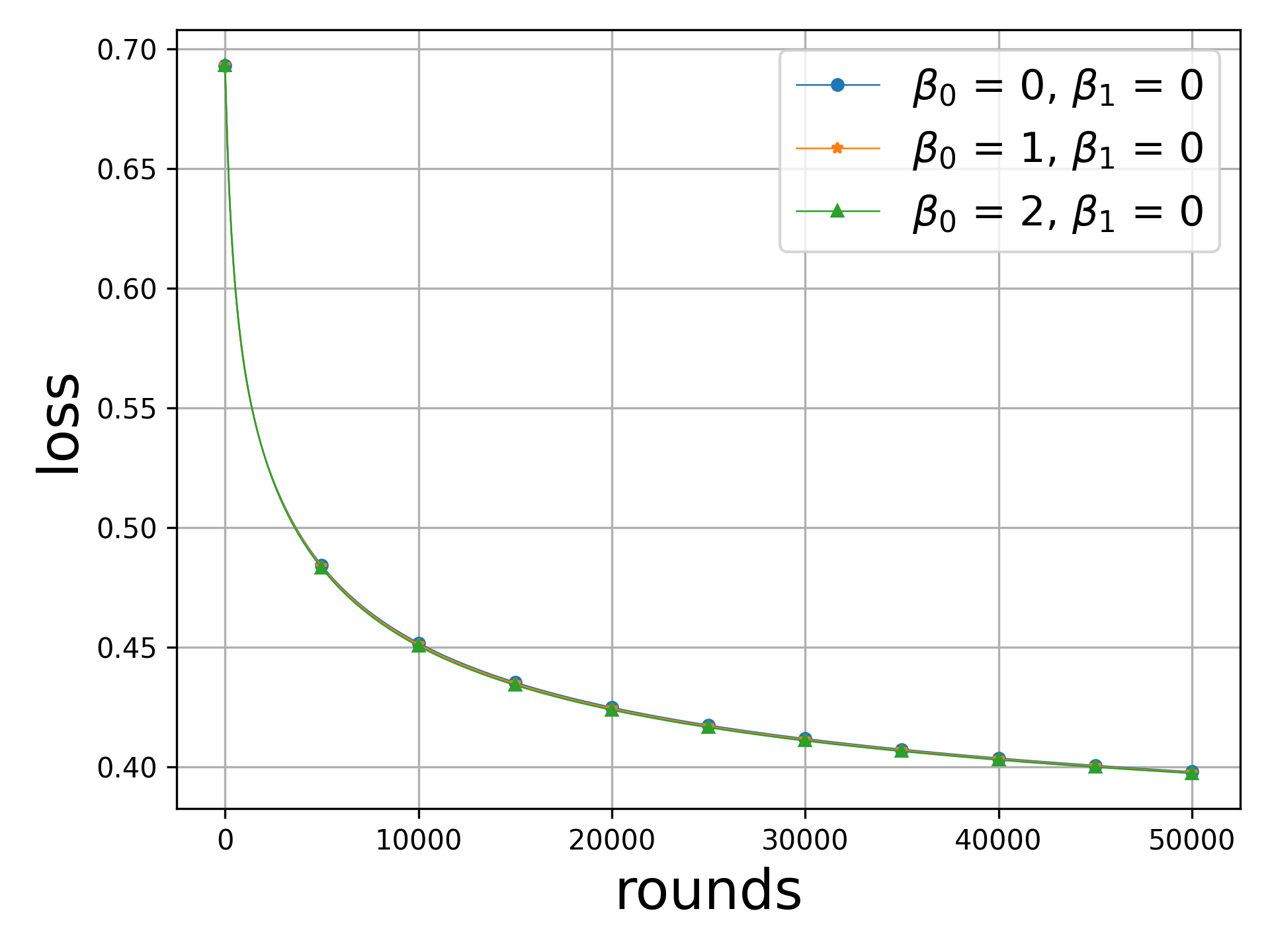}
         \caption{Overall test loss}\label{fig:adult:cvx:beta0:loss}
     \end{subfigure}
     \hfill
     \begin{subfigure}[b]{0.24\textwidth}
         \centering
         \includegraphics[width=\textwidth]{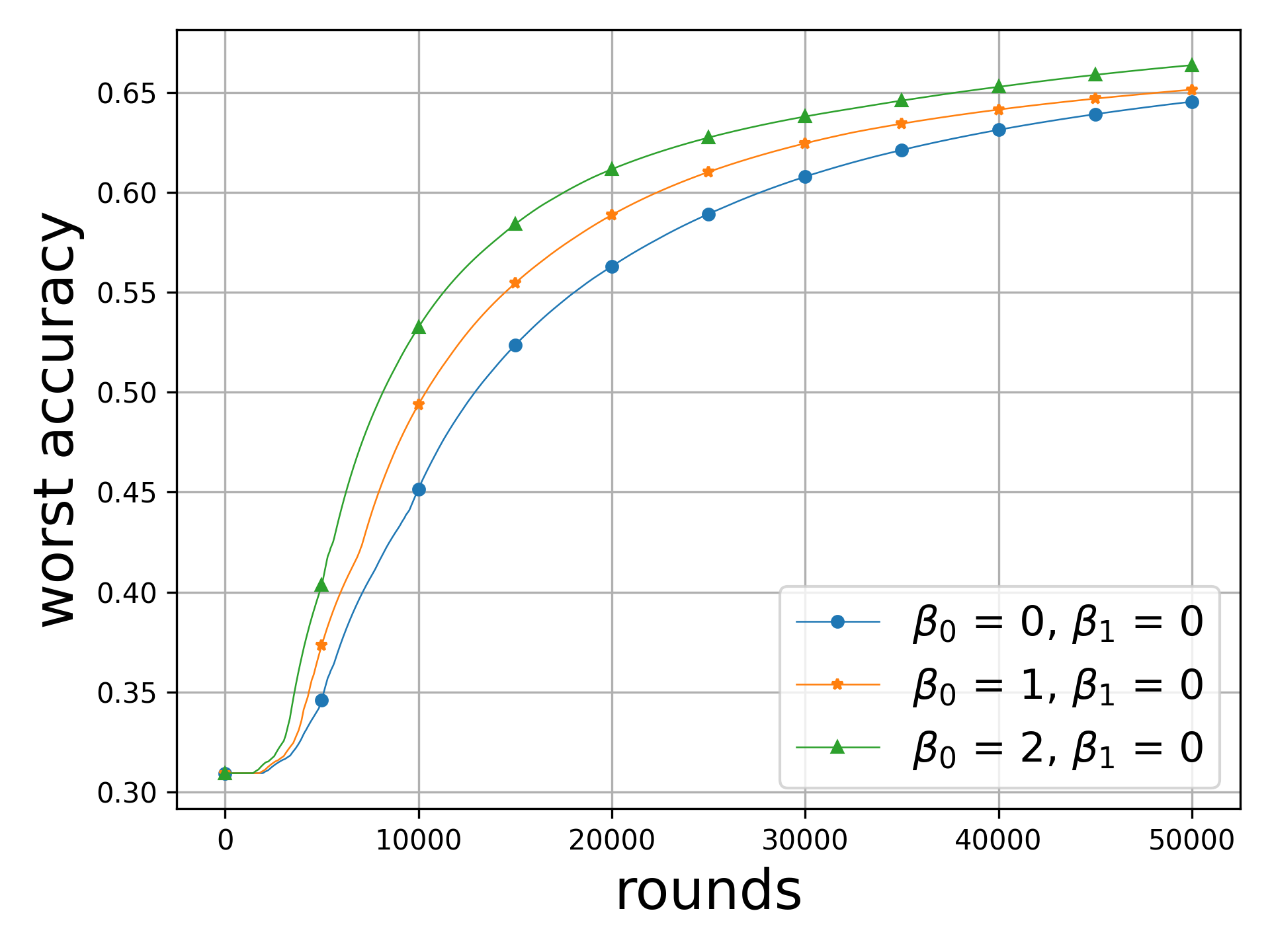}
         \caption{Worst Accuracy}\label{fig:adult:cvx:beta0:worstaccuracy}
     \end{subfigure}
     \hfill
     \begin{subfigure}[b]{0.24\textwidth}
         \centering
         \includegraphics[width=\textwidth]{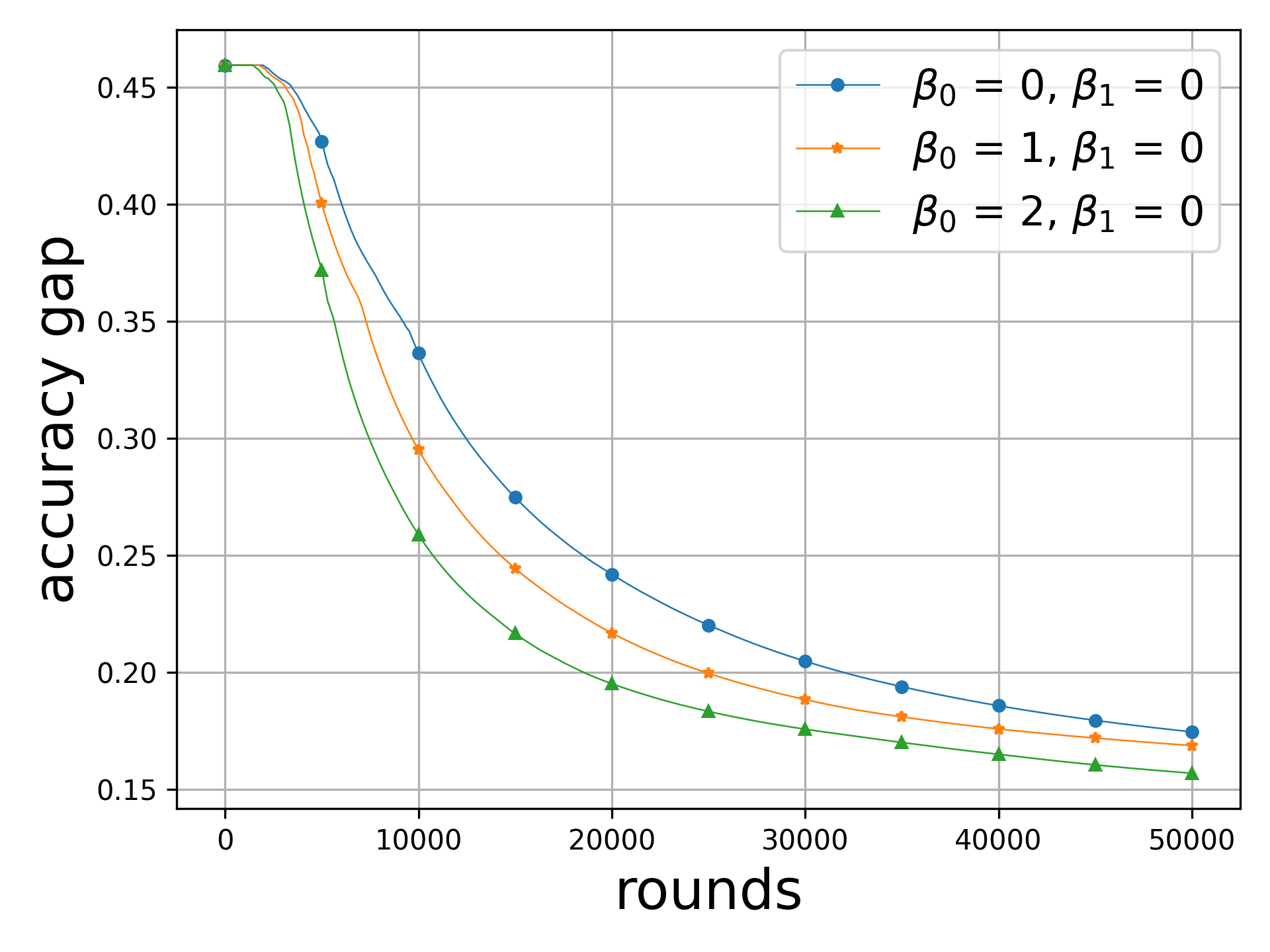}
         \caption{EA violation}\label{fig:adult:cvx:beta0:gap}
     \end{subfigure}
        \caption{Impact of varying $\boldsymbol{\beta}$ values on training dynamics for convex loss on Adult. With $\beta_{1} = 0$, increasing $\beta_{0}$ prioritizes the minority group, improving worst-group accuracy and reducing $\epsilon_{\text{EA}}$.}
        \label{fig:adult:cvx:beta0}
        \vspace{-0.3cm}
\end{figure*}

\subsection{Experimental Settings}

We use the Adult dataset~\cite{UCI} and the COMPAS dataset~\cite{COMPAS}, focusing on education-level disparities in economic opportunity, and the COMPAS dataset~\cite{COMPAS}, examining sex-based disparities in criminal risk assessment, both reflecting real-world sociodemographic inequities. To generalize, we include the Fashion-MNIST dataset~\cite{Fashion_MNIST} in a label-partitioned scenario, where groups are defined by class labels, allowing us to study fairness in settings where disparities arise from label distributions rather than features. Our experiments cover both convex and non-convex classification tasks, solving for the corresponding loss functions. For both datasets models are trained for $50,000$ iterations. %We then evaluate model performance using standard accuracy metrics, alongside fairness metrics, to assess the trade-offs between these objectives across groups.
\subsubsection{Datasets}
\textbf{Adult} dataset includes $48,842$ samples, with $32,561$ training examples and $16,281$ test examples for a salary prediction classification problem in the feature-partitioned scenario. Education level serves as our protected feature, and the data is divided into two groups: Doctorate (\textit{group} 0) with 413 training examples and 181 test examples, and non-Doctorate (\textit{group} 1) with 32,148 training examples and 16,100 test examples. For simplicity, we train the model using only categorical features. The groups are highly imbalanced and we, therefore, choose $\boldsymbol{\alpha}$ based on the ratio of the training examples between the minority and majority groups, i.e., $\alpha_{0} = 0.0127$ and $\alpha_{1} = 0.9873$. We assign the parameters $\beta_{0}$ and $\beta_{1}$ to the Doctorate and non-Doctorate groups, respectively. For the convex case of \textsc{Surrogate-Min}, we train a logistic regression model with cross-entropy loss and for the nonconvex case of \textsc{Surrogate-Min}, we train a multiple-layer neural network model (one hidden layer of $10$ neurons and ReLU as the activation function) with cross-entropy loss. Further details on hyperparameter tuning can be found in Appendix~\ref{appendix:additional}.\\ 
\textbf{COMPAS} dataset contains $6,172$ observations with $14$ features. We adopt the same $2/3$ training and $1/3$ test split used in the Adult dataset.We consider sex as our protected feature and divide data samples into two groups: Females (\textit{group} 0), including 782 training and 393 test samples, and Males (\textit{group} 1), including 3,332 training and 1,665 test samples. Similar to the Adult dataset, we only train using categorical features. The values of $\boldsymbol{\alpha}$ are selected based on the ratio of the training examples between the minority and majority groups, i.e., $\alpha_{0} = 0.1901$ and $\alpha_{1} = 0.8099$. We assign $\beta_{0}$ to the Female group and $\beta_{1}$ to the Male group. For the convex case of \textsc{Surrogate-Min}, we train a logistic regression model with cross-entropy loss and for the nonconvex case of \textsc{Surrogate-Min}, we train a multiple-layer neural network model with further details on hyperparameter shown in Appendix~\ref{appendix:additional}.

\begin{figure*}[tb!]\label{fig:COMPAS:cvx:comparison}
     \centering
     \begin{subfigure}[b]{0.24\textwidth}
         \centering
         \includegraphics[width=\textwidth]{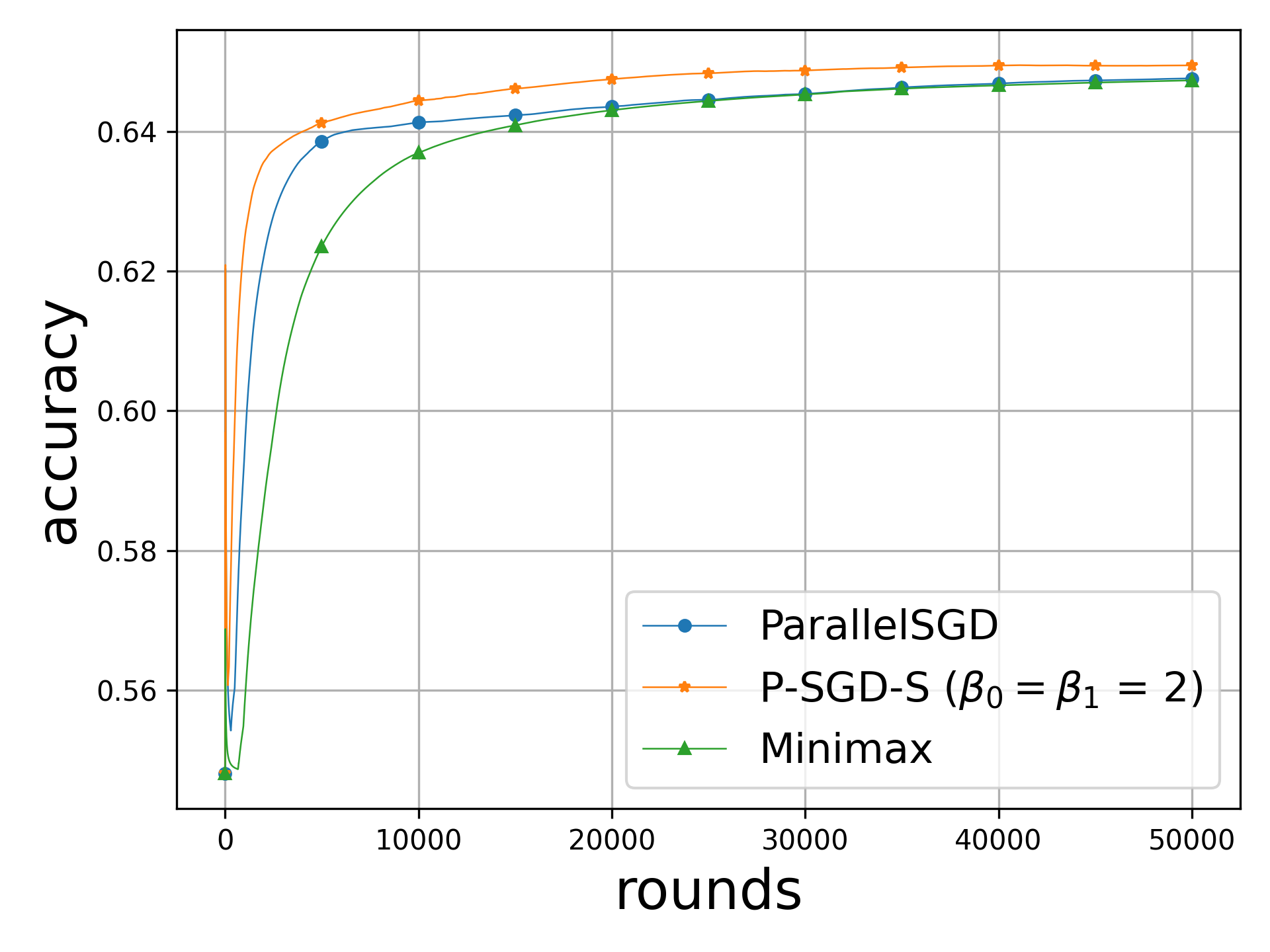}
         \caption{Overall test accuracy}
         \label{fig:COMPAS:cvx:comparison:accuracy}
     \end{subfigure}
     \hfill
     \begin{subfigure}[b]{0.24\textwidth}
         \centering
         \includegraphics[width=\textwidth]{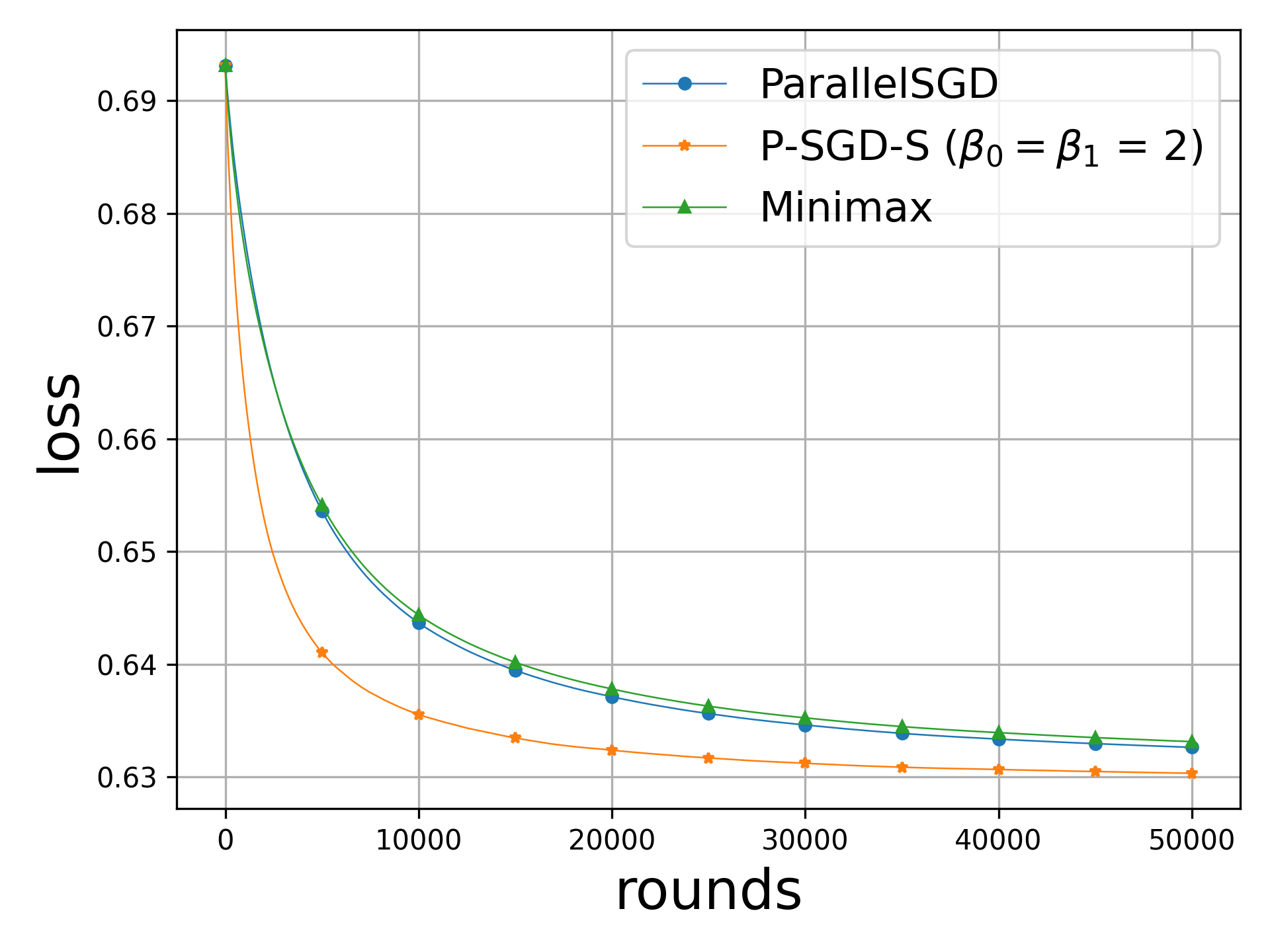}
         \caption{Overall test loss}
         \label{fig:COMPAS:cvx:comparison:loss}
     \end{subfigure}
     \hfill
     \begin{subfigure}[b]{0.24\textwidth}
         \centering
         \includegraphics[width=\textwidth]{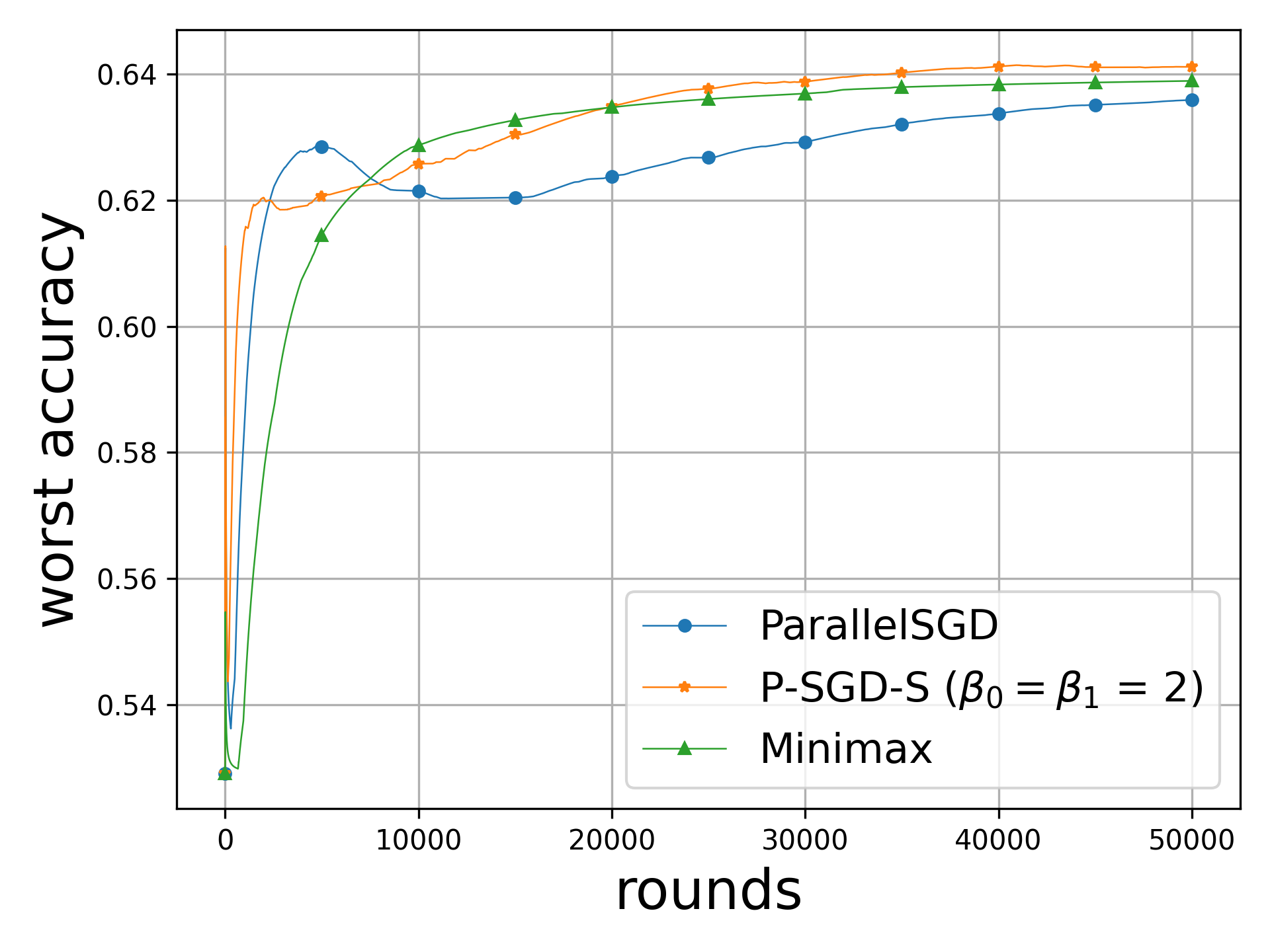}
         \caption{Worst Accuracy}
         \label{fig:COMPAS:cvx:comparison:worstaccuracy}
     \end{subfigure}
     \hfill
     \begin{subfigure}[b]{0.24\textwidth}
         \centering
         \includegraphics[width=\textwidth]{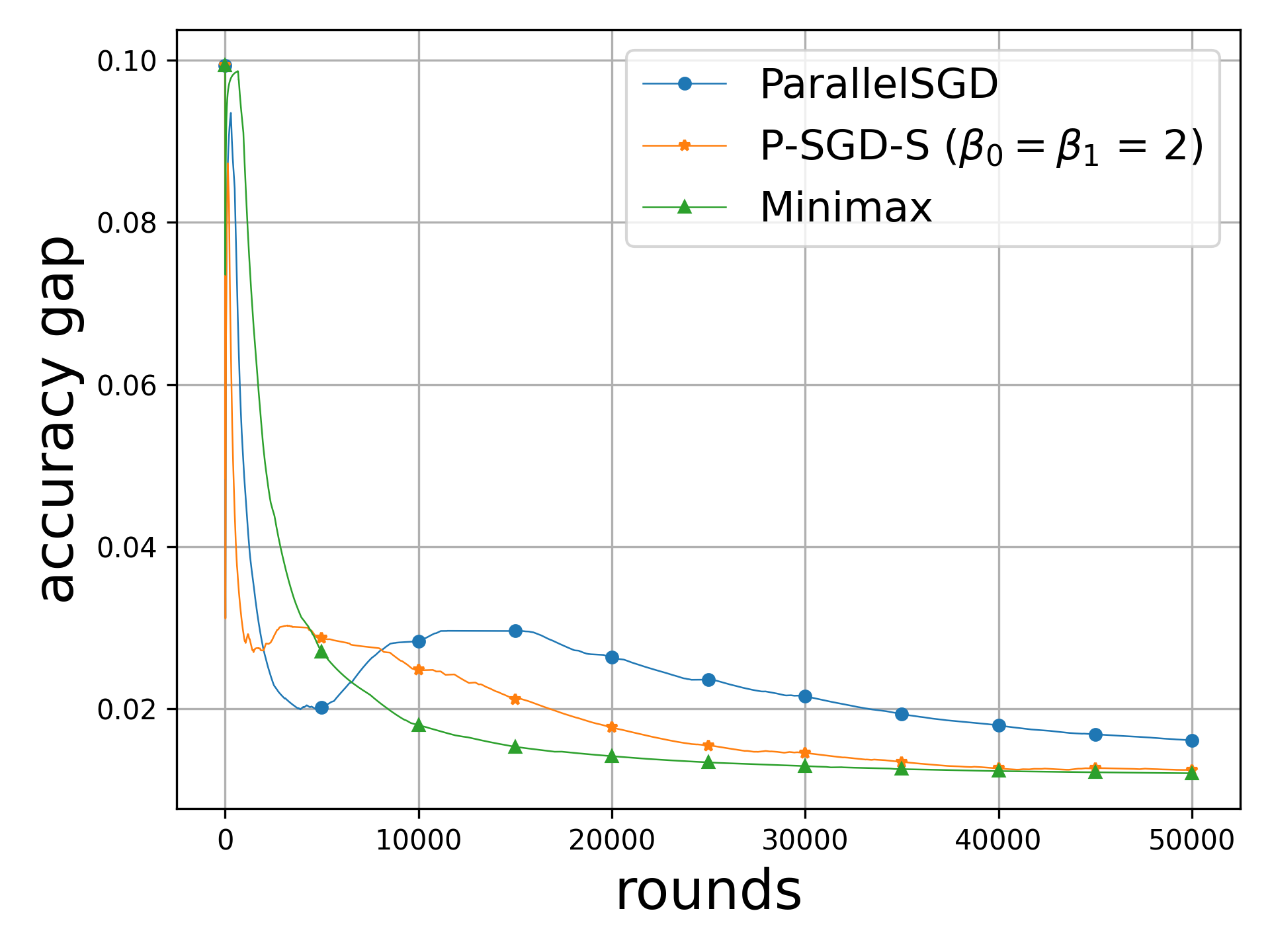}
         \caption{EA violation}
         \label{fig:COMPAS:cvx:comparison:gap}
     \end{subfigure}
        \caption{Comparison of accuracy, loss, worst accuracy, and $\epsilon_{\text{EA}}$ for convex loss on COMPAS. \textsc{Parallel SGD} operates without explicit fairness constraints, while \textsc{Minimax} enforces an extreme fairness constraint by prioritizing the worst-performing group in group fairness.}
        \label{fig:COMPAS:cvx:comparison}
        % \vspace{-0.2cm}
\end{figure*}

\begin{figure*}[tb!]
     \centering
     \begin{subfigure}[b]{0.24\textwidth}
         \centering
         \includegraphics[width=\textwidth]{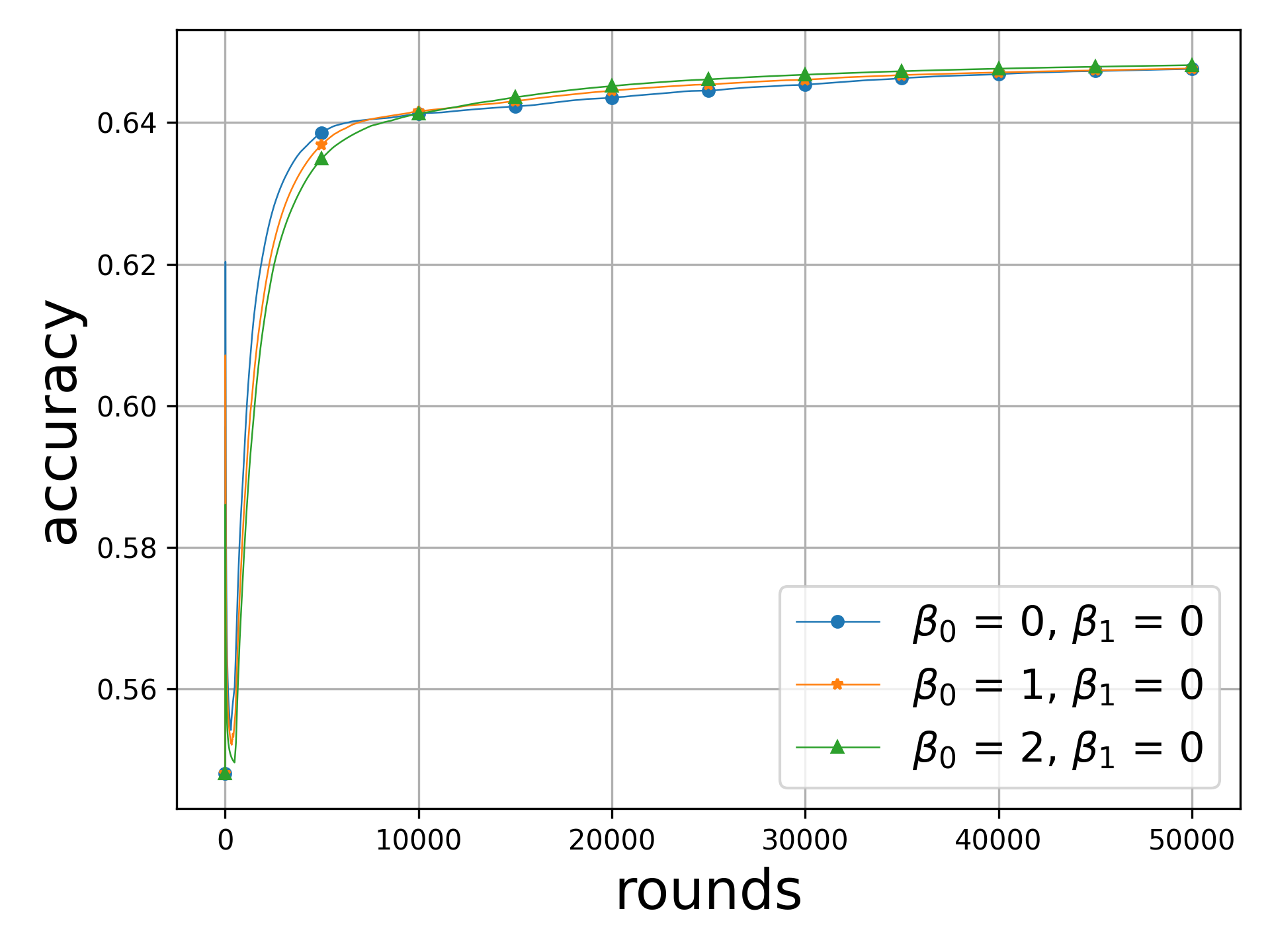}
         \caption{Overall test accuracy}\label{fig:COMPAS:cvx:beta0:accuracy}
     \end{subfigure}
     \hfill
     \begin{subfigure}[b]{0.24\textwidth}
         \centering
         \includegraphics[width=\textwidth]{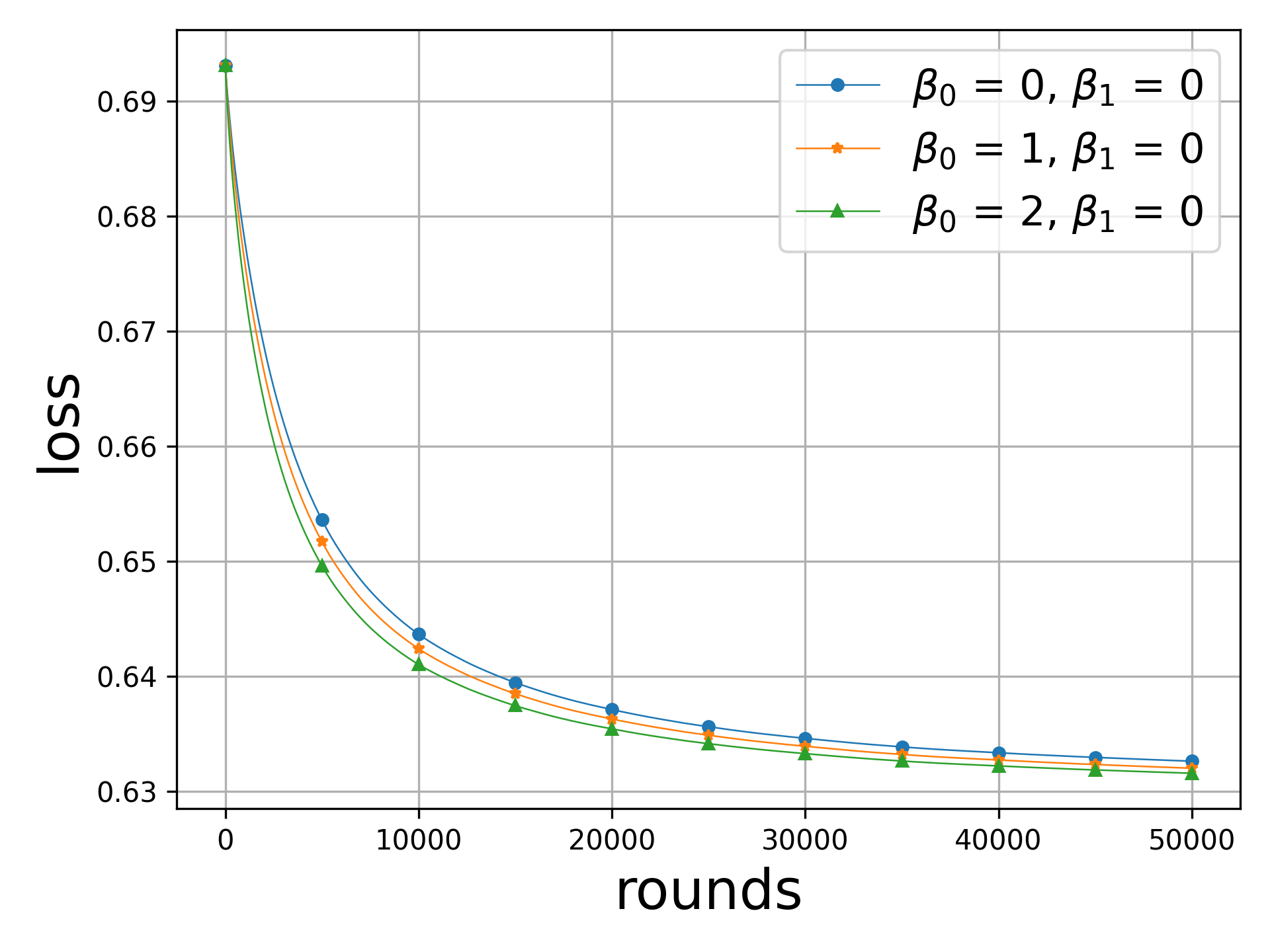}
         \caption{Overall test loss}\label{fig:COMPAS:cvx:beta0:loss}
     \end{subfigure}
     \hfill
     \begin{subfigure}[b]{0.24\textwidth}
         \centering
         \includegraphics[width=\textwidth]{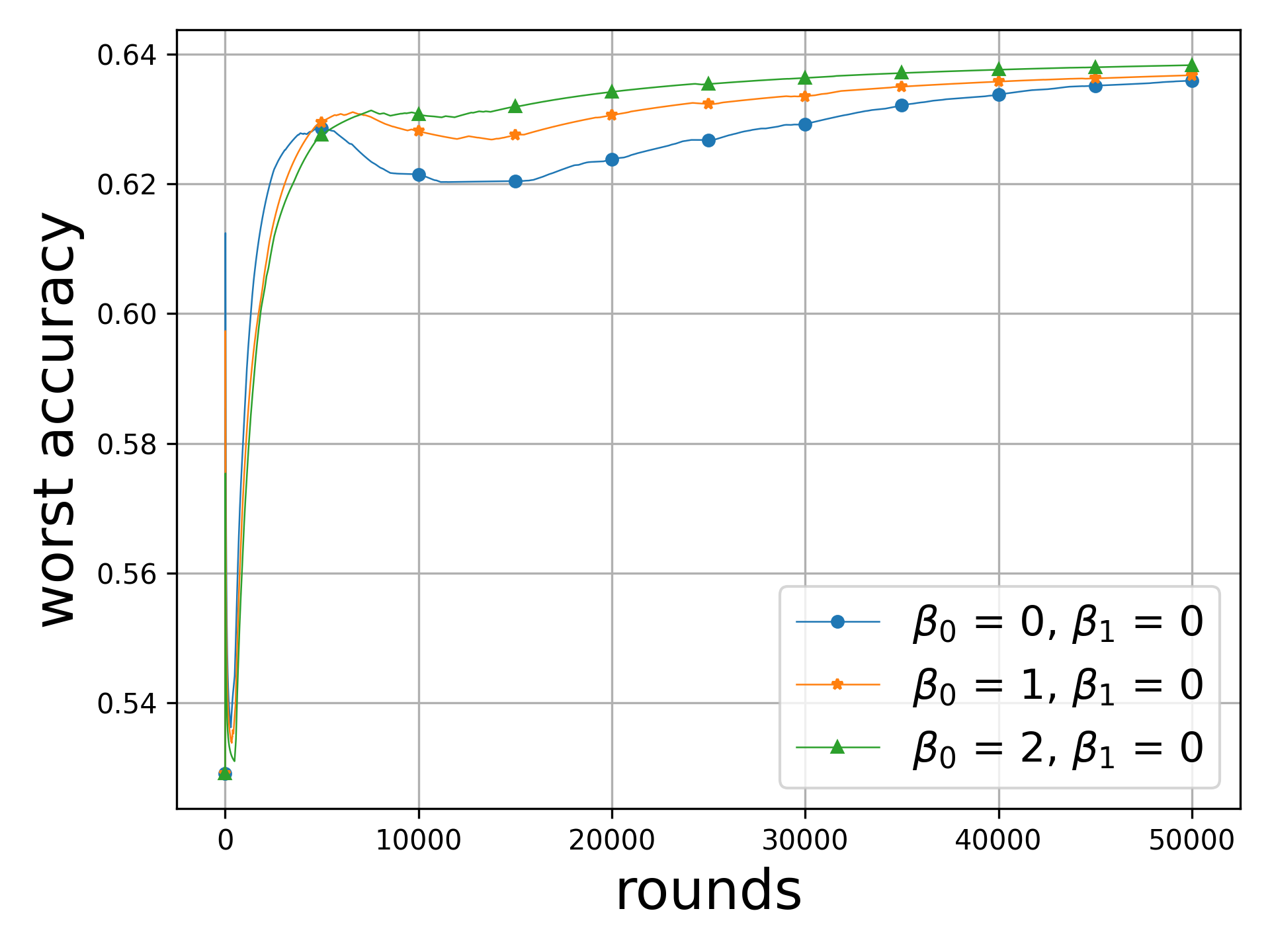}
         \caption{Worst Accuracy}\label{fig:COMPAS:cvx:beta0:worstaccuracy}
     \end{subfigure}
     \hfill
     \begin{subfigure}[b]{0.24\textwidth}
         \centering
         \includegraphics[width=\textwidth]{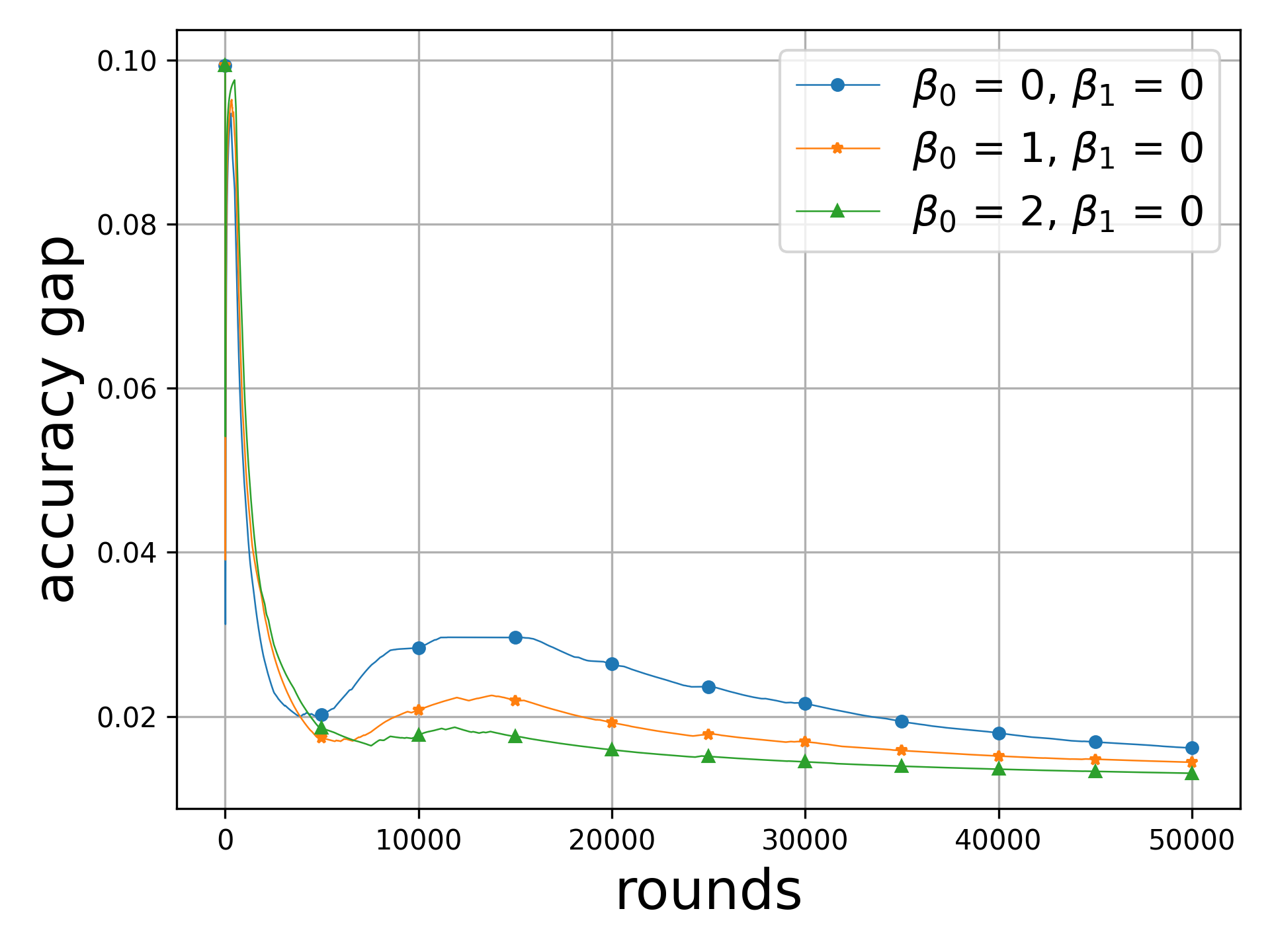}
         \caption{EA violation}\label{fig:COMPAS:cvx:beta0:gap}
     \end{subfigure}
        \caption{Impact of varying $\boldsymbol{\beta}$ values on training dynamics for convex loss on COMPAS. With $\beta_{1} = 0$, increasing $\beta_{0}$ prioritizes the minority group, improving worst-group accuracy and reducing $\epsilon_{\text{EA}}$.}
        % \caption{Comparison of accuracy, loss, worst accuracy, and $\epsilon_{\text{EA}}$ for convex loss on COMPAS with varying $\beta_{0}$ for the minority group.}
        \label{fig:COMPAS:cvx:beta0}
        \vspace{-0.5mm}
\end{figure*}

\subsubsection{Benchmarks}
We consider two classes of benchmarks as follows:
\begin{itemize}
    \setlength\itemsep{0em}
    \item \textsc{Surrogate-Min}: solving the minimization optimization problem in~\eqref{def:alpha-beta:opt} using our proposed \textsc{P-SGD-S} as outlined in Algorithm~\ref{alg:P-SGD-S};
    \item \textsc{Minimax}: solving the minimax optimization problem proposed in~\cite{ICML18_Hashimoto,ICML19_AFL}, i.e., minimizing the loss for the worst mixture of all group distributions, 
    \begin{align}\label{def:minimax:opt}
    \min_{\bw \in \cW} \max_{\boldsymbol{\lambda} \in \Lambda} \sum_{i \in \cI} \lambda_{i}F_{\beta_{i}}(w),
\end{align}
    where $\boldsymbol{\lambda} \in \Lambda \subseteq \Delta_{|\cI|-1}$ and $\beta_{i} = 0, \forall i \in \cI$, using a stochastic gradient descent ascent (SGDA) algorithm.
\end{itemize}
We choose specific values for $\boldsymbol{\alpha}$ and $\boldsymbol{\beta}$ in the \textsc{Surrogate-Min} class of algorithms. The original ERM formulation in~\eqref{def:erm:opt} is a special case of the \textsc{Surrogate-Min} class, where $\alpha_{0} =S_0/N$, $\alpha_{1} =S_1/N$, and $\beta_{0} = \beta_{1} = 0$. For clarity of comparison, we call the corresponding algorithm \textsc{Parallel SGD}. For the \textsc{Minimax} type of algorithms, by solving the minimax optimization problem in~\eqref{def:minimax:opt}, the resulting model tends to have uniform performance with respect to accuracy among different groups, i.e., imposing the most extreme min-max fairness on accuracy. Due to space constraints, the results for the label-partitioned scenario as well as on other fairness notions ~\cite{AISTATS17_Zafar,NIPS16_Hardt} are deferred to Appendix~\ref{appendix:additional}.

\subsubsection{Evaluation}
We evaluate model performance using the following four metrics on the test set: average test accuracy, the worst group accuracy (the lowest accuracy between the two protected groups), and the EA violation.

\subsection{Experimental Results}

\subsubsection{Balanced Fairness with Improved Accuracy}

To evaluate our framework against benchmark methods, we set $\beta_{1} = \beta_{2} = 2$. The results for the convex loss case for Adult dataset are presented in Figure~\ref{fig:adult:cvx:comparison}, while those for the COMPAS dataset are shown in Figure~\ref{fig:COMPAS:cvx:comparison}. For the nonconvex loss case, the corresponding results are provided in Figure~\ref{fig:adult:noncvx:comparison}. Our approach demonstrates superior average accuracy compared to both the \textsc{Parallel SGD} and \textsc{Minimax} (Figure~\ref{fig:adult:cvx:comparison:accuracy} and Figure \ref{fig:COMPAS:cvx:comparison:accuracy}), indicating that our fairness-aware utility function is being properly utilized.  
More importantly, our method improves worst-group accuracy compared to \textsc{Parallel SGD} while remaining slightly lower than \textsc{Minimax} (Figure~\ref{fig:adult:cvx:comparison:worstaccuracy} and Figure \ref{fig:COMPAS:cvx:comparison:worstaccuracy}). This aligns with our expectations, as \textsc{Minimax} explicitly prioritizes worst-case group performance at the cost of overall accuracy, making it a more extreme approach. In contrast, our method achieves a more balanced trade-off, enhancing fairness without imposing overly restrictive constraints on model performance. A similar trend is observed for EA violation, where our approach effectively reduces disparities compared to \textsc{Parallel SGD} but does not enforce as strict a correction as \textsc{Minimax} (Figure~\ref{fig:adult:cvx:comparison:gap} and Figure \ref{fig:COMPAS:cvx:comparison:gap}).  For comparison of other fairness metrics, we provide Table~\ref{tab:comparison:metrics} in Appendix~\ref{appendix:additional}. We have also observed similar trends in the label-partitioned setting using the Fashion-MNIST dataset with the results are provided in Appendix~\ref{appendix:additional}, further supporting the generalizability of our proposed approach.

\subsubsection{Flexibility in Fairness Optimization}
We demonstrate the adaptability of our framework by varying $\boldsymbol{\beta}$ values on training dynamics. The results are shown in Figure~\ref{fig:adult:cvx:beta0} for the Adult dataset and Figure~\ref{fig:COMPAS:cvx:beta0} for COMPAS. Specifically, we fix $\beta_{1} = 0$ and vary $\beta_{0}$, effectively shifting focus toward the minority group. For instance, in the COMPAS dataset, where Females are the minority, adjusting $\beta_{0}$ redistributes optimization focus toward this group (Figures~\ref{fig:adult:cvx:beta0:loss} for Adults and Figures~\ref{fig:COMPAS:cvx:beta0:loss} for COMPAS). As expected, increasing $\beta_{0}$ enhances our fairness metric by improving worst-group accuracy (Figures~\ref{fig:adult:cvx:beta0:worstaccuracy} for Adults and Figures~\ref{fig:COMPAS:cvx:beta0:worstaccuracy} for COMPAS) while also reducing EA violation(Figures~\ref{fig:adult:cvx:beta0:gap} for Adults and Figures~\ref{fig:COMPAS:cvx:beta0:gap} for COMPAS).

These results demonstrate that our approach is highly adaptable and can be tuned to align with different fairness definitions and objectives. By simply adjusting a single parameter $\beta_{0}$, our method offers a smooth fairness-performance trade-off, spanning from \textsc{Parallel SGD} (no explicit fairness constraints) to \textsc{Minimax} fairness (maximizing worst-group performance at the cost of overall accuracy). 

\section{Conclusion}\label{sec:conclusion}
Reweighting techniques for both convex and nonconvex settings, effectively addresses sociodemographic disparities in classification tasks while maintaining strong predictive performance. By tuning fairness parameters, our method provides a flexible trade-off between no fairness constraints and strict fairness enforcement, allowing practitioners to adapt fairness interventions to specific contexts without imposing overly rigid constraints. Future work will extend this framework to intersectional fairness analysis, more complex multimodal datasets, and integration with state-of-the-art deep learning models, including transformers, to further enhance its applicability in real-world decision-making systems.

\newpage
\bibliographystyle{named}
\bibliography{main}
\onecolumn
\appendix
\section*{Supplementary Materials}
\allowdisplaybreaks

\section{Proof of Theorem~\ref{thm:cvx}}\label{appendix:thm1}
\begin{proof}
From the $L_{0}$-smoothness of the surrogate loss, we have
\begin{align}
& L_{(\boldsymbol{\alpha}, \boldsymbol{\beta})}(\bw^{(t+1)}) - L_{(\boldsymbol{\alpha}, \boldsymbol{\beta})}(\bw^{(t)}) \nonumber \\
& \leq \nabla L_{(\boldsymbol{\alpha}, \boldsymbol{\beta})}(\bw^{(t)})^\top (\bw^{(t+1)}-\bw^{(t)}) + \frac{L_{0}}{2}\|\bw^{(t+1)}-\bw^{(t)}\|^{2} \label{thm1:proof:1} \\
& = -\gamma \nabla L_{(\boldsymbol{\alpha}, \boldsymbol{\beta})}(\bw^{(t)})^\top \left(\sum_{i \in \cI} \alpha_{i} \tilde{\nabla} F_{\beta_i}(\bw^{(t)}; \xi_{i}^{(t)})\right) + \frac{L_{0} \gamma^{2}}{2} \left\|\sum_{i \in \cI} \alpha_{i}\tilde{\nabla} F_{\beta_i}(\bw^{(t)}; \xi_{i}^{(t)})\right\|^{2}, \label{thm1:proof:2}
\end{align}
where we substitute in the update rule in~\eqref{thm1:proof:2}. Taking expectations on both sides, we have
\begin{align}
& \bbE\left[L_{(\boldsymbol{\alpha}, \boldsymbol{\beta})}(\bw^{(t+1)})\right] - L_{(\boldsymbol{\alpha}, \boldsymbol{\beta})}(\bw^{(t)}) \nonumber \\
& \leq -\gamma \left\|\nabla L_{(\boldsymbol{\alpha}, \boldsymbol{\beta})}(\bw^{(t)})\right\|^{2} + \frac{L_{0} \gamma^{2}}{2} \bbE\left[\left\|\sum_{i \in \cI} \alpha_{i}\tilde{\nabla} F_{\beta_i}(\bw^{(t)}; \xi_{i}^{(t)}) -\nabla L_{(\boldsymbol{\alpha}, \boldsymbol{\beta})}(\bw^{(t)})+\nabla L_{(\boldsymbol{\alpha}, \boldsymbol{\beta})}(\bw^{(t)})\right\|^{2} \right]  \label{thm1:proof:descent:1} \\
& = (-\gamma+ \frac{L_{0} \gamma^{2}}{2} )\left\|\nabla L_{(\boldsymbol{\alpha}, \boldsymbol{\beta})}(\bw^{(t)})\right\|^{2} + \frac{L_{0} \gamma^{2}}{2} \bbE\left[\left\|\sum_{i \in \cI} \alpha_{i}\tilde{\nabla} F_{\beta_i}(\bw^{(t)}; \xi_{i}^{(t)})-\nabla L_{(\boldsymbol{\alpha}, \boldsymbol{\beta})}(\bw^{(t)})\right\|^{2} \right]\label{thm1:proof:descent:2} \\
& \leq -\frac{\gamma}{2}\left\|\nabla L_{(\boldsymbol{\alpha}, \boldsymbol{\beta})}(\bw^{(t)})\right\|^{2}+\frac{L_{0} \gamma^{2}\sigma^{2}}{2}  \label{thm1:proof:descent:3} \\
& \leq -\frac{\gamma}{2}\left\|\nabla L_{(\boldsymbol{\alpha}, \boldsymbol{\beta})}(\bw^{(t)})\right\|^{2}+\frac{\gamma\sigma^{2}}{2} \label{thm1:proof:descent:4},
\end{align}
where~\eqref{thm1:proof:descent:1} is by Assumption~\ref{assumption:unbiased},~\eqref{thm1:proof:descent:3} is by assuming $\gamma \leq 1/L_{0}$ and Assumption~\ref{assumption:boundedvariance}, and~\eqref{thm1:proof:descent:4} is by assuming $\gamma \leq 1/L_{0}$ again.
From the update rule, we have
\begin{align}
& \bbE\left[\left\|\bw^{(t+1)} - \bw^{*} \right\|^2 \right] \nonumber \\
& =\bbE\left[\left\|\bw^{(t)} - \gamma \left(\sum_{i \in \cI} \alpha_{i} \tilde{\nabla} F_{\beta_i}(\bw^{(t)}; \xi_{i}^{(t)})\right) - \bw^{*} \right \|^2 \right] \label{thm1:proof:main:1} \\
& = \underbrace{\|\bw^{(t)} - \bw^{*}\|^2}_{A^{(t)}} - 2\gamma \nabla L_{(\boldsymbol{\alpha}, \boldsymbol{\beta})}(\bw^{(t)})^\top  (\bw^{(t)} - \bw^{*})  + \gamma^{2} \bbE\left[\left\|\sum_{i \in \cI} \alpha_{i} \tilde{\nabla} F_{\beta_i}(\bw^{(t)}; \xi_{i}^{(t)})\right\|^2\right] \label{thm1:proof:main:2} \\
& \leq A^{(t)} + 2\gamma (L_{(\boldsymbol{\alpha}, \boldsymbol{\beta})}(\bw^{*}) -L_{(\boldsymbol{\alpha}, \boldsymbol{\beta})}(\bw^{(t)})) + \gamma^{2} \bbE\left[\left\|\sum_{i \in \cI} \alpha_{i} \tilde{\nabla} F_{\beta_i}(\bw^{(t)}; \xi_{i}^{(t)})\right\|^2\right] \label{thm1:proof:main:3} \\
& = A^{(t)} \!+\! 2\gamma (L_{(\boldsymbol{\alpha}, \boldsymbol{\beta})}(\bw^{*}) \!-\! L_{(\boldsymbol{\alpha}, \boldsymbol{\beta})}(\bw^{(t)})) \!+\! \gamma^{2} \bbE\bigg[\bigg\|\sum_{i \in \cI} \alpha_{i} \tilde{\nabla} F_{\beta_i}(\bw^{(t)}; \xi_{i}^{(t)})\!-\!\nabla L_{(\boldsymbol{\alpha}, \boldsymbol{\beta})}(\bw^{(t)})\!+\!\nabla L_{(\boldsymbol{\alpha}, \boldsymbol{\beta})}(\bw^{(t)})\bigg\|^2\bigg] \label{thm1:proof:main:4}  \\
& = A^{(t)} + 2\gamma \left(L_{(\boldsymbol{\alpha}, \boldsymbol{\beta})}(\bw^{*}) - L_{(\boldsymbol{\alpha}, \boldsymbol{\beta})}(\bw^{(t)})\right) + \gamma^{2} (\sigma^{2}+\|\nabla L_{(\boldsymbol{\alpha}, \boldsymbol{\beta})}(\bw^{(t)})\|^{2}) \\
& \leq A^{(t)} + 2\gamma \left(L_{(\boldsymbol{\alpha}, \boldsymbol{\beta})}(\bw^{*}) - L_{(\boldsymbol{\alpha}, \boldsymbol{\beta})}(\bw^{(t)})\right)  + \gamma^{2} \left(2\sigma^{2}+\frac{2}{\gamma}\left(L_{(\boldsymbol{\alpha}, \boldsymbol{\beta})}(\bw^{(t)}) - \bbE\left[L_{(\boldsymbol{\alpha}, \boldsymbol{\beta})}(\bw^{(t+1)})\right]\right)\right) \label{thm1:proof:main:5}  \\
& = A^{(t)} + 2\gamma \left(L_{(\boldsymbol{\alpha}, \boldsymbol{\beta})}(\bw^{*}) - \bbE\left[L_{(\boldsymbol{\alpha}, \boldsymbol{\beta})}(\bw^{(t+1)})\right]\right) + 2\gamma^{2} \sigma^{2},\label{thm1:proof:main:6} 
\end{align}
where~\eqref{thm1:proof:main:3} is by convexity of $L_{(\boldsymbol{\alpha}, \boldsymbol{\beta})}$ and~\eqref{thm1:proof:main:5} is by~\eqref{thm1:proof:descent:4}. 
Rearranging terms in~\eqref{thm1:proof:main:6}, we have
\begin{align}
\bbE\left[L_{(\boldsymbol{\alpha}, \boldsymbol{\beta})}(\bw^{(t+1)})\right] - L_{(\boldsymbol{\alpha}, \boldsymbol{\beta})}(\bw^{*}) \leq \frac{1}{2\gamma}\left(A^{(t)} - \bbE\left[A^{(t+1)} \right]\right) + \gamma \sigma^{2}. \label{thm1:proof:main:7} 
\end{align}

Summing~\eqref{thm1:proof:main:7} over $t$ from $0$ to $T-1$, taking total expectation, and dividing both sides by $T$, we have
\begin{align}
\frac{1}{T}\sum_{t=0}^{T-1} \bbE\left[L_{(\boldsymbol{\alpha}, \boldsymbol{\beta})}(\bw^{(t+1)})\right] - L_{(\boldsymbol{\alpha}, \boldsymbol{\beta})}(\bw^{*})
& \leq \frac{1}{2\gamma T}\left(\|\bw^{(0)} - \bw^{*}\|^2 - \bbE\left[\left\|\bw^{(T)} - \bw^{*} \right\|^2 \right]\right) + \gamma \sigma^{2} \\
& \leq \frac{1}{2\gamma T}\left(\|\bw^{(0)} - \bw^{*}\|^2 \right) + \gamma \sigma^{2}.
\end{align}
As $L_{(\boldsymbol{\alpha}, \boldsymbol{\beta})}$ is a convex function, we have
\begin{align}
\bbE\left[L_{(\boldsymbol{\alpha}, \boldsymbol{\beta})}\left(\frac{1}{T} \sum_{t=0}^{T-1} \bw^{(t)}\right)\right] - L(\bw^{*}) \leq \frac{1}{T}\sum_{t=0}^{T-1} \bbE\left[L_{(\boldsymbol{\alpha}, \boldsymbol{\beta})}(\bw^{(t+1)})\right] - L_{(\boldsymbol{\alpha}, \boldsymbol{\beta})}(\bw^{*}),
\end{align}
which completes the proof.
\end{proof}

\section{Proof of Theorem~\ref{thm:noncvx}}\label{appendix:thm2}
\begin{proof}
We derive the bound using the same proof in the convex case up to~\eqref{thm1:proof:descent:3}.
\begin{align}
\bbE\left[L_{(\boldsymbol{\alpha}, \boldsymbol{\beta})}(\bw^{(t+1)})\right] - L_{(\boldsymbol{\alpha}, \boldsymbol{\beta})}(\bw^{(t)}) 
& \leq -\frac{\gamma}{2}\left\|\nabla L_{(\boldsymbol{\alpha}, \boldsymbol{\beta})}(\bw^{(t)})\right\|^{2}+\frac{L_{0} \gamma^{2}\sigma^{2}}{2}.  \label{proof:noncvx:start}
\end{align}
Summing~\eqref{proof:noncvx:start} over $t$ from $0$ to $T-1$, taking total expectation, and dividing both sides by $T$, we have
\begin{align}
\bbE\left[L_{(\boldsymbol{\alpha}, \boldsymbol{\beta})}(\bw^{(T)})\right] - L_{(\boldsymbol{\alpha}, \boldsymbol{\beta})}(\bw^{(0)}) 
& \leq -\frac{\gamma}{2T}\sum_{t=0}^{T-1}\left\|\nabla L_{(\boldsymbol{\alpha}, \boldsymbol{\beta})}(\bw^{(t)})\right\|^{2}+\frac{L_{0} \gamma^{2}\sigma^{2}}{2}. 
\end{align}
Rearranging terms, we have
\begin{align}
\frac{1}{T}\sum_{t=0}^{T-1}\left\|\nabla L_{(\boldsymbol{\alpha}, \boldsymbol{\beta})}(\bw^{(t)})\right\|^{2}
& \leq  \frac{2}{\gamma T}\left(L_{(\boldsymbol{\alpha}, \boldsymbol{\beta})}(\bw^{(0)}) -\bbE\left[L_{(\boldsymbol{\alpha}, \boldsymbol{\beta})}(\bw^{(T)})\right] \right)
+L_{0} \gamma  \sigma^{2} \\
& \leq \frac{2}{\gamma T}\left(L_{(\boldsymbol{\alpha}, \boldsymbol{\beta})}(\bw^{(0)})-L^{*} \right)
+L_{0} \gamma \sigma^{2},
\end{align}
where $L^{*} = \min_{\bw \in \cW} L_{(\boldsymbol{\alpha}, \boldsymbol{\beta})}(\bw)$.
\end{proof}

\section{Additional Experiments}\label{appendix:additional}
In this section, we provide additional experiments that are omitted in the main text due to space limits.
\subsection{Other fairness notions}\label{appendix:fairness}
We consider other popular fairness notions.
\begin{definition}[Demographic Parity (DP)~\cite{AISTATS17_Zafar}]\label{def:dp}
A model satisfies demographic parity if 
\begin{align}
    P(\hat{Y}=1|A=0) = P(\hat{Y}=1|A=1).
\end{align}
\end{definition}

\begin{definition}[Equality of Opportunity (EO)~\cite{NIPS16_Hardt}]\label{def:eo}
A model satisfies equality of opportunity if
\begin{align}
P(\hat{Y}=1|Y=1, A=0) = P(\hat{Y}=1|Y=1, A=1).
\end{align}
\end{definition}
Definition~\ref{def:dp} is satisfied when the predicted outcomes have equal probability between these two groups. Definition~\ref{def:eo} to hold, the average predicted outcomes among individuals with a true label of $Y=1$ must be identical across the two groups determined by the protected feature $A$. Here, we evaluate the learned model's fairness by measuring the violation of EA, DP, and EO properties. The violation of EA is already defined in Definition~\ref{def:ea:violation} and the violation of DP and EO are defined in Definition~\ref{def:violations}.
\begin{definition}[Violations of Fairness Metrics]\label{def:violations}
The DP violation is defined as 
\begin{align}\label{def:dp:violation}
    \epsilon_{\text{DP}} = |P(\hat{Y}=1|A=0) - P(\hat{Y}=1|A=1)|,
\end{align}
and the EO violation $\epsilon_{\text{EO}}$ is defined as 
\begin{align}\label{def:eo:violation}
     \epsilon_{\text{EO}} = |P(\hat{Y}=1|Y=1, A=0) - P(\hat{Y}=1|Y=1, A=1)|.
\end{align}
\end{definition}
Based on the definitions of $\epsilon_{\text{EA}}$, $\epsilon_{\text{DP}}$, $\epsilon_{\text{EO}}$, it is clear that a lower value for each metric means a greater degree of fairness imposed by the model.

The results of different fairness metrics for both convex and nonconvex loss of various $\boldsymbol{\beta}$ under minimization optimization are shown in Table~\ref{tab:comparison:metrics}. It is evident that for both convex and nonconvex loss, with higher $\boldsymbol{\beta}$, the learned model tends to have a smaller EA violation, yet larger DP violation and EO violation in Table~\ref{tab:comparison:metrics}. This observation matches with the argument in~\cite{FairML_book} that those fairness metrics can be contradictory. In the extreme case of \textsc{Minimax}, the EA violation is very low, meaning the accuracy of each group is almost uniform. Our \textsc{Surrogate Min} class of algorithms with different values of $\boldsymbol{\beta}$ can obtain different levels of fairness with respect to those fairness metrics between the vanilla ERM and \textsc{Minimax}.

\begin{table*}[t]
  \caption{Comparison on different fairness metrics for convex and nonconvex loss on Adult.}
  \label{tab:comparison:metrics}
  \centering
  \begin{tabular}{lllllll}
    \toprule
    & \multicolumn{3}{c}{Convex Loss} & \multicolumn{3}{c}{Nonconvex Loss}  \\
     \cmidrule(r){2-4} \cmidrule(r){5-7}
     & $\epsilon_{\text{EA}} \downarrow$ & $\epsilon_{\text{DP}} \downarrow$  & $\epsilon_{\text{EO}} \downarrow$ & $\epsilon_{\text{EA}} \downarrow$ & $\epsilon_{\text{DP}} \downarrow$  & $\epsilon_{\text{EO}} \downarrow$ \\
    \midrule
    \textsc{Surrogate-Min} & & & \\
     $\beta_{0} = \beta_{1} = 0$ & $ 0.1747$ & $\boldsymbol{0.2846}$       & $\boldsymbol{0.1043}$ & $0.1983$ & $\boldsymbol{0.3526}$       & $\boldsymbol{0.1564}$ \\

     $\beta_{0} = \beta_{1} = 1$  & $0.1589$   & $0.3147$  & $0.1309$ & $0.1712$   & $0.3922$ & $0.1953$   \\
    $\beta_{0} = \beta_{1} = 2$   & $0.1452$   & $0.3451$      & $0.1596$  & $0.1466$   & $0.4248$      & $0.2295$ \\    
    % $\beta_{0} = \beta_{1} = 3$   & $0.1334$   & $0.3740$     & $0.1878$  &    &    &  \\    
    \midrule
    \textsc{Minimax} 
    & $\boldsymbol{0.0809}$   & $0.4827$      & $0.3216$ & $\boldsymbol{0.0829}$   & $0.4949$      & $0.3179$\\
    % \midrule
    % \textsc{SearchFair} \cite{ICML20_Lohaus} &   &    & &  & & \\
    % \midrule
    % \textsc{FairSurrogates} \cite{NeurIPS21_Bendekgey} &   &    & &  & & \\
    % \midrule
    % \textsc{xx} \cite{TMLR24_Yao} &   &    & &  & & \\
    \bottomrule
  \end{tabular}
\end{table*}

% \begin{table*}[t]
%   \caption{Comparison on different fairness metrics for convex and nonconvex loss on Adult.}
%   \label{tab:comparison:metrics}
%   \centering
%   \begin{tabular}{lllllll}
%     \toprule
%     & \multicolumn{3}{c}{Convex Loss} & \multicolumn{3}{c}{Nonconvex Loss}  \\
%      \cmidrule(r){2-4} \cmidrule(r){5-7}
%      & $\epsilon_{\text{EA}} \downarrow$ & $\epsilon_{\text{DP}} \downarrow$  & $\epsilon_{\text{EO}} \downarrow$ & $\epsilon_{\text{EA}} \downarrow$ & $\epsilon_{\text{DP}} \downarrow$  & $\epsilon_{\text{EO}} \downarrow$ \\
%     \midrule
%     \textsc{Surrogate-Min} & & & \\
%      $\beta_{0} = \beta_{1} = 0$ & $ 0.1747$ & $\boldsymbol{0.2846}$       & $\boldsymbol{0.1043}$ & $0.1983$ & $\boldsymbol{0.3526}$       & $\boldsymbol{0.1564}$ \\

%      $\beta_{0} = \beta_{1} = 1$  & $0.1589$   & $0.3147$  & $0.1309$ & $0.1712$   & $0.3922$ & $0.1953$   \\
%     $\beta_{0} = \beta_{1} = 2$   & $0.1452$   & $0.3451$      & $0.1596$  & $0.1466$   & $0.4248$      & $0.2295$ \\    
%     \midrule
%     \textsc{Minimax} 
%     & $\boldsymbol{0.0809}$   & $0.4827$      & $0.3216$ & $\boldsymbol{0.0829}$   & $0.4949$      & $0.3179$\\
%     \bottomrule
%   \end{tabular}
% \end{table*}

\subsection{Non-convex Loss on Adult and COMPAS}
For the nonconvex case of \textsc{Surrogate-Min}, we train a multiple-layer neural network model (one hidden layer of $10$ neurons and ReLU as the activation function) with cross-entropy loss with a batch size of $8$ and a learning rate of $0.001$ for both Adult and COMPAS datasets. 

The results of different methods for the nonconvex loss on Adult are already shown in Figure~\ref{fig:adult:noncvx:comparison}. Here, we examine the effects of varying $\boldsymbol{\beta}$ values on the training dynamics, as shown in Figure~\ref{fig:adult:noncvx:beta0}. Specifically, we fix $\beta_{1} = 0$ and vary $\beta_{0}$. With higher $\beta_{0}$, we are essentially focusing more on the minority group, which will improve our fairness metric on EA violation. As shown in Figure~\ref{fig:adult:noncvx:beta0:worstaccuracy} and~\ref{fig:adult:noncvx:beta0:gap}, higher $\beta_{0}$ results improved worst accuracy and reduced EA violation.

The results of different methods for the nonconvex loss on COMPAS are shown in Figure~\ref{fig:COMPAS:noncvx:comparison}. We can observe that our approach achieves superior accuracy (Figure \ref{fig:COMPAS:noncvx:comparison:accuracy}) with the lowest loss (Figure \ref{fig:COMPAS:noncvx:comparison:loss}). It somehow even outperforms \textsc{Minimax} in terms of worst-case (Figure \ref{fig:COMPAS:noncvx:comparison:worstaccuracy}) and EA violation (Figure \ref{fig:COMPAS:noncvx:comparison:gap}), which may be due to the lack of guarantee to achieve optimal solutions for solving nonconvex-concave optimization problems via SGDA. We also examine the effects of varying $\boldsymbol{\beta}$ values on the training dynamics. As demonstrated in Figure \ref{fig:COMPAS:noncvx:beta0}, increasing $\beta_{0}$ leads to improvements in worst-case accuracy and a reduction in EA violation, consistent with our expectations. Similar to the previous experiment, this confirms that our method can flexibly achieve varying degrees of fairness. 

\begin{figure*}[tb!]
     \centering
     \begin{subfigure}[b]{0.24\textwidth}
         \centering
         \includegraphics[width=\textwidth]{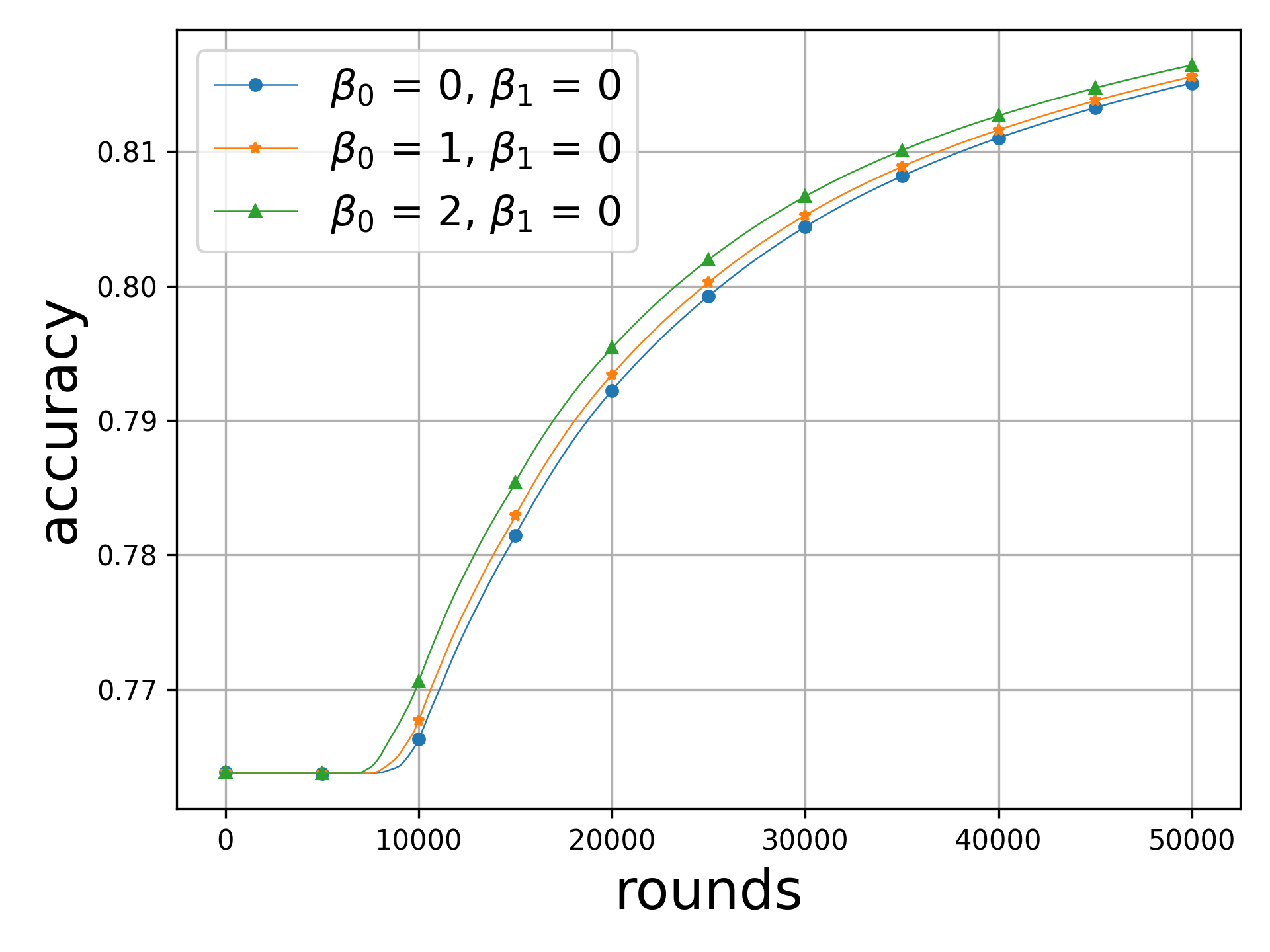}
         \caption{Overall test accuracy}\label{fig:adult:noncvx:beta0:accuracy}
     \end{subfigure}
     \hfill
     \begin{subfigure}[b]{0.24\textwidth}
         \centering
         \includegraphics[width=\textwidth]{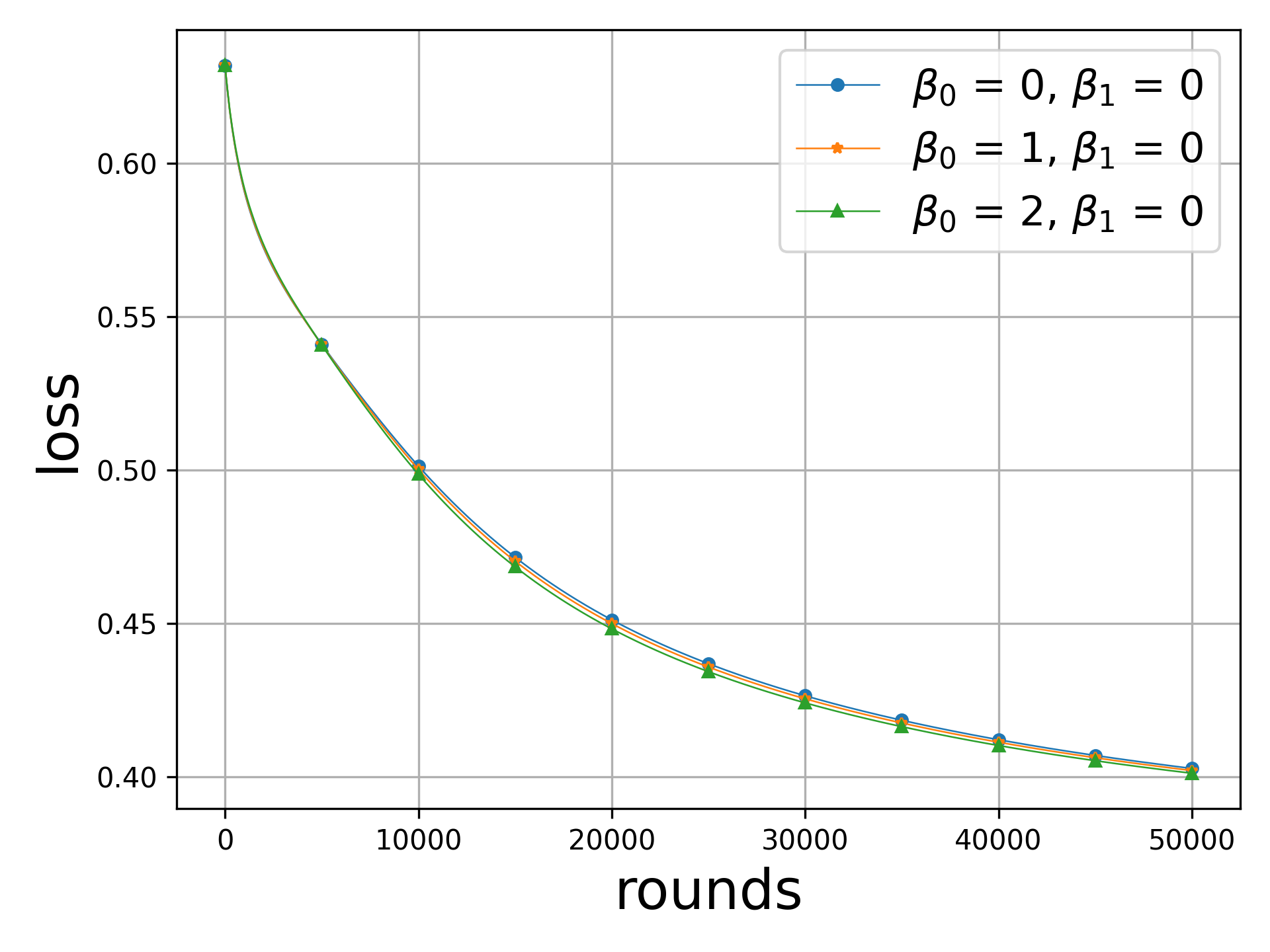}
         \caption{Overall test loss}\label{fig:adult:noncvx:beta0:loss}
     \end{subfigure}
     \hfill
     \begin{subfigure}[b]{0.24\textwidth}
         \centering
         \includegraphics[width=\textwidth]{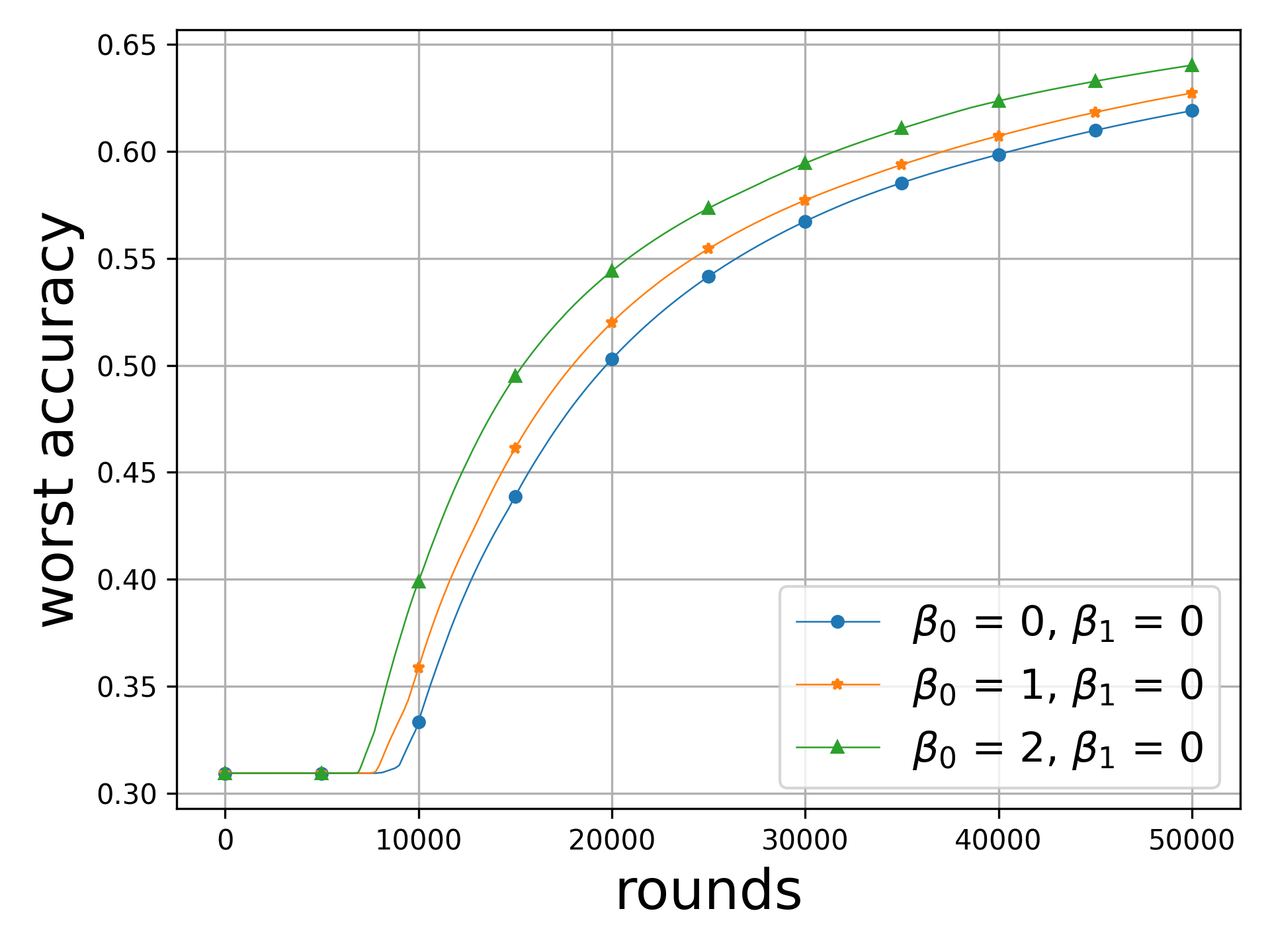}
         \caption{Worst Accuracy}\label{fig:adult:noncvx:beta0:worstaccuracy}
     \end{subfigure}
     \hfill
     \begin{subfigure}[b]{0.24\textwidth}
         \centering
         \includegraphics[width=\textwidth]{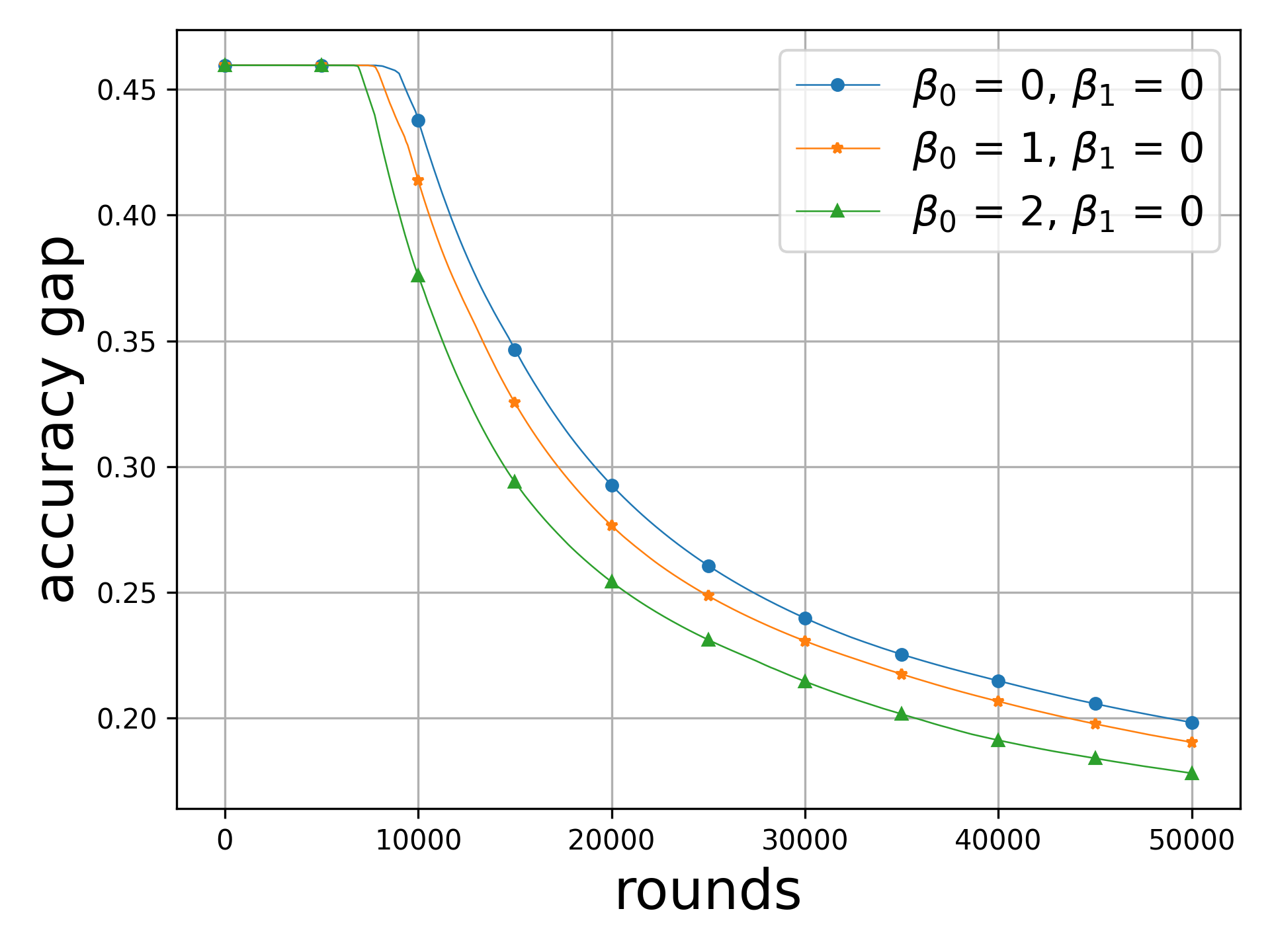}
         \caption{EA violation}\label{fig:adult:noncvx:beta0:gap}
     \end{subfigure}
        \caption{Comparison of accuracy, loss, worst accuracy, and $\epsilon_{\text{EA}}$  for nonconvex loss on Adult with varying $\beta_{0}$.}
        \label{fig:adult:noncvx:beta0}
        % \vspace{-0.5cm}
\end{figure*}

\begin{figure*}[tb!]
     \centering
     \begin{subfigure}[b]{0.24\textwidth}
         \centering
         \includegraphics[width=\textwidth]{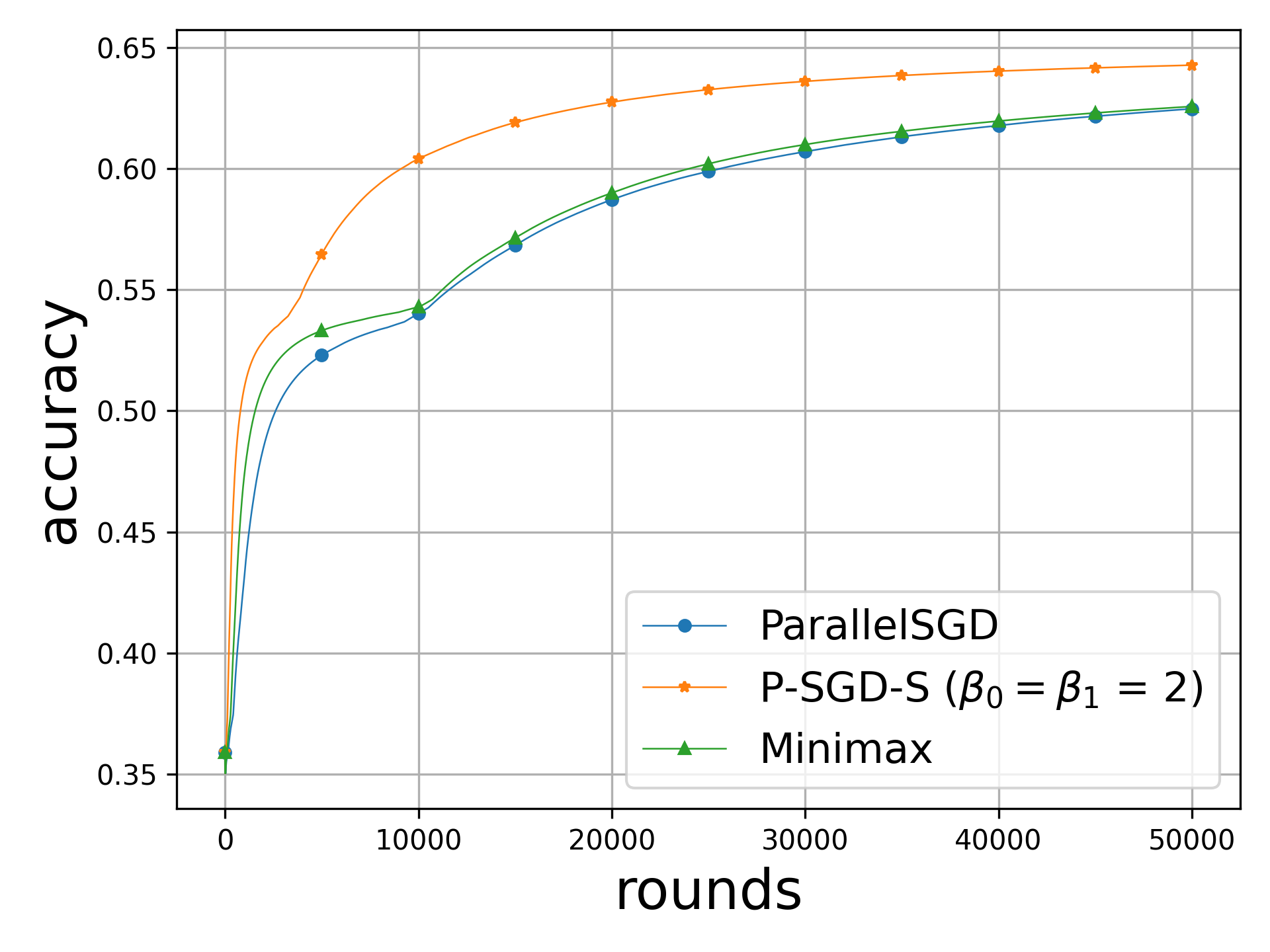}
         \caption{Overall test accuracy}
         \label{fig:COMPAS:noncvx:comparison:accuracy}
     \end{subfigure}
     \hfill
     \begin{subfigure}[b]{0.24\textwidth}
         \centering
         \includegraphics[width=\textwidth]{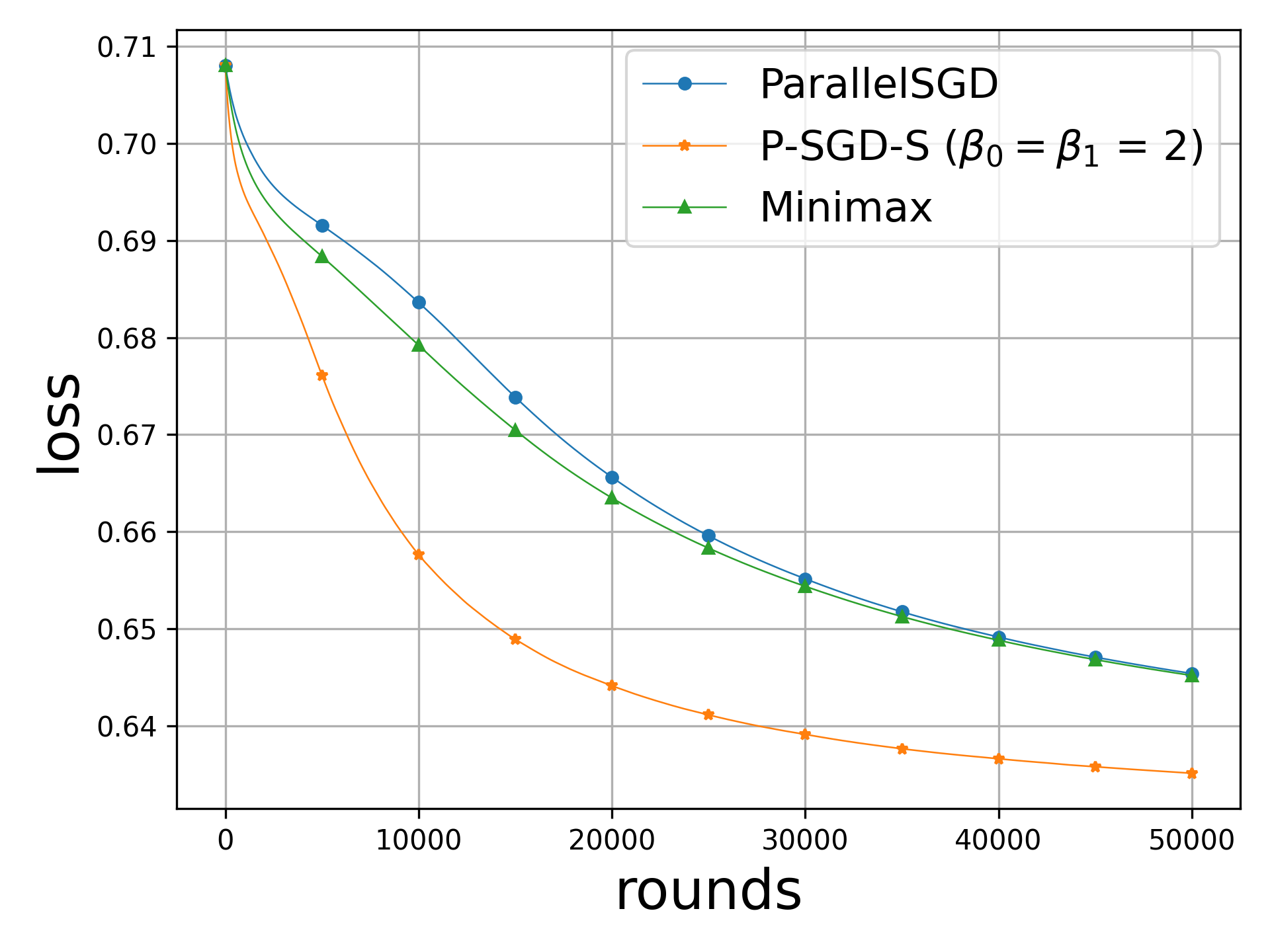}
         \caption{Overall test loss}
         \label{fig:COMPAS:noncvx:comparison:loss}
     \end{subfigure}
     \hfill
     \begin{subfigure}[b]{0.24\textwidth}
         \centering
         \includegraphics[width=\textwidth]{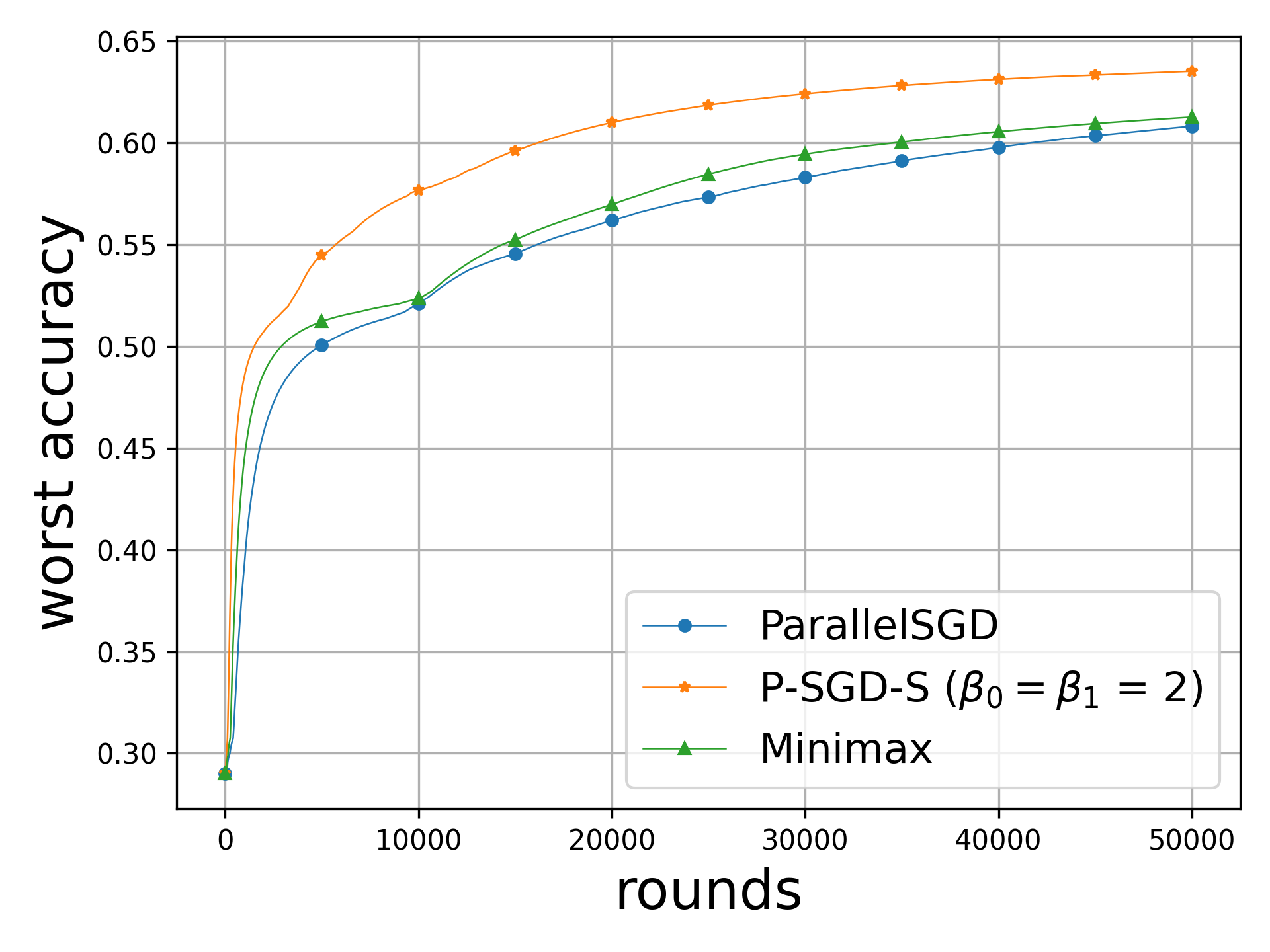}
         \caption{Worst Accuracy}
         \label{fig:COMPAS:noncvx:comparison:worstaccuracy}
     \end{subfigure}
     \hfill
     \begin{subfigure}[b]{0.24\textwidth}
         \centering
         \includegraphics[width=\textwidth]{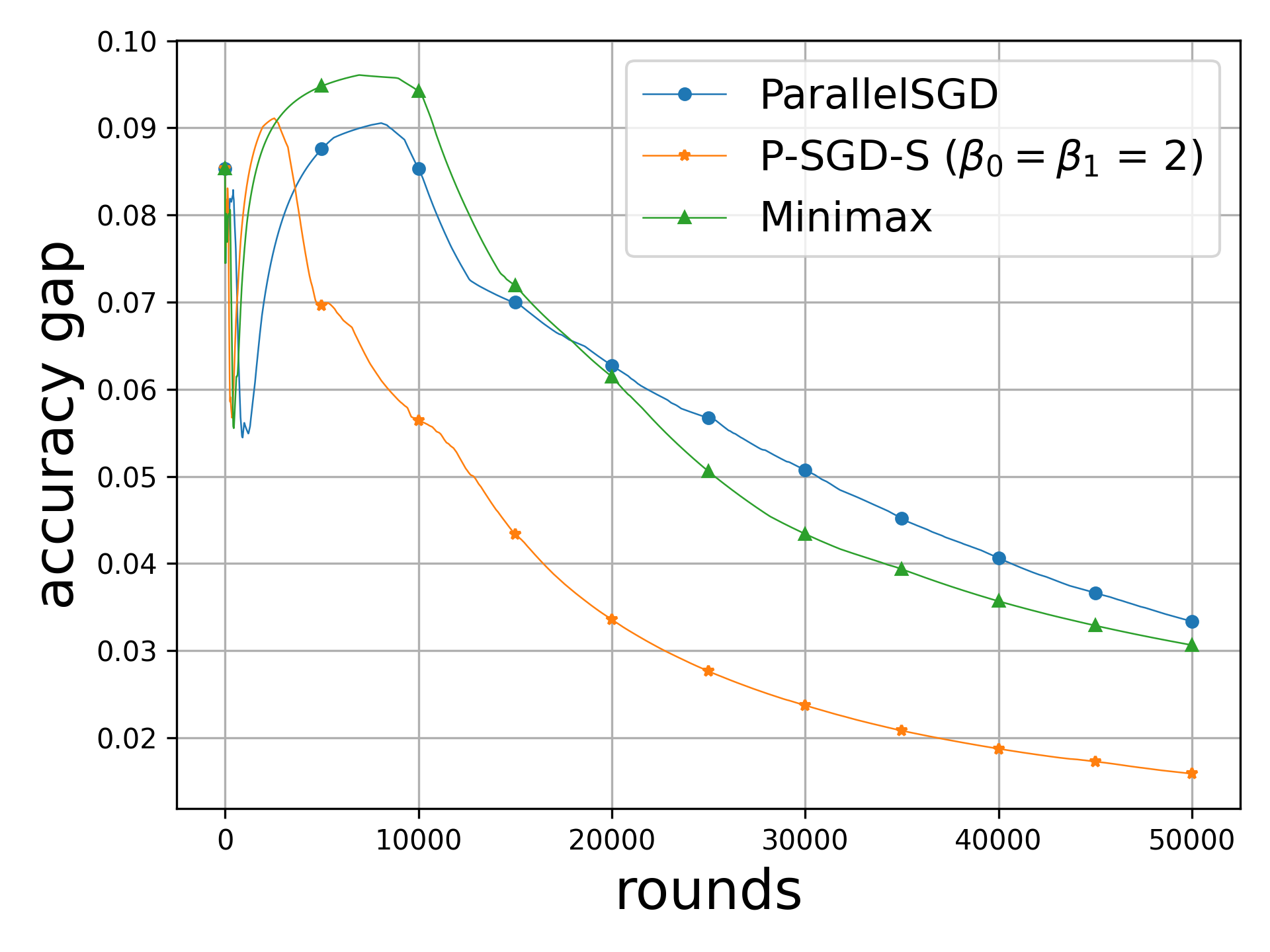}
         \caption{EA violation}
         \label{fig:COMPAS:noncvx:comparison:gap}
     \end{subfigure}
        \caption{Comparison of average accuracy, loss, worst accuracy, and $\epsilon_{\text{EA}}$ for nonconvex loss on COMPAS with different methods.}
        \label{fig:COMPAS:noncvx:comparison}
        % \vspace{-0.5cm}
\end{figure*}

\begin{figure*}[tb!]
     \centering
     \begin{subfigure}[b]{0.24\textwidth}
         \centering
         \includegraphics[width=\textwidth]{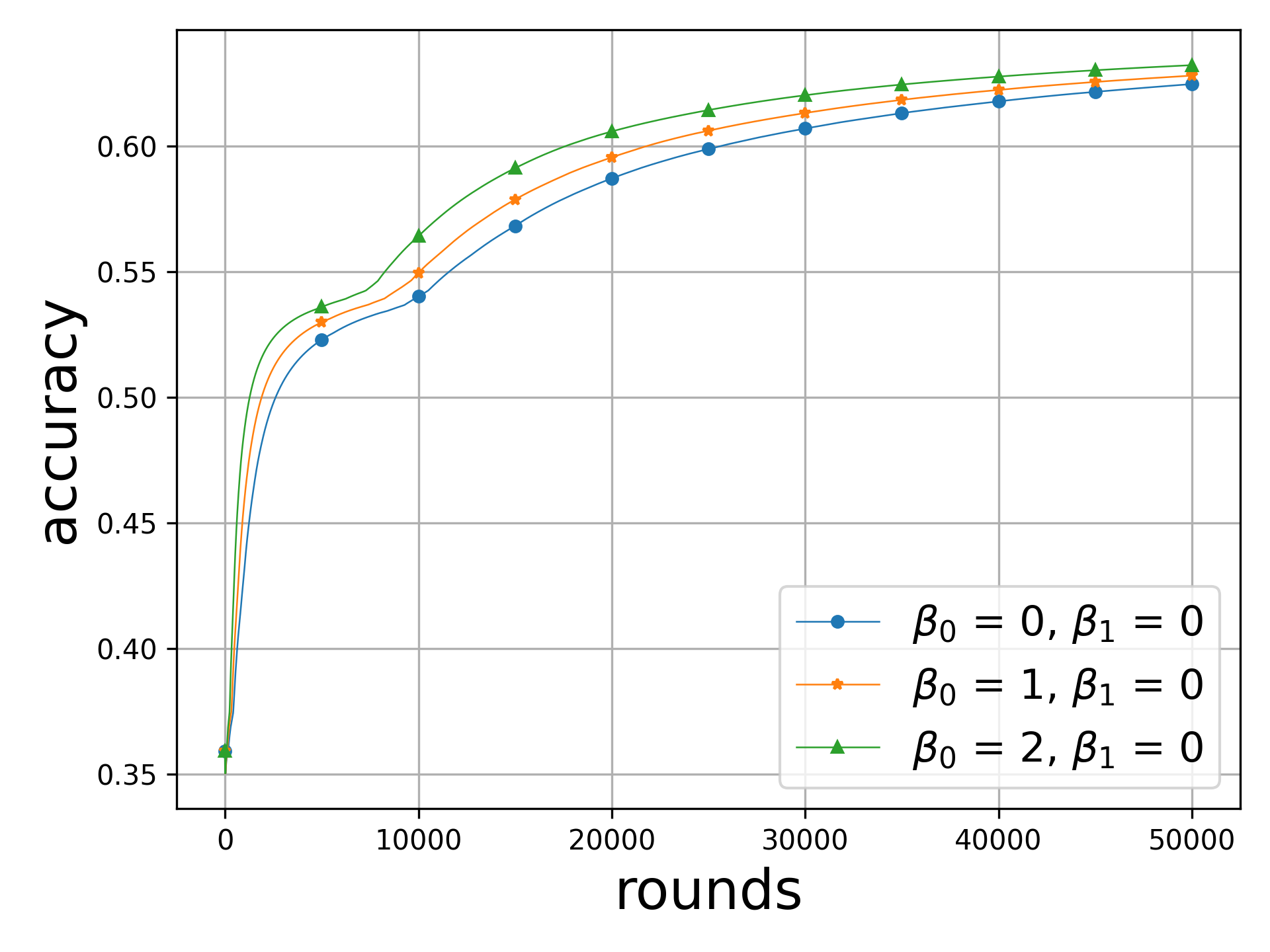}
         \caption{Overall test accuracy}\label{fig:COMPAS:noncvx:beta0:accuracy}
     \end{subfigure}
     \hfill
     \begin{subfigure}[b]{0.24\textwidth}
         \centering
         \includegraphics[width=\textwidth]{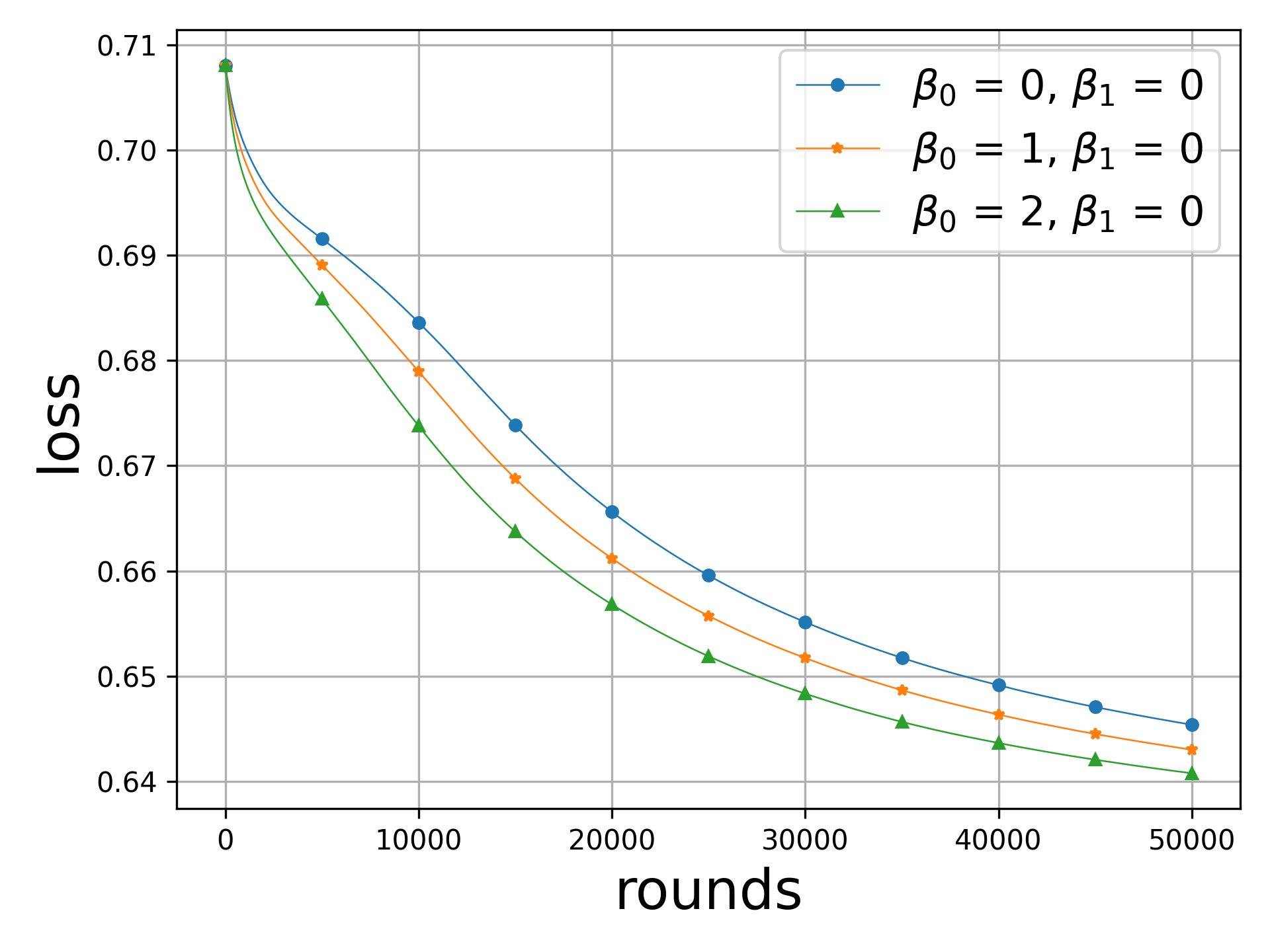}
         \caption{Overall test loss}\label{fig:COMPAS:noncvx:beta0:loss}
     \end{subfigure}
     \hfill
     \begin{subfigure}[b]{0.24\textwidth}
         \centering
         \includegraphics[width=\textwidth]{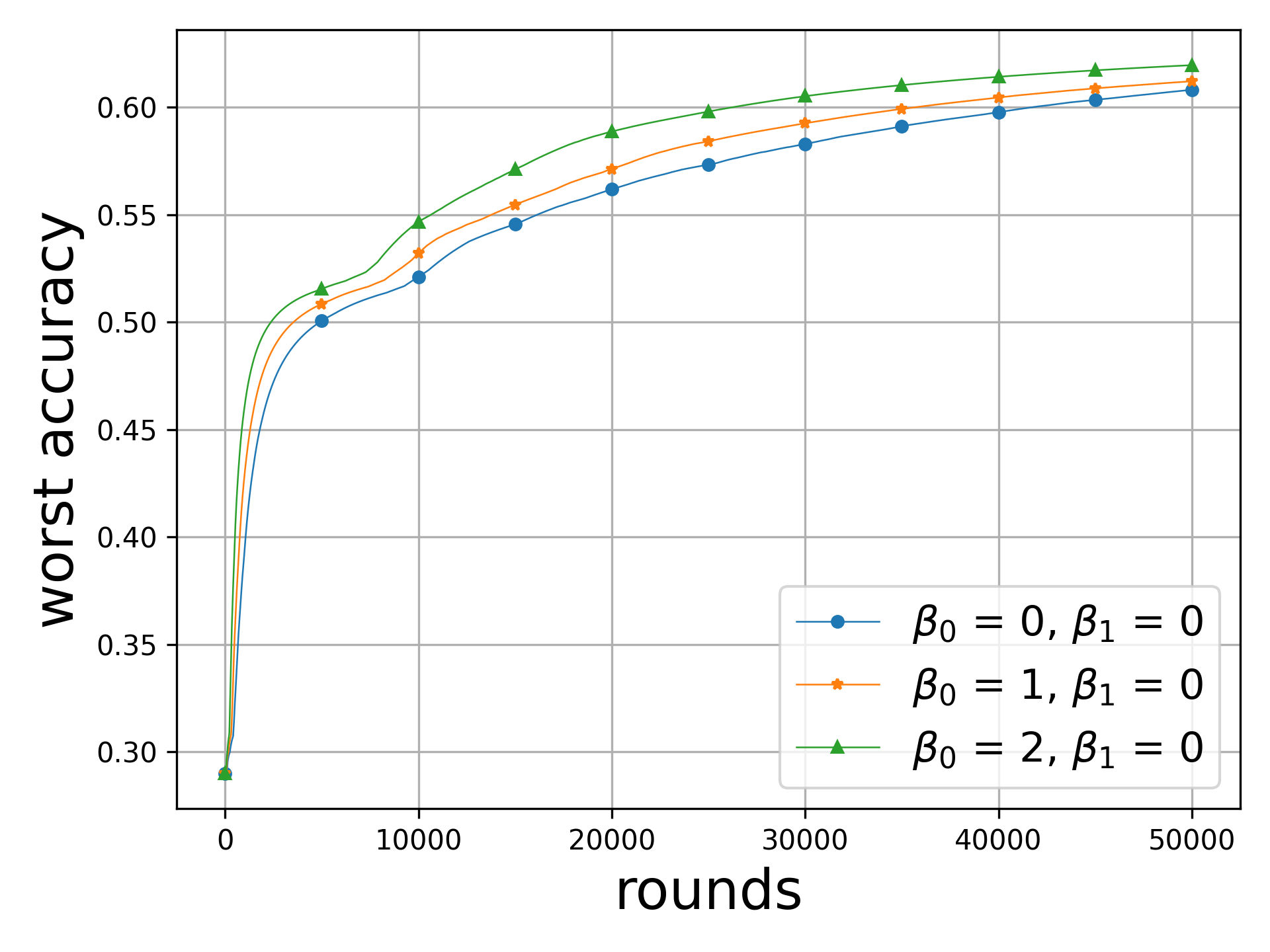}
         \caption{Worst Accuracy}\label{fig:COMPAS:noncvx:beta0:worstaccuracy}
     \end{subfigure}
     \hfill
     \begin{subfigure}[b]{0.24\textwidth}
         \centering
         \includegraphics[width=\textwidth]{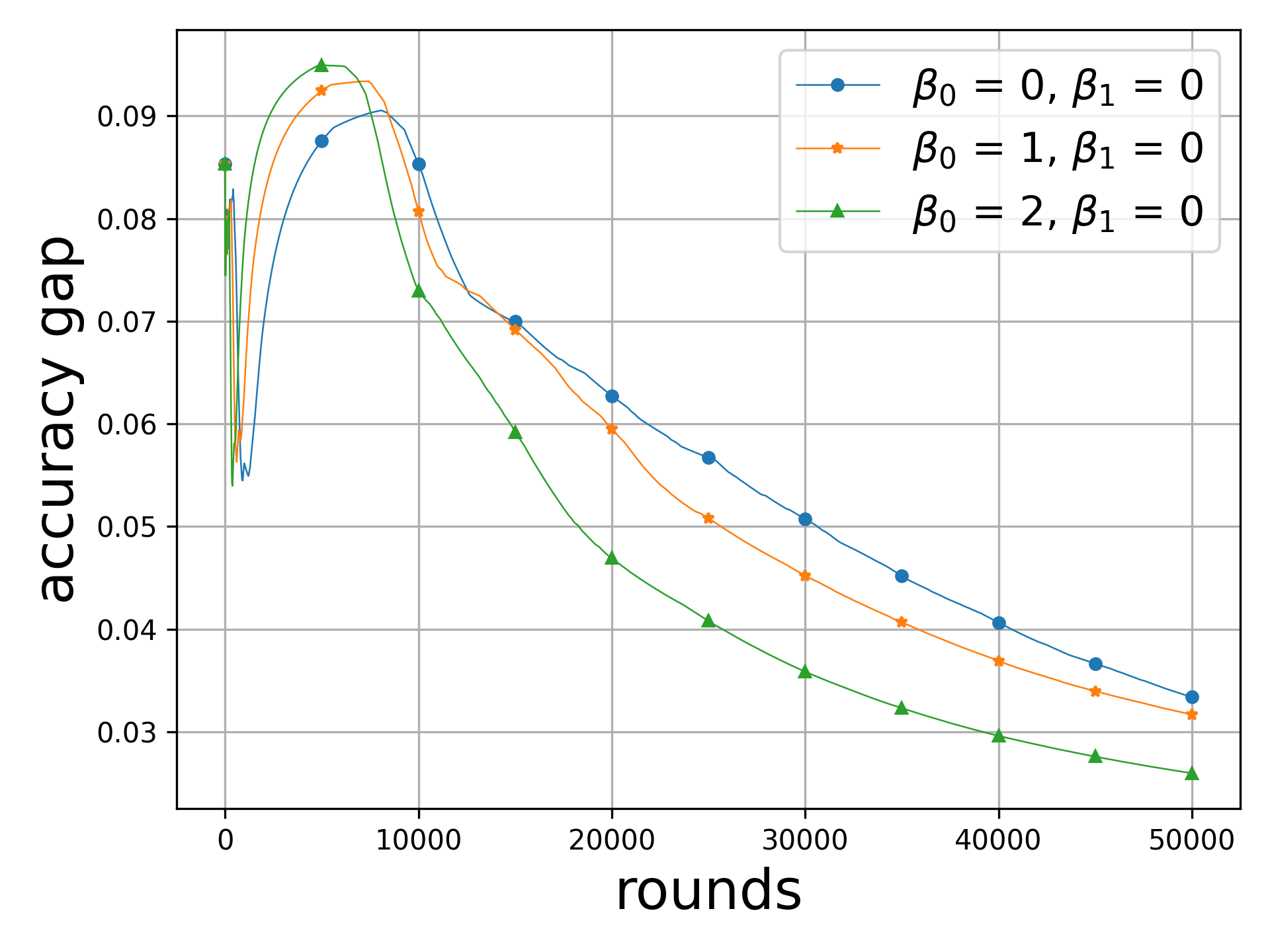}
         \caption{EA violation}\label{fig:COMPAS:noncvx:beta0:gap}
     \end{subfigure}
        \caption{Comparison of accuracy, loss, worst accuracy, and $\epsilon_{\text{EA}}$ for nonconvex loss on COMPAS with varying $\beta_{0}$.}
        \label{fig:COMPAS:noncvx:beta0}
        % \vspace{-0.5cm}
\end{figure*}

\subsection{Label-partitioned Scenario}

We solve an image classification problem on the Fashion-MNIST~\cite{Fashion_MNIST} dataset for the label-partitioned scenario. There are $10$ classes of clothes, each containing $6000$ training examples and $1000$ test examples of gray-scale images of size $28 \times 28$. As we focus on the case of binary labels, we divide the dataset into two groups: shirt, i.e., group $0$, and non-shirt, i.e., group $1$. We choose $\boldsymbol{\alpha}$ according to the data ratio of each group, i.e., $\alpha_{0} = 0.1$ and $\alpha_{1} = 0.9$. We apply $\beta_{0}$ to the shirt group and $\beta_{1}$ to the nonshirt group. We set the number of iterations to $100,000$. For the convex case, we train a logistic regression model on cross-entropy loss with a batch size of $1$ and a learning rate of $0.001$. 
% For the nonconvex case, we train a multiple-layer neural network model (two hidden layers of $300$ and $100$ neurons each and ReLU as the activation function) on cross-entropy loss with a batch size of $8$ and a learning rate of $0.01$. 

Although this scenario does not fall into the case where the level of fairness can be quantified by DP
violation, EO violation, and EA violation, we can still use a similar metric to the EA violation, i.e., accuracy gap, which is defined as the absolute value of the difference between the test accuracy of the two label-partitioned groups. The results of time-averaged test accuracy, test loss, worst test accuracy, and accuracy gap between groups of different benchmarks are shown in Figure~\ref{fig:fashion:cvx:comparison} for the convex case. For our method, we choose $\beta_{0} = \beta_{1} = 2$. In Figure~\ref{fig:fashion:cvx:comparison:accuracy}, we observe that for the average test accuracy, our method converges to a better model compared to the vanilla \textsc{Parallel SGD} as well as the \textsc{Minimax}. However, in terms of the worst test accuracy among the two groups shown in Figure~\ref{fig:fashion:cvx:comparison:worstaccuracy}, it is clear that \textsc{Minimax} achieves the highest worst accuracy, which matches the theory that minimax is by design trying to minimize the loss even for the any mixture of the two distributions. Nevertheless, we can also observe that our method lies in between the \textsc{Parallel SGD} and the \textsc{Minimax}, which is as expected. Similarly, for the accuracy gap, \textsc{Minimax} outperforms our method, which further outperforms the \textsc{ParallelSGD}. We also demonstrate the effects of $\boldsymbol{\beta}$ in Figure~\ref{fig:fashion:cvx:beta1}. We fix $\beta_{0} = 0$ and vary $\beta_{1}$. With higher $\beta_{1}$, we are essentially focusing more on the majority group, which will worsen our fairness metric on accuracy. As shown in Figure~\ref{fig:fashion:cvx:beta1:worstaccuracy} and~\ref{fig:fashion:cvx:beta1:gap}, higher $\beta_{1}$ results a lower worst accuracy and a higher accuracy gap. 

\begin{figure*}[tb!]
     \centering
     \begin{subfigure}[b]{0.24\textwidth}
         \centering
         \includegraphics[width=\textwidth]{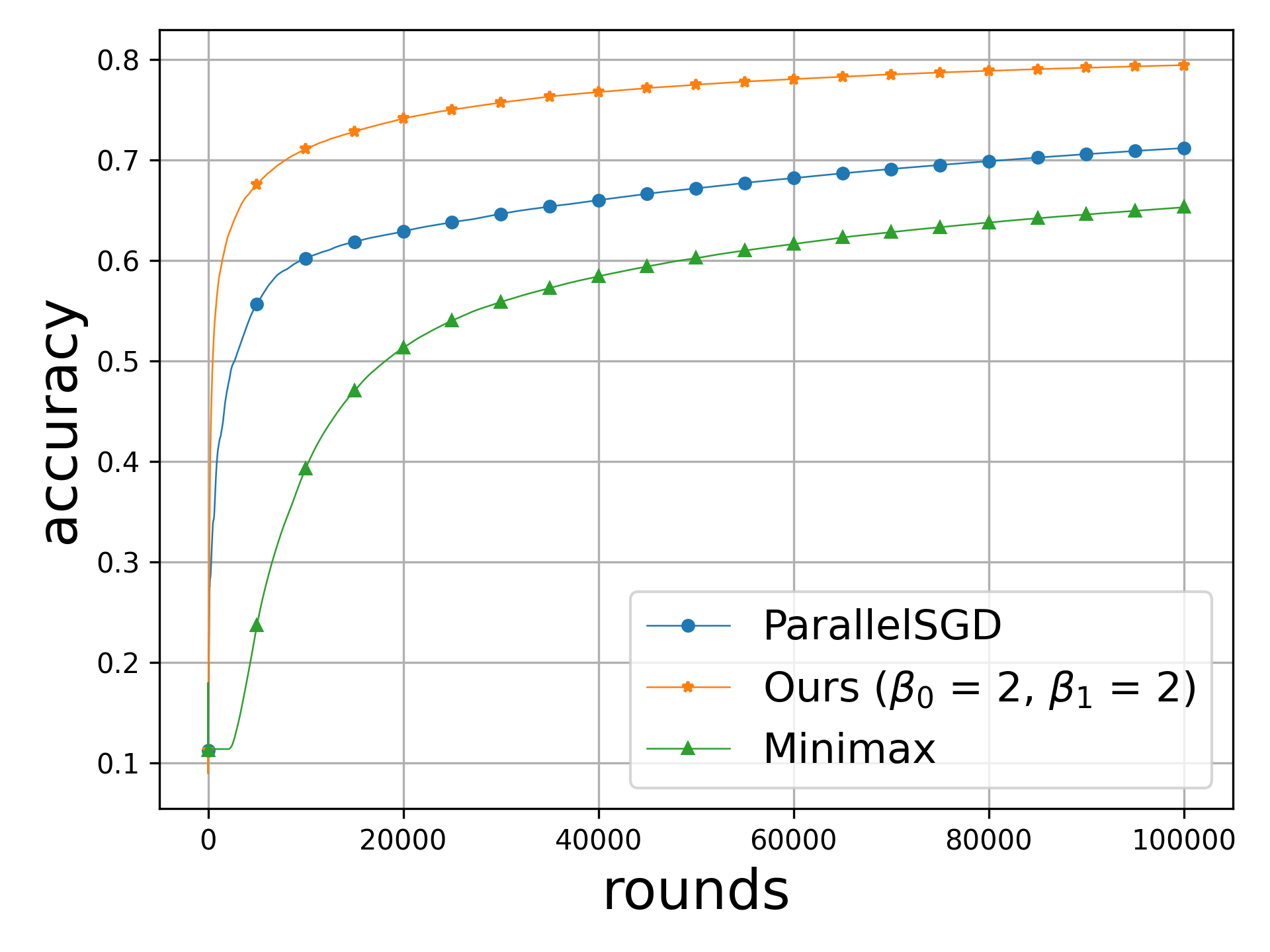}
         \caption{Overall test accuracy}
         \label{fig:fashion:cvx:comparison:accuracy}
     \end{subfigure}
     \hfill
     \begin{subfigure}[b]{0.24\textwidth}
         \centering
         \includegraphics[width=\textwidth]{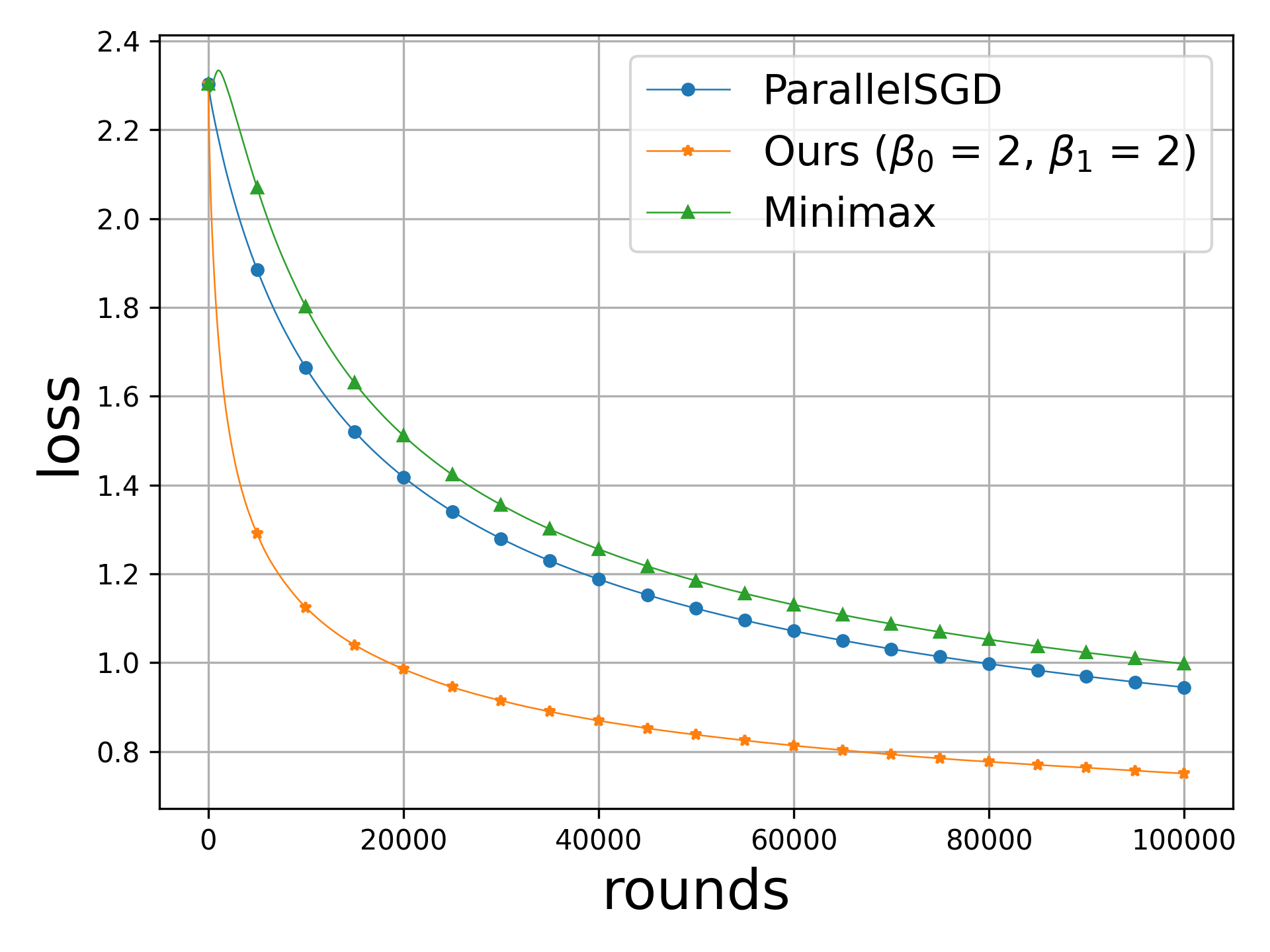}
         \caption{Overall test loss}
         \label{fig:fashion:cvx:comparison:loss}
     \end{subfigure}
     \hfill
     \begin{subfigure}[b]{0.24\textwidth}
         \centering
         \includegraphics[width=\textwidth]{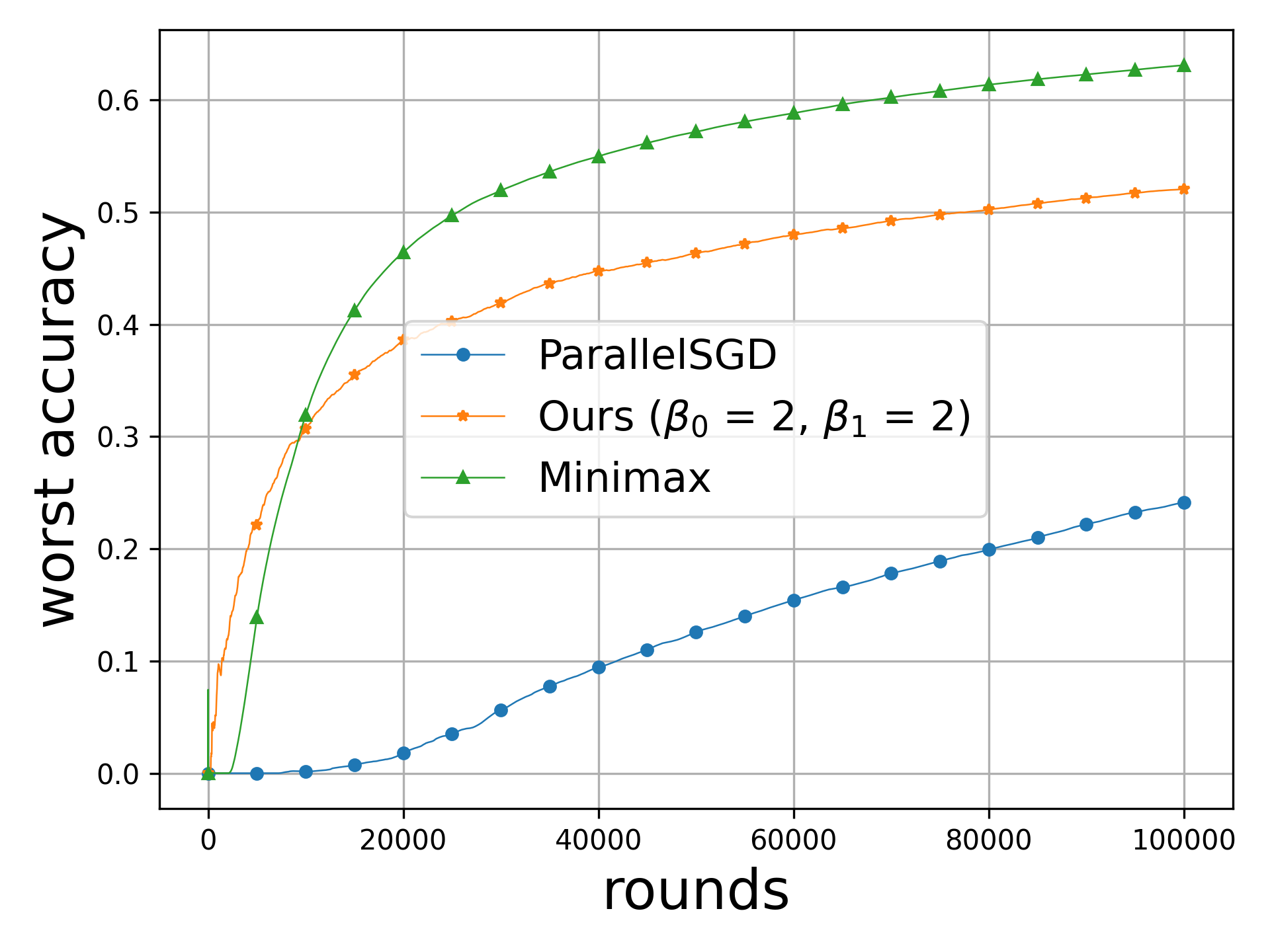}
         \caption{Worst Accuracy}
         \label{fig:fashion:cvx:comparison:worstaccuracy}
     \end{subfigure}
     \hfill
     \begin{subfigure}[b]{0.24\textwidth}
         \centering
         \includegraphics[width=\textwidth]{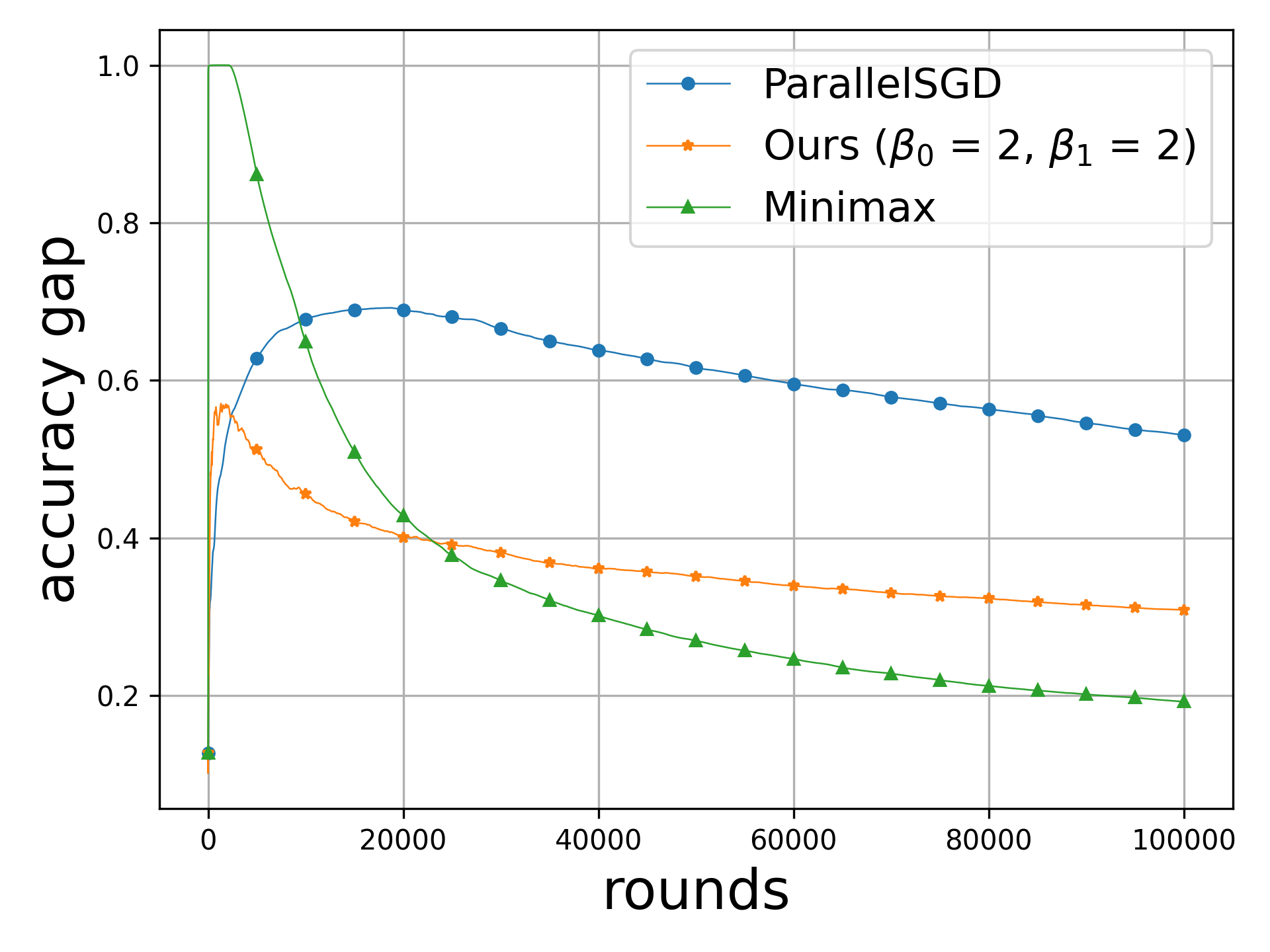}
         \caption{Accuracy gap}
         \label{fig:fashion:cvx:comparison:gap}
     \end{subfigure}
        \caption{Comparison of accuracy, loss, worst accuracy, and accuracy gap for LR on Fashion-MNIST with different methods.}
        \label{fig:fashion:cvx:comparison}
        % \vspace{-0.5cm}
\end{figure*}

\begin{figure*}[tb!]
     \centering
     \begin{subfigure}[b]{0.24\textwidth}
         \centering
         \includegraphics[width=\textwidth]{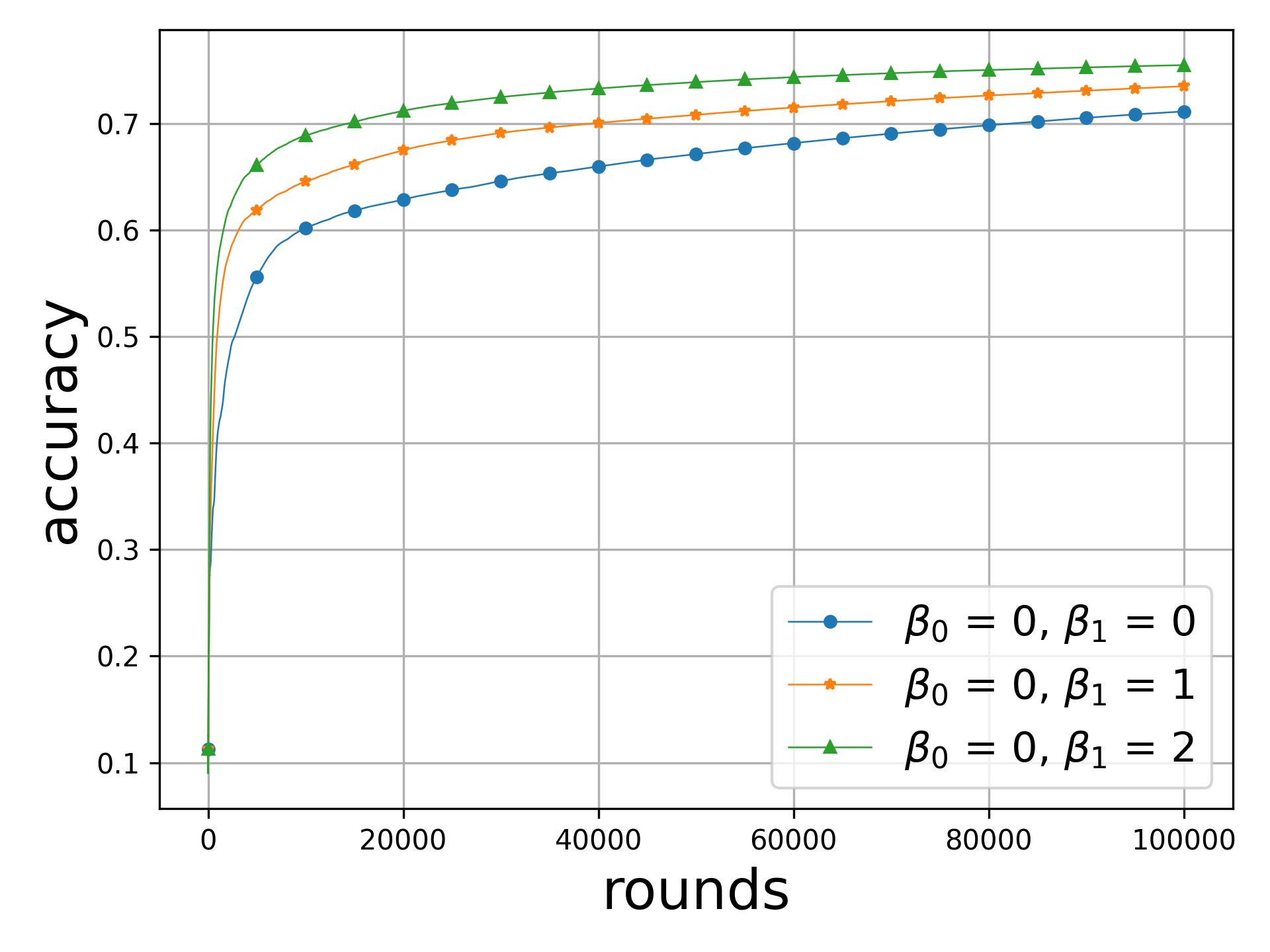}
         \caption{Overall test accuracy}\label{fig:fashion:cvx:beta1:accuracy}
     \end{subfigure}
     \hfill
     \begin{subfigure}[b]{0.24\textwidth}
         \centering
         \includegraphics[width=\textwidth]{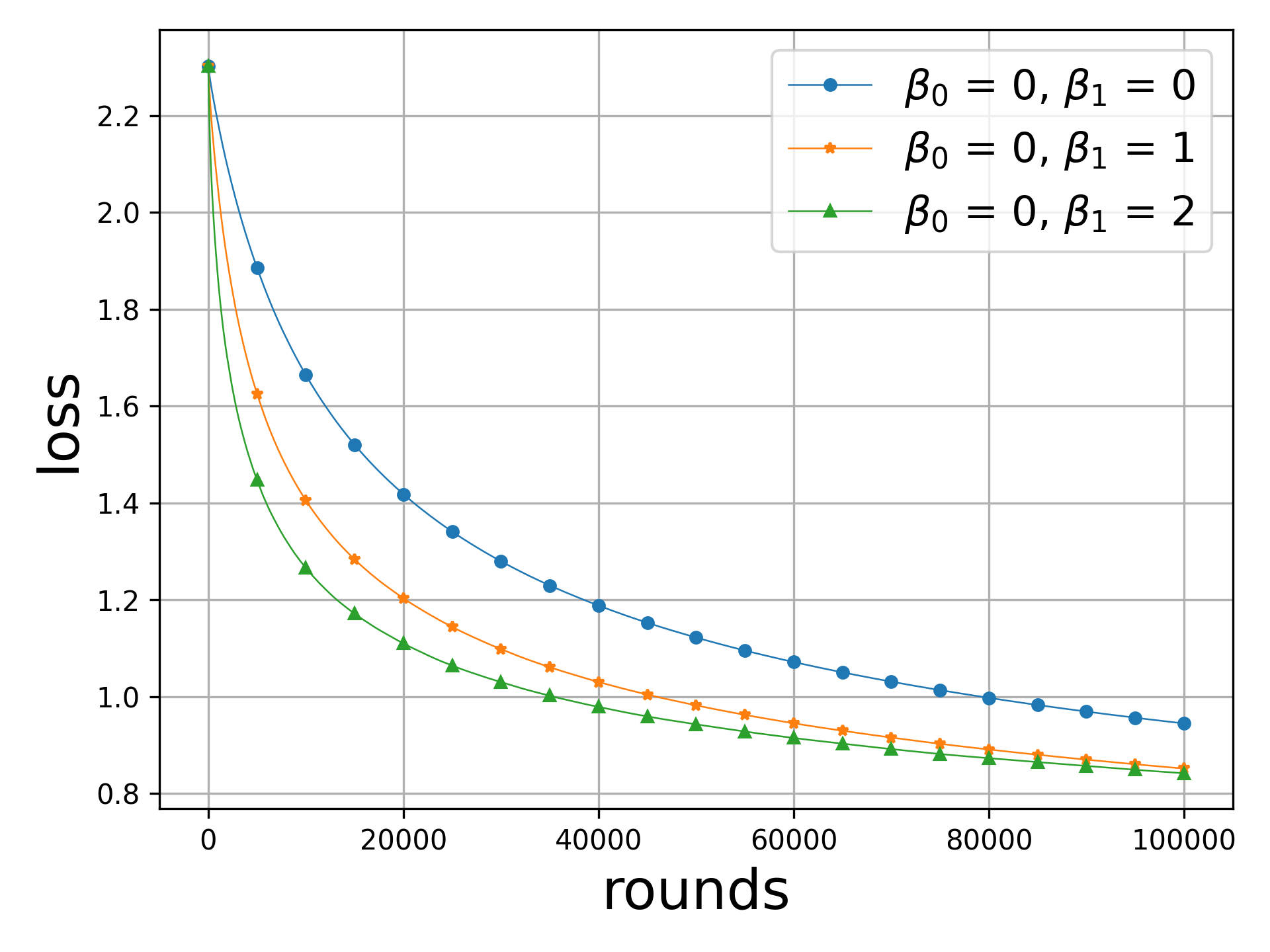}
         \caption{Overall test loss}\label{fig:fashion:cvx:beta1:loss}
     \end{subfigure}
     \hfill
     \begin{subfigure}[b]{0.24\textwidth}
         \centering
         \includegraphics[width=\textwidth]{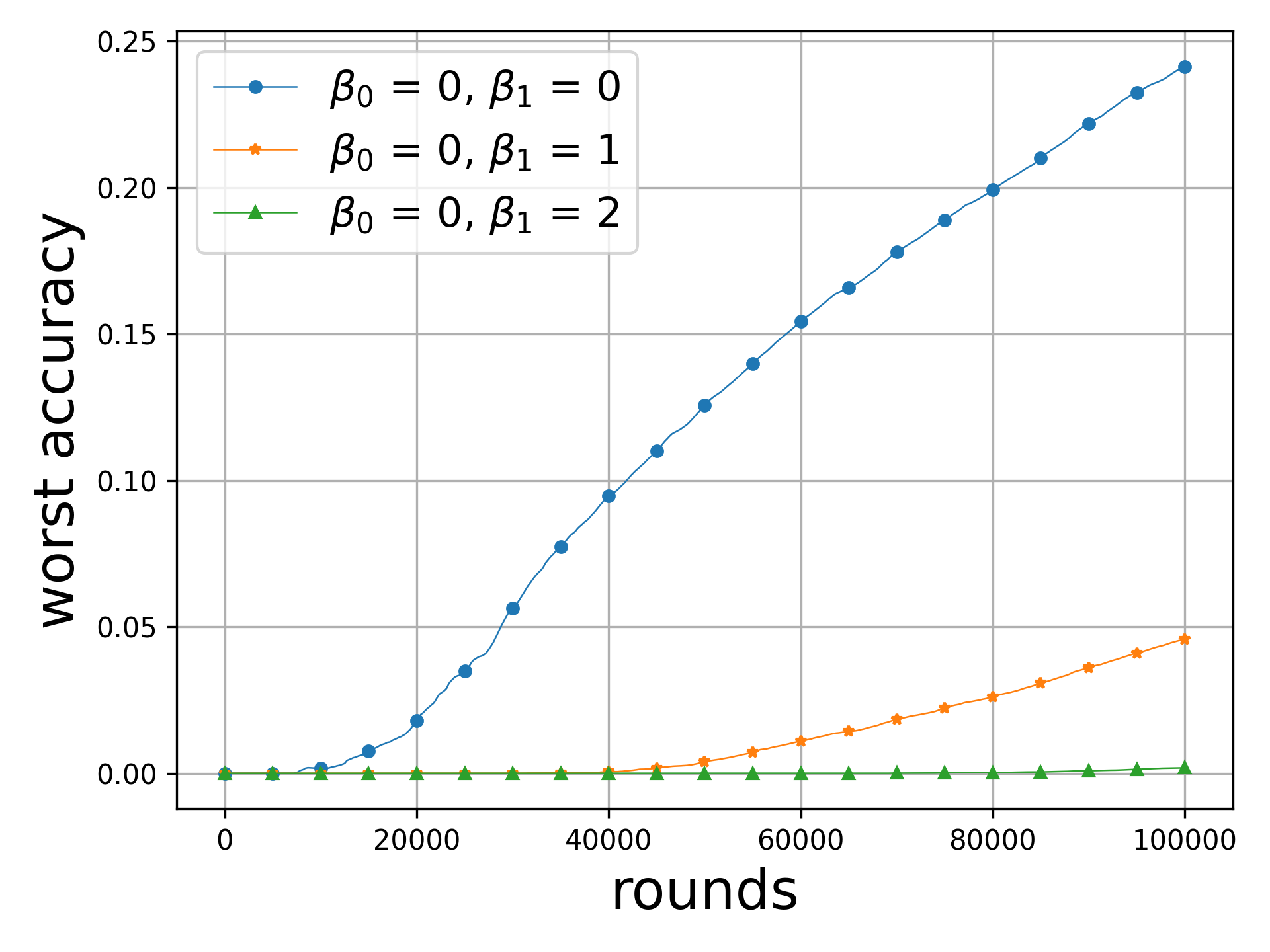}
         \caption{Worst Accuracy}\label{fig:fashion:cvx:beta1:worstaccuracy}
     \end{subfigure}
     \hfill
     \begin{subfigure}[b]{0.24\textwidth}
         \centering
         \includegraphics[width=\textwidth]{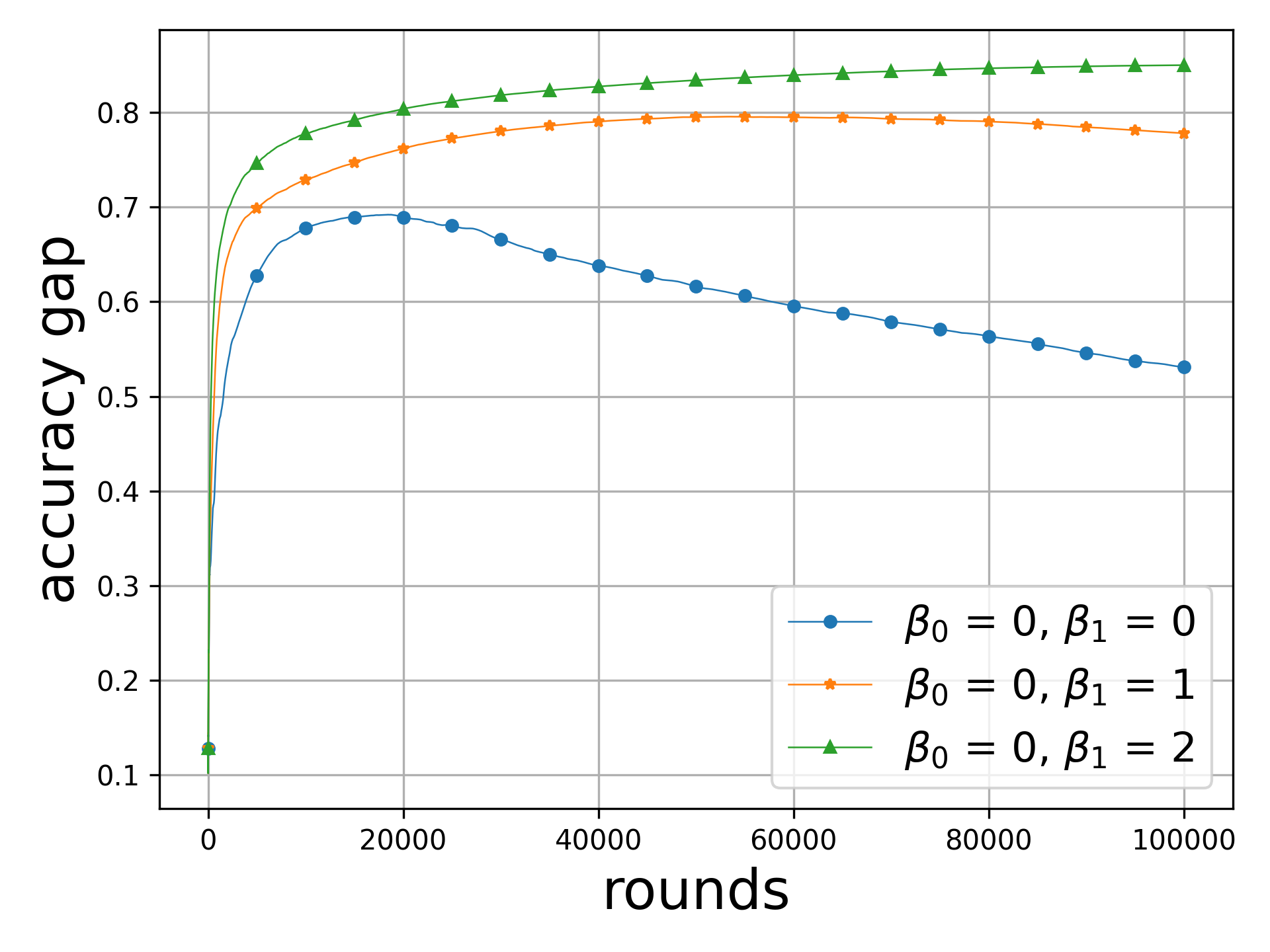}
         \caption{Accuracy gap}\label{fig:fashion:cvx:beta1:gap}
     \end{subfigure}
        \caption{Comparison of accuracy, loss, worst accuracy, and accuracy gap for LR on Fashion-MNIST with varying $\beta_{1}$.}
        \label{fig:fashion:cvx:beta1}
        % \vspace{-0.5cm}
\end{figure*}

\end{document}

%% file: defs.tex
\usepackage{amsmath,amsfonts,bm}

\newtheorem{theorem}{Theorem}
\newtheorem{corollary}{Corollary}
\newtheorem{lemma}{Lemma}

\newtheorem{definition}{Definition}
\newtheorem{remark}{Remark}
\newtheorem{assumption}{Assumption}

\newcommand{\bx}{\boldsymbol{x}}
\newcommand{\ba}{\boldsymbol{a}}
\newcommand{\bz}{\boldsymbol{z}}
\newcommand{\bw}{\boldsymbol{w}}

\newcommand{\cA}{{\mathcal A}}

\newcommand{\cD}{{\mathcal D}}

\newcommand{\cH}{{\mathcal H}}
\newcommand{\cI}{{\mathcal I}}

\newcommand{\cO}{{\mathcal O}}

\newcommand{\cS}{{\mathcal S}}
\newcommand{\cT}{{\mathcal T}}

\newcommand{\cW}{{\mathcal W}}
\newcommand{\cX}{{\mathcal X}}
\newcommand{\cY}{{\mathcal Y}}
\newcommand{\cZ}{{\mathcal Z}}

\newcommand{\bbE}{{\mathbb E}}
\newcommand{\bbR}{{\mathbb R}}

\newcommand{\ignore}[1]{}

\def\eqref#1{equation~\ref{#1}}
\def\1{\bm{1}}

\DeclareMathAlphabet{\mathsfit}{\encodingdefault}{\sfdefault}{m}{sl}
\SetMathAlphabet{\mathsfit}{bold}{\encodingdefault}{\sfdefault}{bx}{n}